\documentclass[twoside,11pt,abbrvbib]{article}

\usepackage{jmlr2e}

\usepackage{Preamble}

\newcommand{\changedCR}[1]{#1}  

\usepackage{amssymb}
\usepackage{amsmath}

\usepackage{float}
\usepackage{multirow}

\usepackage{graphicx}
\usepackage{caption} 

\usepackage[ruled,vlined]{algorithm2e}
\usepackage{microtype}

\usepackage{booktabs}

\usepackage{amsfonts}

\usepackage{mathtools}
\usepackage{microtype}
\usepackage{graphicx}
\usepackage{booktabs} 
\usepackage{lipsum,graphicx,multicol}
\usepackage{arydshln}
\usepackage{amssymb}
\usepackage{graphicx}
\usepackage{amsfonts}
\usepackage{amsmath}
\usepackage{colortbl}
\usepackage{times}
\usepackage{footnote}
\makesavenoteenv{tabular}
\makesavenoteenv{table}

\usepackage{booktabs}

\usepackage{tikz}
\usetikzlibrary{arrows,shapes,calc,positioning,backgrounds} 
\usetikzlibrary{intersections}

\graphicspath{{Results/}}

\newcommand{\Wvec}{\vec{W}}

\newtheorem{lem}{Lemma}
\newtheorem{thm}{Theorem}

\newcommand*{\QED}{\hfill\ensuremath{\blacksquare}}

\newcommand{\sproof}{ \noindent {\textbf{Proof Sketch}}}

\usepackage{lastpage}

\jmlrheading{24}{2023}{1-\pageref{LastPage}}{4/22; Revised
8/23}{8/23}{22-0359}{Hamid Mousavi, Jakob Drefs, Florian Hirschberger, and J\"{o}rg L\"{u}cke}

\ShortHeadings{Generic Unsupervised Optimization for an LVM With EF Observables}{Mousavi, Drefs, Hirschberger, and L\"{u}cke}
\firstpageno{1}

\begin{document}

%
%
%
%
%
%
\title{Generic Unsupervised Optimization for a Latent Variable Model \\
With Exponential Family Observables}
%
%
%
%
%







%
%


%

\author{\name Hamid Mousavi \email hamid.mousavi@uol.de 
       \AND
       \name Jakob Drefs \email jakob.drefs@uol.de 
       \AND
        \name Florian Hirschberger \email florian.hirschberger@uol.de 
       \AND
       \name J\"{o}rg L\"{u}cke \email joerg.luecke@uol.de \\
       \addr Machine Learning Lab\\
       University of Oldenburg\\
       26129 Oldenburg, Germany     
       }

\editor{Aapo Hyv\"{a}rinen}

\maketitle

\begin{abstract}
Latent variable models (LVMs) represent observed variables by parameterized functions of latent variables. 
%
%
Prominent examples of LVMs for unsupervised learning are probabilistic PCA or probabilistic sparse coding which both assume a weighted linear summation of the latents to determine the mean of a Gaussian distribution for the observables. In many cases, however, observables do not follow a Gaussian distribution. 
For unsupervised learning, LVMs which assume specific non-Gaussian observables (e.g., Bernoulli or Poisson) have therefore been considered.
Already for specific choices of distributions, parameter optimization is challenging and only a few previous contributions
considered LVMs with more generally defined observable distributions.
In this contribution, we do consider LVMs that are defined for a range of different distributions, i.e., observables can
follow any (regular) distribution of the exponential family. Furthermore, the novel class of LVMs presented here is defined for binary latents,
and it uses maximization in place of summation to link the latents to observables.
%
%
%
In order to derive an optimization procedure, we follow an expectation maximization approach for maximum likelihood parameter estimation. 
We then show, as our main result, that a set of very concise parameter update equations can be derived which feature the same functional form for all exponential family
distributions. The derived generic optimization can consequently be applied (without further derivations) to different types of metric data (Gaussian and non-Gaussian) as well as to different types of discrete data.
%
%
%
%
Moreover, the derived optimization equations can be combined with a recently suggested variational acceleration which is likewise generically applicable to the
LVMs considered here. 
Thus, the combination maintains generic and direct applicability of the derived optimization procedure, but, crucially, enables efficient scalability.
We numerically verify our analytical results using different observable distributions, and, furthermore, discuss some potential applications such as 
learning of variance structure, noise type estimation and denoising.
\end{abstract}


\section{Introduction}
\label{Sec_intro}
Latent Variable Models (LVMs) represent a probabilistic approach in which observed variables are assumed
to depend on latent variables through specific parameterized functions. 
%
A large variety of unsupervised Machine Learning approaches is based on LVMs, and prominent examples are probabilistic Principal Component Analysis \citep[p-PCA;][]{TippingBishop1999,Roweis1998}, Factor Analysis \citep[FA;][]{Everitt1984,BartholomewKnott1987} or probabilistic Sparse Coding \citep[SC;][]{OlshausenField1996}. 
%
The data models of these three and many of the related approaches consist of latent variables whose weighted linear sum determines the mean of a Gaussian distribution. In probabilistic SC, for instance, the set of $D$ observed variables $y_d$ is modelled to depend on a set of $H$ latent variables $s_h$ (we assume $\vec{y} = (y_1, \ldots, y_D)^T$ and $\vec{s} = (s_1, \ldots, s_H)^T$, where superscript $T$ denotes the transpose operator), according to the following generative model:
\begin{align}
p(\sVec\,| \Theta) &=\prod_{h=1}^{H} p_{\mathrm{sparse}}(s_h;\Lambda) , \label{EqnSCPrior}\\
p(\yVec\,|\,\sVec,\Theta) &= \prod_{d=1}^{D}\mathcal{N}(y_{d};\sum_{h=1}^{H}W_{dh}s_{h},\sigma^{2}) , \label{EqnSCNoise}
\end{align}
where $\Wvec_h=(W_{1h}, \ldots, W_{Dh})^{\mathrm{T}}$ is commonly referred to as the {\em generative field} of unit $h$, and the matrix containing all generative fields, $W=(\Wvec_1,\ldots,\Wvec_H)$, is often known as the model's {\em dictionary}. Here, $\Theta=(\Lambda,W,\sigma^2)$ denotes all parameters of the model (including prior parameter $\Lambda$,  dictionary $W$ and a variance $\sigma^2$),  and the term `$p_{\mathrm{sparse}}$' refers to a sparse distribution that is used as prior distribution $p(\sVec\,| \Theta)$. 
The canonical choice for such a distribution is the Laplace distribution \citep[see, e.g.,][]{OlshausenField1996,Tibshirani1996}. Other choices include
the Cauchy distribution \citep[which was used by][alongside Laplace]{OlshausenField1996}, Student-t \citep[][]{BerkesEtAl2009}, Bernoulli \citep[][]{HaftEtAl2004,HennigesEtAl2010}, Categorical \citep[][]{ExarchakisLucke2017},
or spike-and-slab \citep[][]{TitsiasGredilla2011,goodfellow2012scaling,SheikhEtAl2014}.

All LVMs just referred to (and presumably the great majority of LVMs) assume a Gaussian distribution for the observables. Most LVMs for Gaussian data then
linearly link the set of latent variables to the set of observables using a weighted sum. 
But a Gaussian distribution cannot be assumed for all data. Consequently, other observable distributions have been investigated. For instance, the Poisson distribution has been of interest for count data,  and likewise, the Bernoulli distribution for binary data.
The aforementioned LVMs (p-PCA, SC etc.) represent well-known examples for Gaussian observables. For the case of non-Gaussian observables, examples include FA with Poisson observables \citep[][]{ZhouEtAl2012}, FA with binary (Boolean) observables \citep[BFA;][]{FrolovEtAl2014} or PCA with Poisson observables \citep[][]{salmon2014poisson,ChiquetEtAl2018}. Also well-known models such as noisy-OR Bayes nets \citep[][]{SingliarEtAl2006,JerniteEtAl2013} or shallow Sigmoid Belief Networks \citep[SBNs;][]{SaulEtAl1996,GanEtAl2015} can be considered as LVMs with Bernoulli distributed observables \citep[also compare][]{LindenHambleton2013}. Moreover, types of non-negative matrix factorization 
(where non-Gaussian distributions are also used) can be regarded as a form of LVM
\citep[see, e.g.,][and references therein]{Hoffman2012,basbug2016hierarchical}.

Unlike such models with specific observable distributions, some previous contributions have also considered more generally defined LVMs for unsupervised learning. For instance, \citet[][]{CollinsEtAl2002}, \citet[][]{MohamedEtAl2009} or \citet[][]{LiTao2010} considered a PCA generalization to include observables of the exponential family (referred to as EF-PCA), and \citet{lee2009exponential} considered SC with exponential family observables (EF-SC).
Parameter optimization for specific non-Gaussian distributions is already considered more challenging than for Gaussian distribution. If LVMs are defined
more generally, optimization becomes still more challenging as any method has to be defined without exploiting properties that only apply for a specific distribution.

Besides the approaches mentioned above,  another class of models which is very general by considering exponential family distributions for observables is represented by Generalized Linear Models \citep[GLMs;][]{NelderWedderburn1972,McCullaghNelder1989}. 
GLMs can be related to the aforementioned approaches (and also to our work here) since they likewise exploit other noise distributions besides the Gaussian and do provide generic parameter update equations also for the case of non-Gaussian distributions \citep[e.g.][]{McCullaghNelder1989}.  
Nonetheless, GLMs essentially rely on supervised learning, i.e.\ regression, which is not the subject of this study \citep[see][for a discussion]{CollinsEtAl2002}.

In this study, we consider a class of LVMs for unsupervised learning which is like EF-PCA or EF-SC more generally defined to
include observable distributions of the exponential family. 
Unlike EF-PCA or EF-SC, however, we will seek an optimization procedure that is generally applicable to any regular exponential family distribution. 
%
%
In order to achieve our goal of a generic parameter optimization for any observable distribution of the exponential family,  we replace the common linear summation of latents by a maximization. More precisely, for the LVMs considered here, the latents determine the \textit{mean} of the observable distribution through maximization in contrast to the conventional Gaussian-based LVMs such as p-PCA or SC (and also GLM-related approaches) that determine the mean using summation.
In addition, our approach maintains the same influence of latents on the observable mean (using a maximum non-linearity) independently of the specific choice for the observable distribution.
%
For approaches such as EF-SC, a weighted sum of latents is used to set the natural parameter(s) of the (usually single parameter) exponential family distribution. 
Doing so implies a general non-linear function for the distribution mean, and this function changes depending on the used exponential family distribution.    
Also in the context of regression, GLMs similarly use a weighted summation followed by a non-linear function to set the distribution mean. 

Models with a non-linear superposition of the latents have been of interest for a variety of tasks, and were actively researched in previous works.  
For instance,  the combination of noisy-OR Bayes nets can be considered as such a non-linear link as well
as point-wise sigmoidal functions after summation as used in SBNs \citep[][]{SaulEtAl1996,GanEtAl2015}. Still further forms of models with a specific latent variable structure have been investigated in the context of non-linear Independent Component Analysis \citep[ICA;][]{HyvarinenPajunen1999} and deep generative models such as Variational Autoencoders \citep[VAEs;][]{KingmaWelling2014}, generative adversarial nets, flow nets etc. \citep[e.g.][]{BondTaylerEtAl2022}.
Standard (linear) ICA approaches \citep[][]{HyvarinenOja2000,Comon1994} investigate continuous non-noisy observables while (similar to standard SC) non-Gaussian latents are considered.
Later developments have significantly generalized earlier restrictions with recent work allowing, e.g., for binary observations \citep[][]{HyttinenEtAl2022} and non-linear influences of the latents on the observables \citep[][]{HyvarinenPajunen1999}.
For non-linear ICA approaches, non-linearities can take either the form of a post-linear non-linearity (i.e., a linear superposition of generative fields followed by an invertible non-linear function; \citet[][]{TalebJutten1999}),  or,  more generally, the form of a general smooth and invertible function that can be used in place of the sum \citep[][]{SprekelerEtAl2014,HyvarinenMorioka2016,HyvarinenMorioka2017,HyvarinenEtAl2019}. 
%
%
Non-linear ICA and variational autoencoders as presumably most relevant types of deep models will be discussed further in Sec.~\ref{Relation_to_Previous_Work}, i.e.\,after the definition of the proposed LVMs.  


Like the above types of non-linearities, the here considered fixed non-linearity in the form of a (weighted) maximization, 
%
%
has also been investigated previously. 
For instance, \citet[][]{Roweis2003}, \citet[][]{SheltonEtAl2012}, \citet[][]{SheltonEtAl2017} or \citet[][]{SheikhEtAl2019} used maximization for LVMs with observables distributed according to standard Gaussian distributions, and \citet[][]{LuckeSahani2008} used the maximum for an LVM with Poisson observables. 
Moreover, similar fixed non-linear superpositions have been used to model mutual occlusions of objects in images \citep[][]{LuckeEtAl2009,HennigesEtAl2014}. The approach named Occlusive Component Analysis (OCA) likewise uses Gaussian observables (also see \citet[][]{DaiLucke2014} for another relevant study in this direction).   
Like the work by \citet[][]{LuckeSahani2008} or \citet[][]{SheikhEtAl2019}, we will use the fixed maximum non-linearity in this study. But
unlike such previous works, we will not use the non-linearity together with only one specific observable distribution. Instead,
we will consider LVMs which encompass any (regular) member of the exponential family as observable distribution. 

After defining the general class of LVMs, we will address the challenge of parameter optimization by following a maximum likelihood approach using Expectation Maximization \citep[EM;][]{DempsterEtAl1977}. 
Our central analytical result will then be a set of concise parameter update equations (M-steps) that are generally and directly applicable to any regular member of
the exponential family. 
Furthermore, the here derived M-steps can be combined with likewise generally applicable and efficient variational E-steps using a recently suggested variational approach \citep[][]{DrefsEtAl2022}. Neither the generic M-steps derived in this contribution nor the variational E-steps of \citet[][]{DrefsEtAl2022} require
any analytical derivations before they can be applied. If generic M-steps and variational E-steps are combined, we then obtain an efficient EM optimization algorithm that is, without any derivations, applicable to any observables that follow a regular exponential family distribution\footnote{Generic applicability does by itself not mean that the EM algorithm derived here will necessarily be based entirely on closed-form equations. Even estimating the parameters of many exponential family distributions is itself often not 
possible  
in closed-form (e.g. for Gamma or Beta distributions).
But {\em if} the parameters of the exponential family distribution can be estimated in closed-form, then EM for the corresponding LVM will be closed-form.
%
If the parameters of the exponential family distribution can {\em not} be estimated in closed-form, then this will turn out to be the only part requiring approximations (which are usually readily available). See end of Sec.\,\ref{generalEF_MCA} for details.}. Finally, we provide numerical evaluations for one- as well as two-parameter distributions of the exponential family and point to some potential applications.

\section{Preliminary Definitions}
%
Before defining the family of LVM models, we provide details on exponential family distributions which will be used throughout this paper. We use observable distributions $p(y; \vec{\eta}\,)$ that take the form of an exponential family distribution given by: 
\begin{align}\label{Exp_pdf}
p(y; \vec{\eta}\,) = h(y) \exp\big(\vec{\eta}^{\, T} \vec{T}(y) - A(\vec{\eta}\,) \big), \quad y\in \mathcal{Y}, 
\end{align}
where $\mathcal{Y}$ is the domain of a scalar observable $y$, $h(y)$ is the base measure, $\vec{T}(y)$ represents the \textit{sufficient statistics} of the data, and $\vec{\eta}$ and $A(\vec{\eta}\,)$ are the \textit{natural parameters} and \textit{log-partition} function, respectively. The vectors $\vec{T}(y) = (T_1(y), \ldots, T_L(y))^T$ and $\vec{\eta} = (\eta_1, \ldots, \eta_L)^T$ have $L$ elements each when the distribution $p(y; \vec{\eta}\,)$ is an $L$-parameter distribution. 
We remark that a notational alternative is provided by using  
%
%
$p(y\,|\,\vec{\eta}\,)$ in \eqref{Exp_pdf} instead of $p(y; \vec{\eta}\,)$.  However, we here use the latter notation (which describes $p(\cdot)$ as a function of $y$ with given parameters $\vec{\eta} \,$) for the sake of a better readability. We follow the same argumentation for the other probability density functions used in the paper (including those of Eqns.~\eqref{EqnSCPrior} and \eqref{EqnSCNoise}).

%
%
The log-partition $A(\vec{\eta}\,)$ plays the role of a normalizer. Following from \eqref{Exp_pdf}, it is for a given $\vec{\eta}$ defined as:
\begin{align}\label{log-partition}
A(\vec{\eta}\,) = \log \Big(\int h(y) \exp\big( \vec{\eta}^{\, T} \vec{T}(y)\big) dy \Big).
\end{align}
Formally, the domain of $A$ is usually defined to be the set of all natural parameters $\vec{\eta}$ for which $A(\vec{\eta}\,) < \infty$.
%
We will only consider exponential families that are \textit{regular}, i.e., families
for which the set of all natural parameters $\vec{\eta}$ is a non-empty and open set \citep[e.g.][]{BickelDoksum2015,Efron2023}. Furthermore, all integrals over the domain
${\cal Y}$ can be general Lebesgue integrals but we will not use Lebesgue measures explicitly.
For the purposes of this paper, the integrals can be considered as using the ordinary Lebesgue measure on $\mathbb{R}$ in the continuous case, while in the discrete case integrals become a summation over finite probabilities \citep[compare, e.g.,][]{wainwright2008graphical}.
%
%
%
For any regular exponential family, and for any $l,k \in \{1,\ldots,L\}$, we have \citep[e.g.][]{BickelDoksum2015,Efron2023}:
\begin{align}\label{A_eta_expected_value}
\frac{\partial A(\vec{\eta}\,)}{\partial \eta_l} = 
\mathbb{E}_{p(y;\,\vec{\eta}\,)}[T_l(y)], \quad
\frac{\partial^2 A(\vec{\eta}\,)}{\partial \eta_l \partial \eta_k} = 
\mathrm{cov} \big( T_l(y), T_k(y) \big),
\end{align}
where $\mathbb{E}_{p(\cdot)}[f(y)]$ denotes the expectation value of $f(y)$ w.r.t.\ the distribution $p(\cdot)$ (evaluated per entry if $f$ is a vector) and $\mathrm{cov}(\cdot, \cdot)$ represents the covariance. 
Also note that an alternative notation for the expectation value is $\langle f(y) \rangle_{p(\cdot)}$ which has been used in previous works \citep[e.g.,][]{LuckeSahani2008,SheikhEtAl2014}. 

As an important tool to later derive parameter update equations, we will use the {\em mean value parametrization} (a.k.a.\ the \textit{mean parametrization}) of the exponential family \citep[e.g.][]{wainwright2008graphical,Cox2006,Efron2023}. 
That is, we consider parameters $\vec{w}$ (as the model's parameters) defined by:
\begin{align}\label{mean_value_parameters}
\vec{w} := \mathbb{E}_{p(y;\,\vec{\eta}\,)}[ \vec{T}(y) ], \quad \vec{w}=(w_1,\ldots,w_L)^T.     
\end{align}
%
 %
The mapping from $\vec{\eta}$ to $\vec{w}$ is bijective, i.e. it is invertible, if the distribution has a \textit{minimal representation}\footnote{A distribution of the exponential family is said to be minimal if there are no linear dependencies among components of the sufficient statistics,
i.e., there is no coefficient $\vec{a}\in\mathbb{R}^L$ with $\vec{a}\neq \vec{0}$ such that 
\begin{equation*}
\sum_{l=1}^L a_l T_l(y) = \mathrm{constant} \quad \mbox{for\, all} \ y \in \mathcal{Y}.
\end{equation*}} property \citep[compare][and many others]{wainwright2008graphical,Efron2023}.
Assuming minimality from now on, there exists a function $\vec{\Phi}$ which maps the mean value parameters to the natural parameters as follows:
%
%
\begin{align}\label{phi_function}
\vec{\eta} = \vec\Phi(\vec{w}) \ \mbox{\ \ such that\ \ }\ \vec{w} = \mathbb{E}_{p(y;\, \vec{\Phi}(\vec{w}))}[ \vec{T}(y) ].
\end{align}
For some distributions of the exponential family the function $\vec{\Phi}$ can be obtained in closed-form (e.g., for Poisson, Bernoulli and Gaussian). For others, however, a closed-form expression may in general not be available. Still, the existence of the inverse function $\vec{\Phi}$ is guaranteed as long as the distribution has a minimal representation property. 
In the next section, we will show how the function $\vec{\Phi}$ along with other properties presented above will be exploited to define a specific family of generative models for exponential family observables. 


\section{Proposed Family of Latent Variable Models}
%
%
%
%
\label{sec:nonlinear_sc_models}

We hereafter consider LVMs with the same notation as in Eqns.~(\ref{EqnSCPrior})-(\ref{EqnSCNoise}). 
However, we assume the latents to be binary (i.e. each element $s_h \in \{0, 1\}$), and for the sake of consistency with previous work, we assume each $s_h$ to follow a Bernoulli distribution.  
Using small values of Bernoulli parameters, the model can thus realize a sparse encoding of the latents. 
We remark, nonetheless, that the analytical results derived in the following 
only require binary latents such that the results will also apply for non-sparse encodings and priors more complex than $H$ independent Bernoulli distributions. 
As a further assumption we assume, similar to the standard SC, that all observables $y_d$ are conditionally independent given the latents.
Unlike standard SC, however, the conditional distribution for the observables can follow any regular distribution of the exponential
family \eqref{Exp_pdf} with natural parameters depending on $\vec{s}$.  
Finally, and importantly, the observables in our model class depend on the latents through maximization.

\subsection{Model Definition}

We investigate LVMs according to the following generative model:
\vspace{-3ex}

\begin{align} 
p(\vec{s}\,|\, \Theta) &= \prod_{h=1}^H \big( \pi_h^{s_h} (1 - \pi_h)^{1 - s_h} \big) , \label{modelEq1}\\
p(\vec{y}\,|\, \vec{s}, \Theta) &= \prod_{d=1}^D p\big(y_d; \vec{\tilde{\eta}}_d(\vec{s},\Theta)\big), \ \ \mbox{where}\label{modelEq2}\\
p(y; \vec{\eta}\,\big) &= h(y)\exp\big( \vec{\eta}^{\, T} \vec{T}(y) - A(\vec{\eta}\,) \big) , \label{modelEq3}
\end{align} 
and where $\pi_h \in (0, 1)$ parameterizes the prior distribution on $s_h$; we let $\vec{\pi} = (\pi_1, \ldots, \pi_H)^{\mathrm{T}}$.  
For model (\ref{modelEq1})-(\ref{modelEq3}), the latent variables $\sVec$ couple to the observed
variable $y_d$ through the function $\vec{\tilde{\eta}}_d(\vec{s},\Theta)$ which is defined further below (the function can be considered as a link function in a broader sense). The observable distribution (\ref{modelEq3}) will in the following also be referred to as {\em noise distribution}.

Note that $\vec{\tilde{\eta}}_d(\vec{s},\Theta)$ is a vector with $L$ entries for each $d$.
For one-parameter distributions of the exponential family (Poisson, Exponential, Pareto etc.) $L=1$, for two-parameter distributions (Gamma, Beta etc.) $L=2$, and so forth (e.g., Categorical and Multinomial distributions can be mentioned for the case of arbitrary $L$).
For standard SC, a Gaussian with known variance is commonly considered 
(just the mean is parameterized).  
Standard SC then sets the {\em mean} of the observables using a matrix $W\in\mathbb{R}^{D \times H}$, see Eqns. \eqref{EqnSCPrior} and \eqref{EqnSCNoise}. Here, we seek to couple latents and observables in an analog but more generally applicable way. 

In order to facilitate our notation, let us from now on assume two-parameter distributions ($L=2$) as it applies, e.g., for Gamma, Beta and Gaussian distributions. Arbitrary $L$ will be treated later on. 
For the generative model presented above, we will for each observable $y_d$ parameterize the distribution \eqref{modelEq2} in terms of the 
%
%
 mean value parameters $\vec{w}=(w_1,w_2)^{\mathrm{T}}$ (Eqn.~\ref{mean_value_parameters}). 
 For this purpose and for the $L=2$ case, let $w_1=\bar{W}_d$ and $w_2=\bar{V}_d$ for each $d$, where
$\bar{W}_d$ and $\bar{V}_d$ are functions depending on the 
latents $\sVec$ and \textit{two} matrices $W$ and $V$ with $D \times H$ entries. 
Linear and non-linear combination of latents to determine the observable distributions will be possible
through different choices for the functions $\bar{W}_d$ and $\bar{V}_d$ (see below).
%
The notation using $L=2$ serves for gaining some intuition as (e.g., for Gaussian or Gamma) $\bar{W}_d$ may denote the distribution mean of observable $y_d$ and $\bar{V}_d$ the distribution variance (or the second moment of the distribution). Figure~\ref{fig_Network} illustrates the structure of the proposed model where, in the case of $L = 2$, two matrices connect the latents to the observables. Now, using the function $\vec{\Phi}(\vec{w})$ given by (\ref{phi_function}), we define the link from latents to observables as follows:
%
%
\begin{figure}[t!]
\centering
\includegraphics [scale=0.3]{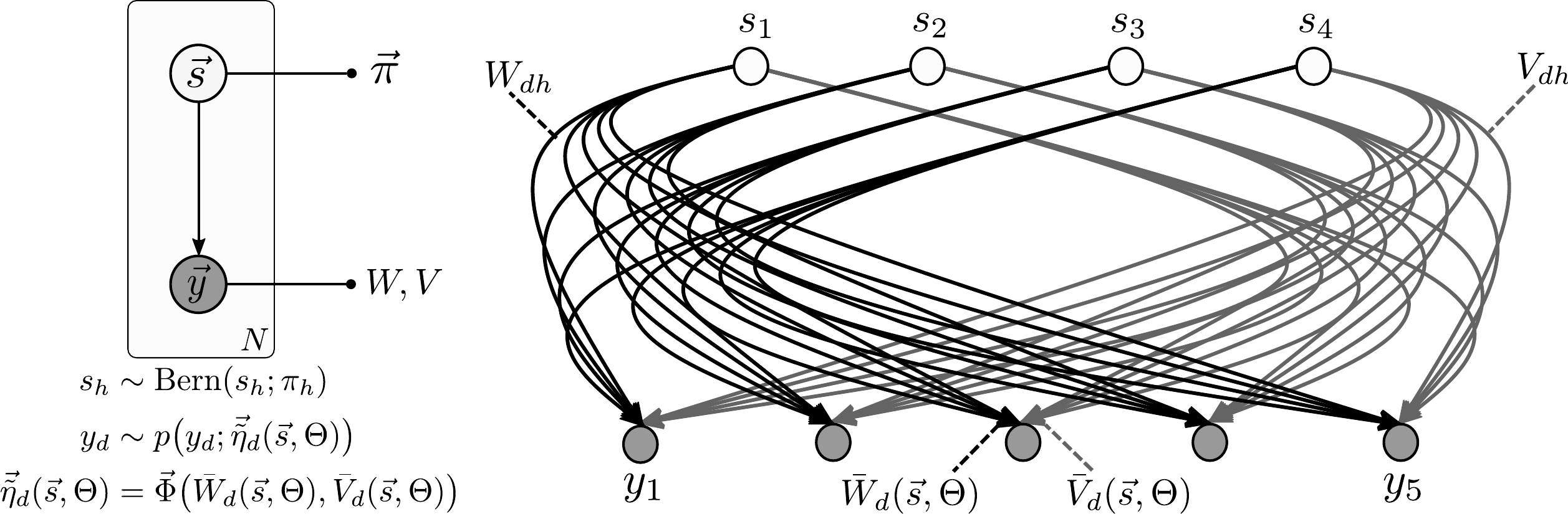}
  \caption{Graphical representation of the proposed generative model with (as an example) $H = 4$ latent variables and $D = 5$ observed variables. Given a binary vector $\vec{s}$, the value $y_d$ of an observable follows an exponential family distribution $p(y_d; \vec{\tilde{\eta}}_d(\vec{s},\Theta) \big)$ with natural parameters
  determined by the functions $\bar{W}_d(\vec{s}, \Theta)$ and $\bar{V}_d(\vec{s}, \Theta)$ for the distribution's mean value parameters.
The functions depend on the latents and on the matrices $W\in\mathbb{R}^{D\times{}H}$ and $V\in\mathbb{R}^{D\times{}H}$ visualized at the right-hand-side. Here, the set of all model parameters is $\Theta=(\vec{\pi},W,V)$
 and $N$ denotes the number of samples.
%
%
}
\label{fig_Network}
\end{figure}
%
%
%
\begin{align}\label{eta}
\vec{\tilde{\eta}}_d(\vec{s},\Theta) := \vec{\Phi}\big(\bar{W}_d(\vec{s}, \Theta), \bar{V}_d(\vec{s}, \Theta)\big), \quad d = 1, \ldots, D.
\end{align}

The mapping $\vec{\Phi}$, as mentioned before, is specific to the choice of observable distribution \eqref{modelEq3} which is a regular exponential family distribution. Although a closed-form formulation of 
the function is only available for some instances of the family, 
the function is well-defined as long as the noise distribution \eqref{modelEq3} has a minimal representation property (the assumption that we consider here). Thus, such a general link function is well-defined for the proposed data models using $\bar{W}_d(\vec{s},\Theta)$ and $\bar{V}_d(\vec{s},\Theta)$ (also compare Fig.~\ref{fig_Network}).
For the case of arbitrary $L$, we will define this link using $L$ parameters and subsequently $L$ matrices (see Sec.~\ref{generalEF_MCA} further below). To complete the definition of our generative models, it now remains to define the functions $\bar{W}_d(\vec{s},\Theta)$ and $\bar{V}_d(\vec{s},\Theta)$ 
which determines the influence of latents on the observable distributions.
Before that, let us elaborate the notation above using a preeminent example of the exponential family distributions.

\begin{example}
\label{example1}
To provide an intuition for the link defined by (\ref{eta}), consider Gaussian observation noise. In that case, sufficient statistics and natural parameters are given by $\vec{T}(y) = (y, y^2)^T$ and $\vec{\eta} = (\frac{\mu}{\sigma^2}, \frac{-1}{2 \sigma^2})^T$ (where $\mu$ and $\sigma^2$ are the mean and variance of the Gaussian distribution, respectively). By expressing the mean value parameters $\vec{w}$ in terms
of the natural parameters $\vec{\eta}$, we get:
\begin{align}\label{Gaussian_ex1}
\vec{w}=\left(\begin{array}{c} \mu \\ \mu^2+\sigma^2 \end{array}\right) = \frac{1}{4\eta_2^2} \left(\begin{array}{c} -2\eta_1\eta_2\\ \eta_1^2-2\eta_2 \end{array}\right).
\end{align}
Given the equation above, the inverse mapping $\vec{\Phi}$ can be computed in closed-form (in this case) and is given by:
\begin{align}\label{Gaussian_ex2}
\vec{\Phi}(\vec{w} ) = \frac{1}{  2(w_2 - w_1^2)} \left(\begin{array}{c} 2w_1 \\ -1 \end{array}\right). 
\end{align}
Using the definition of the link function in (\ref{eta}), the coupling of latents to observables is consequently given by:
\begin{align}\label{etaexamplePre}
\vec{\tilde{\eta}}_d(\vec{s},\Theta) &= \vec{\Phi}\big(\bar{W}_d(\vec{s}, \Theta), \bar{V}_d(\vec{s}, \Theta)\big)\\
&= \frac{1}{2 \big(\bar{V}_d(\vec{s}, \Theta) - \bar{W}^2_d(\vec{s}, \Theta)\big)} \left(\begin{array}{c} 2\,\bar{W}_d(\vec{s}, \Theta) \\ -1 \end{array}\right). \nonumber
\end{align}
In order to recover standard linear SC, we can complete the definition of $\vec{\tilde{\eta}}_d(\vec{s}, \Theta)$ by using parameters $(\sigma^2,W)$ with
a global, scalar $\sigma^2\in\mathbb{R}^+$ and $W\in\mathbb{R}^{D \times H}$ and further setting:
\begin{align}\label{etaexampleStandardSC}
\bar{W}_d(\vec{s}, \Theta) = \sum_h W_{dh}s_h\,\phantom{iii}\mbox{and}\phantom{iii}
\bar{V}_d(\vec{s}, \Theta) = \sigma^2 + \bar{W}_d^2(\vec{s}, \Theta), 
\end{align}
which then results in:
\begin{align}\label{etaexampleStandardNormSC}
\vec{\tilde{\eta}}_d(\vec{s}, \Theta) = \left(\begin{array}{c} (W\vec{s} \,)_d / \sigma^2 \\[1mm] - 1 / 2 \sigma^2 \end{array}\right).
\end{align}
That is, we recover the standard SC parametrization with $\mu_d=(W\vec{s}\,)_d$ as the mean of observable $d$ and $\sigma^2$ as its variance (with $\sigma^2$ being the same for all $d$). 
\end{example}

\subsection{Definition of a Maximum Superposition Model}
%

In Example~\ref{example1}, given a
matrix $W\in\mathbb{R}^{D \times H}$, the mean $\mu_d$ of observable $d$ was given by $(W\vec{s}\,)_d$. Our definition of the
link (\ref{eta}) from latents to observables is sufficiently flexible to also allow for types of latent
combinations other than the weighted linear sum that is very commonly used (e.g., for \mbox{p-PCA}, sparse coding etc).
In this respect, let us first ask the question what aspects of a given observable distribution the two functions $\bar{W}_d$ and $\bar{V}_d$ shall determine. Most commonly, the latents determine the {\em mean} of observable $d$ using summation (again as, e.g., for \mbox{p-PCA}, sparse coding etc). 
For many members of the exponential family, the first mean value parameter coincides with the mean of the corresponding exponential family distribution (that is
we have $T_1(y)=y$). Note that if there exists an $l$ with $T_l(y)=y$ (which is not always the case), we will assume this sufficient statistics'
component to be listed at the first position ($l=1$). If $T_1(y)=y$, then (following from our definition) the function $\bar{W}_d(\vec{s}, \Theta)$ alone directly determines the observable mean:
\begin{align}
\mu_d &=\, \mathbb{E}_{p ( y;\,\vec{\tilde{\eta}}_d(\vec{s},\Theta) ) } [ y ] =\, \mathbb{E}_{ p ( y;\,\vec{\Phi} ( \bar{W}_d(\vec{s},\Theta), \bar{V}_d(\vec{s},\Theta) ) ) } [ y ] \cr
&=\, \mathbb{E}_{ p(y;\,\vec{\Phi} ( \bar{W}_d(\vec{s}, \Theta), \bar{V}_d(\vec{s}, \Theta) ) ) }  [ T_1(y) ] = \bar{W}_d(\vec{s}, \Theta), \label{basicWa}
\end{align}
%
where the last step follows from (\ref{phi_function}) and the fact that $w_1 = \bar{W}_d (\vec{s}, \Theta)$.  
The function $\bar{W}_d(\vec{s}, \Theta)$ can then be either a weighted summation $\bar{W}_d(\vec{s}, \Theta)=(W\vec{s}\,)_d$ as in Example~\ref{example1} or a non-linear function of the latents.
For our purposes, we will follow previous work \citep[][]{LuckeSahani2008,BornscheinEtAl2013,SheikhEtAl2019} and demand that the mean of an observable distribution be given by a weighted maximum instead of a weighted sum. More concretely, we will replace (e.g., for the Gaussian or Gamma distribution):
\begin{align}\label{maxBasic}
\ \nonumber\\[-3ex]
 \bar{W}_d(\vec{s},\Theta) = \textstyle(W\vec{s}\,)_d = \sum_h W_{dh}s_h\phantom{ii}\mbox{by}\phantom{ii}\bar{W}_d(\vec{s},\Theta) = \textstyle\max_h\{ W_{dh}s_h \}. 
\end{align}
Figure \ref{fig_linearMaxCombination} illustrates the differences between weighted summation and weighted maximization 
of the latents using artificial generative fields in the form of bars. Maximization instead of summation has been suggested before for acoustic data \citep[][]{Roweis2003,LuckeSahani2008,SheikhEtAl2019} and for visual data \citep[][]{BornscheinEtAl2013} where it aligns more closely with the actual data generating process. As summation of generative fields is also referred to as a {\em linear superposition} of generative fields, maximization can be considered as a {\em non-linear superposition}. In the following, summation and maximization are, respectively, referred to linear or maximum {\em superposition models}.
%
%
%
\begin{figure}[t!]
\centering
\includegraphics [scale=0.5]{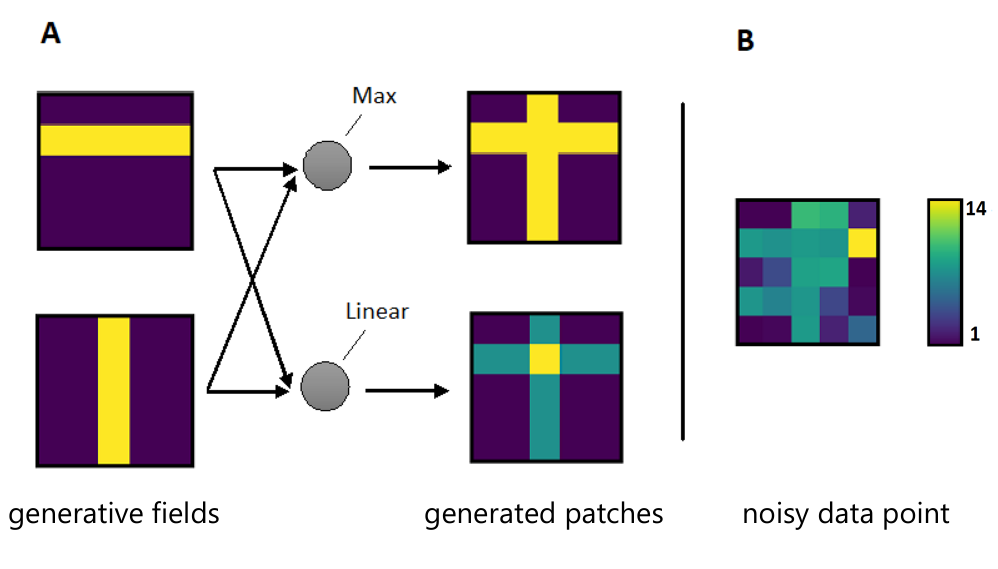}
  \caption{Illustration of non-linear vs. linear superposition of generative fields. \textbf{A} Two generative fields in the form of bars are combined through
  maximization and summation to generate data. \textbf{B} Example of a noisy data point which is generated using the Exponential-MCA model (i.e.,  we used the patches generated by the maximum superposition in \textbf{A} together with Exponential noise).}
\label{fig_linearMaxCombination}
\end{figure}
%
Importantly for the purposes of this study, the specific properties of maximization will enable us to derive a generic parameter optimization
procedure for the family of data models \eqref{modelEq1}-\eqref{modelEq3}.

Equation (\ref{maxBasic}) would cover members of the exponential family such as Gaussian and Gamma (as $T_1(y) = y$ holds for these distributions). However, to be applicable to the whole exponential family, the definition of the maximum superposition has to be further generalized. 
%
%
The challenge is that, in general, the sufficient statistics $\vec{T}(y)$ may not have an element that is proportional to $y$,
and also a definition for the function $\bar{V}_d(\vec{s}, \Theta)$ is required. As an example which makes a generalization necessary
consider the Beta distribution where $T_1(y) = \log(y)$ and $T_2(y) = \log(1-y)$. As a consequence, the definition
of the maximum superposition has to involve both functions $\bar{W}_d(\vec{s},\Theta)$ and
$\bar{V}_d(\vec{s},\Theta)$. Nevertheless, it is possible to define $\bar{W}_d(\vec{s},\Theta)$ and 
$\bar{V}_d(\vec{s},\Theta)$ such that the mean of the observables is determined by a maximum superposition. 
For this purpose, consider a weight matrix $M(\Theta)$ with $D \times H$ entries which can potentially be a non-trivial function of the parameters~$\Theta$. Given a latent vector $\vec{s}$, we now require the mean $\mu_d$ of an observable $d$ to be given by:
\begin{align}\label{basicWc}
\mu_d &= \mathbb{E}_{ p(y;\, \vec{\Phi} ( \bar{W}_d(\vec{s}, \Theta), \bar{V}_d(\vec{s}, \Theta) ) ) }  [ y ] \stackrel{!}{=} \mathrm{max}_{h} \{M_{dh}(\Theta)\, s_h\}.
\end{align}
Clearly, if $T_1(y)=y$, we can satisfy the demand by choosing $M_{dh}(\Theta)=W_{dh}$ and further setting $\bar{W}_d(\vec{s}, \Theta)=\textstyle\max_h\{ W_{dh}s_h \}$. From (\ref{basicWa}), we can consequently obtain:
\begin{align}
\mu_d \,=\, \mathbb{E}_{ p(y;\,\vec{\Phi} ( \bar{W}_d(\vec{s}, \Theta), \bar{V}_d(\vec{s}, \Theta) ) ) }  [  y ] \,=\, \bar{W}_d(\vec{s}, \Theta) 
      \,=\, \hspace{0ex}\textstyle\max_h\{ W_{dh}s_h \} \,=\, \mathrm{max}_{h} \{M_{dh}(\Theta)\, s_h\}, \label{basicWd}
\end{align}
which fulfils (\ref{basicWc}). 
%
For the general case, nonetheless, the definition of $M(\Theta)$ has to be more elaborate in order to apply for any exponential family distribution.
As mentioned before, we consider two $D \times H$ matrices $W$ and $V$ as part of the model parameters $\Theta$, i.e. $\Theta = (\vec{\pi}, W, V)$. We can then define $M(\Theta)$ as follows:
\begin{align}
\forall d,h:\ \ M_{dh}(\Theta) :=\, F(W_{dh},V_{dh}) \phantom{ii}\mbox{where}\phantom{ii} F(w,v) =\, \mathbb{E}_{p(y;\,\vec{\Phi}(w,v))}  [ y ]. \label{meanf_eq}
%
\end{align}  
%
%
%
Using matrix $M(\Theta)$, we now define our mappings $\bar{W}_d(\vec{s},\Theta)$ and $\bar{V}_d(\vec{s},\Theta)$ as follows:
\begin{align}\label{maxh}
\bar{W}_d(\vec{s},\Theta) &:= W_{dh(d,\vec{s}, \Theta)},\ \ \bar{V}_d(\vec{s},\Theta) := V_{dh(d,\vec{s}, \Theta)}, \\[1mm]
\mbox{where}\phantom{iiii}h(d,\vec{s}, \Theta) &:= \mathrm{argmax}_{h} \{M_{dh}(\Theta)\, s_h\}. \nonumber
\end{align}
That is, we define the functions $\bar{W}_d(\vec{s},\Theta)$ and $\bar{V}_d(\vec{s},\Theta)$ via an index $h(d,\vec{s}, \Theta)$
which selects, for a given observable $d$, the latent with the maximal value of $M_{dh}(\Theta)$.
Using the relations in (\ref{meanf_eq}) and (\ref{maxh}) for the link (\ref{eta}) completes the definition of the non-linear superposition model.
The model represents a generalization of the superposition model (\ref{basicWd}) suitable for the whole exponential family. To see this, we derive:
%
%
\begin{align}
\mu_d &= \mathbb{E}_{p(y; \, \vec{\tilde{\eta}}_d(\vec{s},\Theta))}  [ y ] \nonumber
      = \mathbb{E}_{ p(y; \, \vec{\Phi} ( \bar{W}_d(\vec{s},\Theta), \bar{V}_d(\vec{s},\Theta) ) ) }  [ y ] \\
      &= F( \bar{W}_d(\vec{s},\Theta), \bar{V}_d(\vec{s},\Theta) ) 
      = F( W_{dh(d,\vec{s},\Theta)}, V_{dh(d,\vec{s},\Theta)} ) \nonumber\\
      &= M_{dh(d,\vec{s},\Theta)}(\Theta) = \textstyle\max_h\{ M_{dh}(\Theta)\,s_h \},
\end{align}
which again fulfils (\ref{basicWc}) but this time for the general case. 
In virtue of (\ref{meanf_eq}) and (\ref{maxh}), we do {\em not} have to require that there exists a sufficient statistic $T_1(y)=y$,
i.e., no further restrictions to an exponential family distribution have to be demanded. 
%
If $T_1(y)=y$, then definition (\ref{meanf_eq}) results in $F(w,v)=w$.
Consequently, we drop back to the easier special case of $M_{dh}(\Theta)=W_{dh}$.


The definition of functions $\bar{W}_d(\vec{s},\Theta)$ and $\bar{V}_d(\vec{s},\Theta)$ are relatively technical but allow
for a general definition of generative models using any exponential family (with a minimal representation property). Our LVM defines
%
%
a link $\vec{\tilde{\eta}}_d(\vec{s},\Theta)$ from latents to observables for the proposed data models which is
a consistent generalization of the maximum non-linearity (\ref{basicWd}) used, e.g., for Poisson \citep[][]{LuckeSahani2008} and Gaussian \citep[][]{LuckeEggert2010} observables. 
In our case, the proposed link remains consistently defined for all noise
distributions of the exponential family. 
That is, 
independently of the choice of noise distribution, the link $\vec{\tilde{\eta}}_d(\vec{s},\Theta)$ defined by 
(\ref{meanf_eq}) and (\ref{maxh}) ensures that the latents change the \textit{mean} of the observables always according to the maximum
superposition model. 
This is in contrast to the previous LVMs for exponential family distributions such as EF-PCA \citep[][]{CollinsEtAl2002}, Bayesian EF-PCA \citep[][]{MohamedEtAl2009} and EF-SC \citep[][]{lee2009exponential} 
where the latents set the natural parameters of the observables using  
a linear superposition model (we will elaborate on these studies further below). 

%
To summarize, the definition of the proposed family of generative models is given by 
(\ref{modelEq1})-(\ref{modelEq3}) and link (\ref{eta}) with (\ref{meanf_eq}) and (\ref{maxh}). In analogy to previous models defined using the maximum non-linearity \citep[][]{LuckeSahani2008}, we will refer to the family of models as {\em Exponential Family Maximal Causes Analysis} (EF-MCA). 
While the EF-MCA data models are very general, we will later see that the chosen parametrization results in generic equations for parameter updates. Before that, we will illustrate the above technical definitions using again the Gaussian distribution as an example.


\begin{example}
\label{example2}
Consider the Gaussian case of Example~\ref{example1}. For the corresponding EF-MCA model, the matrix $M(\Theta)$ is (because of $T_1(y)=y$) given by $M_{dh}(\Theta)=W_{dh}$ and
\begin{align}\label{etaexampletwoA}
\bar{W}_d(\vec{s}, \Theta) &= W_{dh(d,\vec{s},\Theta)} = \mathrm{max}_h \{ W_{dh}s_h \} \\
\bar{V}_d(\vec{s}, \Theta) &= V_{dh(d,\vec{s},\Theta)}.\label{etaexampletwoB}
\end{align}
Using Eqn. (\ref{phi_function}) and the $\vec{\Phi}$ function (\ref{Gaussian_ex2}), we again obtain $\vec{\tilde{\eta}}_d(\vec{s}, \Theta)$
similar to (\ref{etaexamplePre}) but this time with $\bar{W}_d(\vec{s}, \Theta)$ and $\bar{V}_d(\vec{s}, \Theta)$ given by (\ref{etaexampletwoA}) and (\ref{etaexampletwoB}), respectively.
%
%
In terms of the standard Gaussian parametrization, we would thus obtain as noise model:
\begin{align}
p(y_d;\vec{\tilde{\eta}}_d(\vec{s}, \Theta)) &= \mathcal{N}(y_d;\bar{W}_d(\vec{s}, \Theta),\bar{V}_d(\vec{s}, \Theta)-\bar{W}^2_d(\vec{s}, \Theta)). \nonumber\vspace{-1ex}
%
\end{align}
The matrix $V\in\Theta$ thus allows for the latents to also parameterize the variance. We will later see how this parametrization can also be changed back to the
normal parametrization using a matrix $\sigma^2_{dh}=V_{dh}-W^2_{dh}$ instead of $V_{dh}$ 
(using $\sigma^2_{dh}$ is more closely aligned with the conventional parametrization). Note, however,  
that we here use a variance element $\sigma^2_{dh}$ for each $d$ and $h$ which is in contrast to a global scalar variance considered in \eqref{EqnSCPrior}-\eqref{EqnSCNoise}.  
A generalization assuming a matrix~\big(with elements $\sigma^2_{dh}$\big), to the best of our knowledge, has not been considered before and the data model presented in (\ref{modelEq1})-(\ref{modelEq3}) together with (\ref{eta}), (\ref{meanf_eq}) and (\ref{maxh}) is the first non-linear LVM which allows for the training of two matrices (one corresponding to the mean of observables, $W$, and another which parameterizes the variances, $V$). This feature makes the data model more flexible compared to previous approaches (e.g. \cite{LuckeEggert2010,PuertasEtAl2010}) and provides more information about the given dataset. The generalized Gaussian model has previously been applied to model the statistics of natural
image patches \citep[][]{mousavi2020ddDictionary} where standard Gabor-like generative fields were obtained together with variance profiles
per latent variable. The new statistical properties captured by the novel matrix $V$ can predict richer response properties of simple cells in primary visual cortex (V1) than conventional modeling using one matrix $W$ for the mean responses (we refer the reader to \citet[][]{mousavi2020ddDictionary} and Sec.~\ref{sec_natural_image_patches} for more information). 
If we wanted to enforce scalar variances (as conventionally used), we could define $\bar{V}_d(\vec{s}, \Theta)$ as in (\ref{etaexampleStandardSC}). 
\end{example}

\section{Relation to Previous Work on Gaussian and Non-Gaussian Observables}
\label{Relation_to_Previous_Work}
%
It has been the goal of many previous approaches
to derive learning algorithms for non-Gaussian data \citep[see, e.g.,][]{CollinsEtAl2002,MohamedEtAl2009,lee2009exponential,mohamed2011bayesian,giryes2014sparsity,lu2016sparse,ZhouEtAl2012,salmon2014poisson,ChiquetEtAl2018,tolooshams2020convolutional}.
Often the derivation of an update procedure for just a specific non-Gaussian distribution is aimed at \citep[e.g., by assuming Poisson distributed observables for count data,][]{salmon2014poisson}. However, a subset of these approaches do aim at more generally deriving updates allowing for a range of different non-Gaussian observable distributions \citep[see, e.g.,][]{CollinsEtAl2002,MohamedEtAl2009,lee2009exponential,LiTao2010,ChiquetEtAl2018,tolooshams2020convolutional}.
%
%
%
%
%

As an important reference for this work, the approach of exponential family sparse coding \citep[EF-SC;][]{lee2009exponential}, for instance, defines a SC approach for noise distributions of the exponential family. The work chooses a link function that assumes the natural
parameter of an exponential family distribution to be given
by a linear superposition of generative fields. Essentially, if $\tilde{\eta}_d(\vec{s},\Theta)$ is the natural parameter of a
one-parameter distribution, then $\tilde{\eta}_d(\vec{s},\Theta)=(W\vec{s}\,)_d$ \citep[in][$W$ is denoted by $B$]{lee2009exponential}.
%
%
With a link defined in this way and a standard Laplacian prior, the posteriors of the model maintain a mono-modal shape such that a maximum a-posteriori (MAP) approximation can be applied \citep[][]{lee2009exponential}. Nevertheless, MAP-based training of SC is usually restricted to the generative fields; like for standard SC, neither prior parameters nor parameters of the noise model are inferred. Cross-validation may be used in addition but is only feasible for very few additional parameters. 

Setting the natural parameters of a given exponential family distribution using a linear superposition may (while mathematically convenient) not in general capture the true data generating process well.  
Which of the modeling choices is preferable (conventional linear superposition or a specific maximum non-linearity) does very much depend on the data, of course. As an example for the conventional linear modeling choice, consider Poisson noise as a one-parameter  exponential family distribution. Then the natural parameter is given by $\eta=\log(\lambda)$, where $\lambda$ is the mean.  
Using a linear superposition to set the natural parameter (as for EF-SC according to \cite{lee2009exponential}) would then mean $\lambda_d=\exp\big(\sum_{h}W_{dh}s_{h}\big)$ for an observable $d$. The same choice for the link was used for the Poisson PCA approach \citep[][]{salmon2014poisson} which focuses on Poisson noise and reported state-of-the-art denoising performance when it was first suggested \cite[also see][]{giryes2014sparsity}.  
The example shows that the conventional modeling choice results in a specific post-linear non-linearity to set the mean of the observable distribution. 
Other distributions of the exponential family will result in other types of non-linearities if the natural parameter is
assumed to be set by a linear sum.  
Examples of data that may well be described by a non-linear generative model include, e.g., low-intensity images where image pixels may be assumed to approximately follow a Poisson distribution whose mean is determined by a non-linearity that is close to a maximum function (also compare Poisson FA; \citealp[][]{ZhouEtAl2012}). Acoustic waveforms, on the other hand, suggest the use of conventional linear modeling (with an appropriate noise distribution).

Most approaches like Poisson PCA \citep[][]{salmon2014poisson,ChiquetEtAl2018} or EF-PCA \citep[][]{CollinsEtAl2002,SajamaOrlitsky2004,LiTao2010}, use links from latents to observables very similar to EF-SC \citep[][]{lee2009exponential}.  
%
%
Concretely, the method for generalizing
linear regression to non-Gaussian data as introduced for the GLM \citep[][]{NelderWedderburn1972,McCullaghNelder1989} was 
used in definitions of the corresponding models for unsupervised learning.
%
%
Furthermore, the parameter optimization procedures for unsupervised learning are themselves also often strongly inspired by GLM parameter optimization \citep[][and others]{CollinsEtAl2002,lee2009exponential}. EF-PCA thus uses (like Poisson PCA)
a deterministic optimization for parameter learning. 

The original EF-PCA work only considered one-parameter distributions of the exponential family, however, with sufficient statistics proportional to $y$ \citep[see][Eqn.\,2]{CollinsEtAl2002}; this restriction excludes the Pareto, Laplace and other one-parameter distributions. 
A fully Bayesian approach to exponential family PCA \citep[][]{MohamedEtAl2009} also uses the GLM-inspired linear superposition to set the natural parameters but replaces MAP-based training with a Hybrid Monte Carlo approach (which also allows for learning parameters other than the weight matrix). 
Bayesian EF-PCA is formulated for distributions of the exponential family with one and several parameters, but the concrete algorithm uses a sufficient statistic equal to the identity function \citep[see][Eqn.\,7]{MohamedEtAl2009}. EF-SC \citep[][]{lee2009exponential} can be considered as more general in its initial formulation. However, the used generative fields are all defined to set one-parameter distributions (their matrix $B$ sets the natural parameter of an exponential family distribution). Also for the Gaussian case, which is used as an introductory example, a single-parameter Gaussian (with known variance) is chosen. Closely related to such assumptions is the restriction to the natural exponential family as explicitly discussed by some approaches \citep[][]{ChiquetEtAl2018,tolooshams2020convolutional}. 
%
%

%
Notably, none of the above mentioned previous approaches, i.e.\ Poisson PCA \citep[][]{salmon2014poisson}, Poisson FA \citep[][]{ZhouEtAl2012}, EF-PCA \citep[][]{CollinsEtAl2002,SajamaOrlitsky2004,ChiquetEtAl2018}, Bayesian EF-PCA \citep[][]{MohamedEtAl2009} and EF-SC \citep[][]{lee2009exponential}, used any other than one-parameter distributions of the exponential family in their contributions as a concrete example.
Moreover, and to the best knowledge of the authors, the numerical experiments of these contributions (including EF-SC) actually only used the Bernoulli, Exponential or the Poisson distribution (usually two of the three).
Although, e.g.\ EF-SC, may in principle be able to treat other distributions of the exponential family, the meaning of setting the natural parameters of a distribution with more than one parameter using a linear superposition of the latents is arguably unclear. For instance, even for the Gaussian (as a standard two-parameter distribution), the second moment would be determined by two linear sums. Such a case (and more intricate cases resulting from other two-parameter distributions) have not been discussed or treated by EF-SC nor
by EF-PCA or similar approaches. In addition, note that approaches such as EF-SC are restricted beyond being focused on one-parameter distributions. As a deterministic
MAP-based learning is used, only weight parameters are updated (and neither additional noise nor prior parameters are learned).
%
%

In contrast to such previous approaches for exponential family observables, the here studied EF-MCA models define the links from latents to the observables using a maximum superposition model. Moreover, the link function always sets the mean of the observable distribution using the auxiliary weight matrix $M(\Theta)$ which is defined by the two dictionaries $W$ and $V$ (if the distribution has two parameters). 
This novel definition enables us to also treat more intricate cases such as Gamma and Beta distributions (and, indeed, any two parameter distributions of the exponential family).
We will later show how our definitions can be generalized to exponential family distributions with an arbitrary number of parameters $L$,
and how a generic parameter optimization can be derived using the EM algorithm. The derived algorithm will consequently
apply for the entire set of (regular) exponential family distributions. 
We note, however, that the proposed EF-MCA models exploit binary latents (here in the form of Bernoulli distributed priors), which may be restrictive in some settings. Also, the proposed link function defined by (\ref{eta}) with (\ref{meanf_eq}) and (\ref{maxh}) imposes
a specific non-linearity in the form of a (weighted) maximum function. While these assumptions will be advantageous for our focus on different (possibly intricate) noise models, other data models in the literature show other advantages. Examples are, for instance, models with emphasis on learnable non-linearities to link latents to observables; we discuss non-linear ICA \citep[][]{HyvarinenPajunen1999} and VAEs \citep[][]{KingmaWelling2014} further below.
%

One example of studies which did investigate other than Bernoulli and Poisson distributions is the work by \cite{khan2010variational} where a generalized mixture of factor analyzers model is presented that allows for addressing both continuous and discrete data (Gaussian and Multinomial distributions have been treated). Moreover, they established a new variational EM algorithm 
based on a lower bound of the Multinomial likelihood for fitting the model to the considered mixed data. Although the model can, in principle, deal with binary and categorical data along with the continuous case, it is not defined to subsume all distributions of the exponential family. Further examples are LVMs with the focus on estimating the noise distribution for observables of unknown types \citep[][]{valera2017automatic,MolinaEtAl2018,vergari2019automatic}. Those approaches can typically select between a predefined set of observable distributions to estimate the best fit given a set of (usually) heterogeneous observations. Grounded on
sum-product networks with assumptions such as piecewise polynomial leaf distributions, generalized Bayesian inference procedures are developed \citep[e.g.][]{MolinaEtAl2018}. Such approaches are consequently very different from approaches discussed above or the here investigated approach.
They are related to the considered approach here as being relatively generally defined, however, and we will come back to common features when we discuss noise type estimation in Sec.~\ref{subsec:noise_type_estimation}.
%

Still other lines of research for unsupervised learning follow deep generative modeling approaches with 
variational autoencoders \citep[VAEs;][]{KingmaWelling2014,RezendeEtAl2014} being among the most
popular such models. In order to relate latents to observables, VAEs make
use of non-linear functions which are parameterized using deep neural networks (DNNs). The non-linear functions are used to define an encoding model that, in the standard case, performs amortized inference (the encoder) and a generative model (the decoder). 
VAEs owe
much of their capabilities to the used DNNs, while prior and observation noise distributions 
for standard VAEs are very elementary (i.e., diagonal Gaussians). Nevertheless, VAEs can in principle be used to also model
non-Gaussian observables, and standard stochastic gradient ascent for optimization is not limited
to the Gaussian case. 
However, already optimization of Gaussian VAEs is considered challenging:
Major challenges are inherited from DNN optimization (starting with choices for decoder architecture, many hyper-parameters etc.)
paired with the same choices for encoder DNN(s) and necessarily stochastic optimization.
Concrete examples of non-Gaussian VAEs use one specific observation noise, usually in the form of elementary distributions
such as Bernoulli or Poisson \citep[e.g.,][]{GronbechEtAl2020,JangEtAl2017,BabyBourlard2021}. 
%
%
%
The latents are for standard forms of VAEs continuous and usually Gaussian distributed.  
But variants with discrete latents have also been investigated \citep[e.g.,][]{JangEtAl2017,MaddisonEtAl2017,Rolfe2017,YinEtAl2019,BerlinerEtAl2021,DrefsEtAl2022a}. 

In a related field,  recent works on non-linear ICA approaches focus on identifiability which can be guaranteed under certain latent dependency conditions including temporal dependencies \citep[see, e.g.,][]{SprekelerEtAl2014,HyvarinenMorioka2016,HyvarinenMorioka2017,HyvarinenEtAl2019}. 
Commonly for these approaches, a contrastive learning (CL) algorithm is applied in order to estimate the non-linear ICA
model: \citet[][]{HyvarinenMorioka2016} use a time-contrastive learning (TCL), \citet[][]{HyvarinenMorioka2017} use a
permutation-contrastive learning (PCL), and \citet[][]{HyvarinenEtAl2019} use a generalized form of CL.
Identifiability in non-linear ICA is, in principle, possible also for non-Gaussian observation noise. Recent work explicitly specifies relatively general conditions for noise and identifiability in a framework for structured non-linear ICA \citep[][]{HalvaEtAl2021}. At the same time, the concrete realization of an algorithm within this framework uses Gaussian noise (which is explicitly stated as limitation). 
In another related work, \citet[][]{KhemakhemEtAl2020} investigate an identifiable VAE (iVAE) 
that bridges the gap between VAEs and non-linear ICA. 
The approach, however, assumes the noise distribution to be known and additive which 
likewise restricts the applicability of the model to Gaussian distributed observables. 
A later extension of iVAE (termed pi-VAE), was proposed by \citet[][]{ZhouWei2020} which assumes Poisson noise for the observables. 
In the pi-VAE approach, the mean of the Poisson distribution is set by a link function that is defined in terms of a general invertible function.

%
%

\begin{table}[t!]
\caption{Summary of related previous work as discussed in Secs.~\ref{Sec_intro} and \ref{Relation_to_Previous_Work}.}
\centering
\newcommand\spac{-0.1cm}
\renewcommand{\aboverulesep}{0ex}
\renewcommand{\belowrulesep}{0ex}
\newcolumntype{K}[1]{>{\centering\arraybackslash}p{#1}}
\resizebox{\linewidth}{!}{
\begin{tabular}[c c c]{@{\vline} c c c @{\vline}}
\toprule
\multirow{2}{*}{Acronym}&\multirow{2}{*}{Model Details}&\multirow{2}{*}{Reference}\\
&&\\
\bottomrule
\multicolumn{3}{c}{\vspace{\spac}}\\
\midrule
\multicolumn{3}{@{\vline} c @{\vline}}{Examples of LVMs with a linear component structure}\\
\midrule
EF-PCA & exp. family obser. / dist. with sufficient stat. proportional to $y$  & \citet[][]{CollinsEtAl2002} \\
\ Bayes. EF-PCA &  exp. family obser. / sampling-based optim.  / dist. with sufficient stat. proportional to $y$ & \citet[][]{MohamedEtAl2009}\\
EF-SC &  exp. family obser. / MAP approx. / only one parameter dist.  used & \citet[][]{lee2009exponential}\\
Poiss. FA & Poiss. obser. / sampling-based optim.   & \citet[][]{ZhouEtAl2012}\\
Linear ICA &  contin. obser. / non-Gaussian sources & \citet[][]{HyvarinenOja2000}\\
Binary ICA & binary obser. / non-stationary sources / linear mixing of sources followed by binarization & \citet[][]{HyttinenEtAl2022}\\
NLSPCA & Poiss. obser. / non-local (sparse) PCA / designed for Poiss. denoising & \citet[][]{salmon2014poisson}\\
SPDA & Poiss. SC / designed for Poiss. denoising & \citet[][]{giryes2014sparsity}\\
\midrule
\multicolumn{3}{c}{\vspace{\spac}}\\
\midrule
\multicolumn{3}{@{\vline} c @{\vline}}{Examples of LVMs with a specific non-linear component structure}\\
\midrule
non-linear ICA & non-linear blind source separa. / slow feature analy. / temporal correl. & \citet[][]{SprekelerEtAl2014}\\
non-linear ICA & exp. family sources / temporal non-stationary structure / TCL & \ \citet[][]{HyvarinenMorioka2016}\,\,\\
non-linear ICA & non-Gaussian and temporally dependent stationary sources / PCL & \ \citet[][]{HyvarinenMorioka2017}\,\,\\
non-linear ICA & conditionally independent sources (given an auxiliary var.) / gen. CL & \citet[][]{HyvarinenEtAl2019}\\ 
OCA & Gauss. obser. / occlusion non-linearity / var. to control the depth of objects & \citet[][]{HennigesEtAl2014} \\
MCA & Poiss. obser. / truncated approximation / binary latents & \citet[][]{LuckeSahani2008} \\
ss-MCA & Gauss. obser. / trunc.\,approx. / non-linear and spike-and-slab SC & \citet[][]{SheltonEtAl2012} \\
G-MCA & Gauss. obser. / truncated variational approximations / binary latents & \citet[][]{SheikhEtAl2019} \\
{\bf EF-MCA} & exp. family obser. / generic M-steps and EVO / binary latents / $L$ dictionaries & our contribution\\
\midrule
\multicolumn{3}{c}{\vspace{\spac}}\\
\midrule
\multicolumn{3}{@{\vline} c @{\vline}}{Examples of related DNNs for Gaussian / non-Gaussian observables}\\
\midrule
DCEA & natural exp. family obser. / trained both supervised and unsupervised & \citet[][]{tolooshams2020convolutional}\\
DenoiseNet & Poiss. obser. / designed for Poiss. denoising & \citet[][]{remez2017deep} \\
TVAE & Gauss. obser. / variational training based on EVO / binary latents & \citet[][]{DrefsEtAl2022a}\\
iVAE & noisy (Gauss.) obser. / identifiable VAE (unifies  non-linear ICA and VAE) & \citet[][]{KhemakhemEtAl2020} \\
pi-VAE & Poiss. obser. / extends the iVAE model to Poisson noise & \citet[][]{ZhouWei2020} \\
\midrule
\end{tabular}
}
\label{table:ModelComparisons}
\end{table}

Other than non-linear ICA and VAEs, there also have been investigations of more conventional deep models, e.g., on convolutional neural networks (CNNs), that focused on non-Gaussian observables.
For example, contributions such as \citet[][]{remez2018class,remez2017deep} proposed an approach for denoising images corrupted with Gaussian and Poisson noises
using supervised training.
Also, the study by \citet[][]{tolooshams2020convolutional} considered observables of non-Gaussian distributions using convolutional autoencoders trained in both a supervised or an unsupervised manner. The method explicitly investigated the natural exponential family distributions which is a specific subset of the exponential family. 

For supervised learning (i.e., regression), GLMs do explicitly consider different noise distributions and specifically distributions of the exponential family with more than one parameter. GLMs can, for instance, use a specific dependency of the observable variances on the mean to treat two-parameter distributions. 
However, regression approaches for specific distributions (such as the Beta distribution) still require further considerations related to GLM techniques in order to realize concrete regression algorithms \citep[see, e.g.,][for a concrete discussion]{FerrariCribariNeto2004}. 

In summary, the here considered LVM (while based on a standard bipartite graphical model) is exceptionally general in allowing for any
(regular) exponential family distribution to be used to model observation noise.  
For a general treatment of such observables, the crucial challenge is, of course, the derivation of a generic parameter optimization procedure, which we will address in the following. 
Tab.~\ref{table:ModelComparisons} presents a short summary of the approaches discussed here. 

\section{Maximum Likelihood}
Having defined the family of EF-MCA models in Sec.\,\ref{sec:nonlinear_sc_models}, we now seek parameters $\Theta$ that optimize the models for a given set of $N$ independent and identically
distributed (i.i.d.) data points $\vec{y}^{\,(1)},\ldots,\vec{y}^{\,(N)}$.
%
%
%
We follow a standard maximum likelihood approach and seek parameters $\Theta$ that optimize the log-likelihood $\mathcal{L}(\Theta)=\sum_n \log \big(p(\vec{y}^{\,(n)}\,| \Theta)\big)$. 
We use the EM algorithm \citep[][]{DempsterEtAl1977,NealHinton1998} which instead of maximizing the log-likelihood directly, optimizes a lower bound known as the \textit{free energy} \citep[][]{NealHinton1998} or the \textit{evidence lower bound} \citep[ELBO;][]{BleiMcAuliffe2008}. 
For the family of generative models defined in \eqref{modelEq1}-\eqref{modelEq3}, the ELBO is given by:
%
%
\begin{align}\label{EqnFreeEnergy}
\mathcal{F}(q, \Theta) :=& \sum_{n} \sum_{\vec{s}} q^{(n)}(\vec{s}\,) \big\{\sum_d\log\big(p(y_d^{(n)};  \vec{\tilde{\eta}}_d(\vec{s}, \Theta)) \big) + \sum_h\log(p(s_h\,| \Theta)) \big\} + \mathcal{H}(q),
\end{align}
where $q^{(n)}(\vec{s}\,)$ for each $n$ is a variational distribution, $q = (q^{(1)}, \ldots, q^{(N)})$ and $\mathcal{H}(q)$ is the Shannon entropy term.  
The ELBO is the objective we seek to optimize, and its relation to the log-likelihood is, for completeness, given in Appendix~\ref{appendix_ELBO}. 
In our learning algorithm, the ELBO is optimized iteratively w.r.t.\ distributions $q$ (the E-step) and w.r.t.\ the model parameters $\Theta$ (the M-step) as
commonly done.
It has been shown, e.g. by \cite{NealHinton1998}, that the ELBO is maximized w.r.t.\ distributions $q$ if, for each $\vec{s}$ and $n$, we equate $q^{(n)}(\sVec \,)$ to the posteriors, i.e. $q^{(n)}(\vec{s}\,)=p(\vec{s}\,|\,\vec{y}^{\,(n)},\Theta)$.
Indeed, the distributions $q^{(n)}(\vec{s}\,)$ can be the exact posteriors for tractable models (we refer to this case as \textit{full EM} or \textit{exact EM}) or they can represent variational approximations (in this case we refer to \textit{variational EM} and consequently \textit{variational E-steps}). The terms (variational) free energy $\mathcal{F}(q, \Theta)$ or ELBO are more common for intractable models, i.e., if the $q^{(n)}$ are approximations of intractable exact posteriors. 
%

For tractable models, the E-steps reduce to the computation of exact posteriors. For our EF-MCA model, a generic form to compute the E-steps is, therefore, already given by Bayes' rule as long as the model is sufficiently small scale (small $H$). For larger scale EF-MCA models, Sec.~\ref{sec:estep} will address how generic
variational E-steps can be used. 
%
%
%
%
%
%
%
For all EF-MCA models, the central challenge that we have to address here is, however, the derivation of parameter update equations in the M-step, which is pertinent to the noise distribution.  We therefore pursue such a task in the next section and importantly show that the non-linear link function \eqref{maxh} enables us to derive a set of concise and generally applicable updates that 
maximizes the ELBO \eqref{EqnFreeEnergy}.
\subsection{Parameter Update Equations -- The M-step}
Following the standard optimization procedure, we will set the derivatives of $\mathcal{F}(q, \Theta)$
w.r.t.\ all model parameters to zero, and derive parameter update rules from the resulting equation systems.
For this, first note that the derivatives of $\mathcal{F}(q, \Theta)$ will contain derivatives of $\bar{W}_d(\vec{s}, \Theta)$ and $\bar{V}_d(\vec{s}, \Theta)$.
Then, the following applies for these derivatives:
\begin{align}
\frac{\partial}{\partial W_{dh}} \bar{W}_d(\vec{s}, \Theta) &= \mathcal{A}_{dh}(\vec{s}, \Theta), &\frac{\partial}{\partial W_{dh}} \bar{V}_d(\vec{s}, \Theta) = 0, \label{derivatives_WbarVbar1}\\
\frac{\partial}{\partial V_{dh}} \bar{V}_d(\vec{s}, \Theta) &= \mathcal{A}_{dh}(\vec{s}, \Theta), &\frac{\partial}{\partial V_{dh}} \bar{W}_d(\vec{s}, \Theta) = 0, \label{derivatives_WbarVbar2}
\end{align}
where \vspace{-3ex}
\begin{align}\label{Adh}
\phantom{iiii}\mathcal{A}_{dh}(\vec{s}, \Theta) := \begin{cases} 1 \quad &h = h(d,\vec{s}, \Theta)\\
0\quad &\mathrm{otherwise} \end{cases}
\end{align}
\noindent{}which follows simply from considering the cases $h=h(d,\vec{s}, \Theta)$ and $h\not=h(d,\vec{s}, \Theta)$ separately. We will later discuss the properties of the function $\mathcal{A}_{dh}(\vec{s}, \Theta)$ in Appendix~\ref{app:theorem} (in Lemma~\ref{lemma1}). Now, given the equations above, we present the main result of this study: We derive concise equations for $W$ and $V$ which guarantee that all derivatives of $\mathcal{F}(q, \Theta)$ w.r.t.\ $W_{dh}$ and $V_{dh}$ vanish. These equations can then be used for parameter updates in each M-step. For the sake of readability,  however,  we only provide a proof sketch of the theorem now and present the full proof in Appendix~\ref{appendixA0}. Also, note that we here discuss the case for $L = 2$, and the general case will be further presented in Sec.~\ref{generalEF_MCA} and Appendix~\ref{appendixA1}.
\begin{thm}\label{theorem1}
Consider an EF-MCA data model \eqref{modelEq1}-\eqref{modelEq3} with $p(y;\vec{\eta}\,)$ being a regular exponential family distribution with \mbox{$L=2$}.
Furthermore, let the parameters $\Theta$ contain the matrices $W$ and $V$ with $D \times H$ entries and $\vec{\tilde{\eta}}_d(\vec{s},\Theta)$
to be defined as in Eqns. \eqref{eta} and \eqref{meanf_eq}-\eqref{maxh}. Then, the derivatives of the ELBO \eqref{EqnFreeEnergy} w.r.t.\ all matrix elements $W_{dh}$ and $V_{dh}$ are zero if for all $d$ and $h$:
\begin{align}\label{Update_W}
W_{dh} &= \frac{\sum_{n = 1}^N \mathbb{E}_{q^{(n)}}  [ \mathcal{A}_{dh}(\vec{s},\Theta) ] \, T_1(y_d^{(n)})}{\sum_{n = 1}^N \mathbb{E}_{q^{(n)}}  [ \mathcal{A}_{dh}(\vec{s},\Theta) ] },
\end{align}
and 
\begin{align}
V_{dh} &= \frac{\sum_{n = 1}^N \mathbb{E}_{q^{(n)}}  [ \mathcal{A}_{dh}(\vec{s},\Theta) ] \, T_2(y_d^{(n)})}{\sum_{n = 1}^N \mathbb{E}_{q^{(n)}}  [ \mathcal{A}_{dh}(\vec{s},\Theta) ] }, \label{Update_V}
\end{align}
where $\mathcal{A}_{dh}(\vec{s},\Theta)$ is given by \eqref{Adh}. 
\end{thm}
\sproof \,
To proof the theorem, we decisively make use of the relations in \eqref{A_eta_expected_value} and the maximum superposition defined in \eqref{eta} and \eqref{meanf_eq}-\eqref{maxh}, and set the derivatives of the ELBO w.r.t. the dictionary elements $W_{dh}$ and $V_{dh} $ to zero. 
In short, if we abbreviate $ \bar{W}_d(\vec{s}, \Theta)$ and $\bar{V}_d(\vec{s}, \Theta)$ by $\bar{W}_d$ and $\bar{V}_d $, then for a single dictionary element $W_{dh}$ we have (using the chain rule):

\begin{align}
\frac{\partial}{\partial W_{dh}} & \log\Big(p\big(y_d; \vec{\tilde{\eta}}_d(\vec{s}, \Theta)\big)\Big) = \frac{\partial}{\partial W_{dh}} \log(p(y_d; \vec{\Phi}(\bar{W}_d, \bar{V}_d)))\\
&= \sum_{l=1}^2 \mathcal{A}_{dh}(\vec{s}, \Theta) \Big( \frac{\partial}{\partial w} \Phi_l(w, \bar{V}_d) \big|_{w = \bar{W}_d} \Big) \Big( T_l(y_d) - \frac{\partial A(\vec{\eta}\,)}{\partial \eta_l} \Big) \Big |_{\vec{\eta} = \vec{\Phi}(\bar{W}_d, \bar{V}_d)}\,.\nonumber
\end{align}
For this expression we now substitute $\frac{\partial A(\vec{\eta}\,)}{\partial \eta_1} =\, \mathbb{E}_{p(y;\, \vec{\eta}\,)}  [ T_1(y) ]$ and $\frac{\partial A(\vec{\eta}\,)}{\partial \eta_2} =\, \mathbb{E}_{p(y;\, \vec{\eta}\,)}  [ T_2(y) ]$ from Eqn.\,\eqref{A_eta_expected_value}. Furthermore, we use that for any function $g$ of $\bar{W}_d$ and $\bar{V}_d$ applies:
\begin{align}
\mathcal{A}_{dh}(\vec{s}, \Theta)\, g ( \bar{W}_d, \bar{V}_d ) = \mathcal{A}_{dh}(\vec{s}, \Theta)\, g(W_{dh}, V_{dh}), \,\nonumber
\end{align}
see Lemma~\ref{lemma1} in Appendix~\ref{app:theorem}. We then obtain:
\begin{align}
\frac{\partial}{\partial W_{dh}} \log\Big(p\big(y_d; \vec{\tilde{\eta}}_d(\vec{s}, \Theta)\big)\Big) =& \sum_{l=1}^2  \mathcal{A}_{dh}(\vec{s}, \Theta) \Big( \frac{\partial}{\partial w} \Phi_l(w, V_{dh}) \big|_{w = W_{dh}} \Big) \cr
&\times \Big( T_l(y_d) - \mathbb{E}_{p(y;\, \vec{\Phi}(W_{dh}, V_{dh}))}  [ T_l(y) ] \Big).
\end{align}
%
Now, using Eqn.\,\eqref{EqnFreeEnergy} and the definition of the mean value parameters, $\vec{w} :=\, \mathbb{E}_{p(y;\, \vec{\Phi}(\vec{w}))}  [\vec{T}(y) ]$, we get:
\begin{align}
\frac{\partial}{\partial W_{dh}}\mathcal{F}(q, \Theta) =& \sum_{n} \sum_{\vec{s}} q^{(n)}(\vec{s}\,) \Big( \sum_{d'} \frac{\partial}{\partial W_{dh}} \log \big( p(y_{d'}^{(n)}; \vec{\tilde{\eta}}_{d'}(\vec{s}, \Theta)) \big) \Big)\\
=& \sum_{n} \sum_{\vec{s}} q^{(n)}(\vec{s}\,) \sum_{l=1}^2 \mathcal{A}_{dh}(\vec{s}, \Theta) \Big( \frac{\partial}{\partial w} \Phi_l(w, V_{dh}) \big|_{w = W_{dh}} \Big)\cr
&\times \Big( T_l(y_d^{(n)})\, - \mathbb{E}_{p(y;\, \vec{\Phi}(W_{dh}, V_{dh}))}  [ T_l(y) ] \Big)\cr
=& \sum_{l=1}^2 \Big( \frac{\partial}{\partial w} \Phi_l(w, V_{dh}) \big|_{w = W_{dh}} \Big) \sum_{n} \sum_{\vec{s}} q^{(n)}(\vec{s} \,)\, \mathcal{A}_{dh}(\vec{s}, \Theta) \cr
&\times \Big( T_l(y_d^{(n)})\, - \mathbb{E}_{p(y;\, \vec{\Phi}(W_{dh}, V_{dh}))}  [ T_l(y) ] \Big). \nonumber
\end{align}
The expression is equal to zero, i.e. $\frac{\partial \mathcal{F}(q, \Theta)}{\partial W_{dh}} = 0$, if the following equations are satisfied:
%
%
\begin{align}
\sum_n \mathbb{E}_{q^{(n)}}  [ \mathcal{A}_{dh}(\vec{s}, \Theta) ] \big(T_1(y_d^{(n)}) - W_{dh}\big) &= 0,\ \ \mbox{and}\\
\sum_n \mathbb{E}_{q^{(n)}}  [ \mathcal{A}_{dh}(\vec{s}, \Theta) ] \big(T_2(y_d^{(n)}) - V_{dh}\big) &= 0.
\end{align}
%
Then, the above equations yield \eqref{Update_W} and \eqref{Update_V}, which completes the proof. \QED


\ \\

\noindent{}Fulfilling Eqns.~\eqref{Update_W} and \eqref{Update_V} guarantees vanishing derivatives and provides a generally applicable approach for updating $W$ and $V$ in each M-step. We do remark, however, that we have not strictly proven that \eqref{Update_W} and \eqref{Update_V} correspond to a maximum (and not a minimum or a saddle point). 
In this respect, one can further investigate the second derivatives of the ELBO \eqref{EqnFreeEnergy} at these stationary points to observe if Eqns.~\eqref{Update_W} and \eqref{Update_V} do in fact correspond to a maximum.  We will later discuss this point in  Appendix~\ref{Sec_second_derivative}. 
%
%
Furthermore, we here emphasize that Eqns. \eqref{Update_W} and \eqref{Update_V} are valid for any regular two-parameter distribution of the exponential family which includes Gaussian, Gamma, Beta and many more, i.e., a large variety of noise models is covered. Importantly, the above theorem reveals that the same functional form is obtained for the parameters of all these distributions under the generative model \eqref{modelEq1}-\eqref{modelEq3} (note that we here considered the $L=2$ case but the general case will be  discussed further below). A straightforward outcome of the foregoing theorem is when the distribution does contain a sufficient statistic proportional to $y$, i.e.\ $T_1(y)=y$. This specific form yields a further simplification:
%
%
\begin{corollary}
Prerequisites as for Theorem~\ref{theorem1}. If the distribution $p(y;\vec{\eta}\,)$ has a sufficient statistic $T_1(y)=y$, then the condition for $W_{dh}$ is given by:
\begin{align}\label{oldupdate}
W_{dh} = \frac{\sum_{n = 1}^N \mathbb{E}_{q^{(n)}}  [ \mathcal{A}_{dh}(\vec{s},\Theta) ] \, y_d^{(n)}}{\sum_{n = 1}^N \mathbb{E}_{q^{(n)}}  [ \mathcal{A}_{dh}(\vec{s}, \Theta) ] }.
\end{align}
\end{corollary}
The corollary is instructive about previous contributions: It explains why the update equation for the original MCA data model \citep[][]{LuckeSahani2008} (which used Poisson noise) and the update equation for the later MCA models \citep[][]{LuckeEggert2010,BornscheinEtAl2013,SheikhEtAl2019} (which used Gaussian noise) are identical and given by (\ref{oldupdate}).

Finally, observe that Eqns. \eqref{Update_W} and \eqref{Update_V} do not represent
closed-form solutions for $W$ and $V$ since the right-hand-sides also depend on $W$ and $V$ through the function $\mathcal{A}_{dh}(\vec{s},\Theta)$. 
We thus follow previous work
\citep[e.g.,][]{LuckeSahani2008,LuckeEggert2010}, and use Eqns. \eqref{Update_W} and \eqref{Update_V} in the fixed-point sense and 
update (for each $d$ and $h$):
\begin{align}\label{update_W_real}
W^{\mathrm{new}}_{dh} = \frac{\sum_{n = 1}^N \mathbb{E}_{q^{(n)}}  [ \mathcal{A}_{dh}(\vec{s}, \Theta^{\mathrm{old}}) ] \, T_1(y_d^{(n)})}{\sum_{n = 1}^N \mathbb{E}_{q^{(n)}}  [ \mathcal{A}_{dh}(\vec{s}, \Theta^{\mathrm{old}}) ] },
\end{align} 
and
\begin{align}\label{update_V_real}
V^{\mathrm{new}}_{dh} = \frac{\sum_{n = 1}^N \mathbb{E}_{q^{(n)}}  [ \mathcal{A}_{dh}(\vec{s}, \Theta^{\mathrm{old}}) ] \, T_2(y_d^{(n)})}{\sum_{n = 1}^N \mathbb{E}_{q^{(n)}}  [ \mathcal{A}_{dh}(\vec{s}, \Theta^{\mathrm{old}}) ] },
\end{align}
where also $q^{(n)}=q^{(n)}(\vec{s} \,;\Theta^{\mathrm{old}})$ depends on the old parameters.
If repeated updates result in the values of $W$ and $V$ to converge, then the converged values fulfill \eqref{Update_W} and \eqref{Update_V}, respectively.

To complete the M-step parameter updates, we can also derive updates for the prior parameters $\vec{\pi}$. Those derivations do not involve the specific form of the observables' distribution.
We can therefore use previous derivations \cite[e.g.,][]{HennigesEtAl2010,PuertasEtAl2010} and update $\vec{\pi}$ as follows:
\begin{align}\label{Update_pi}
\pi_h^{\mathrm{new}} = \frac{1}{N} \sum_{n = 1}^N \mathbb{E}_{q^{(n)}}  [ s_h ], \quad h = 1, \ldots, H.
\end{align}
%
%
%

\begin{algorithm}[h!]
\SetAlgoLined
 initialize model parameters $\Theta = (\vec{\pi}, W, V)$\;
 \Repeat{parameters $\Theta$ have sufficiently converged}{
 consider the inverse mapping $\vec{\Phi}$ defined in \eqref{phi_function}\;
 \eIf{the distribution $p(y;\vec{\eta}\,)$ has a sufficient statistic $T_1(y)=y$}{
   $M(\Theta) = W$\;
   }{
   consider the function $F$ in \eqref{meanf_eq} based on the relation of the natural parameters $\vec{\eta}$ and the mean of $p(y;\vec{\eta}\,)$\;
   $M(\Theta) = F(W, V)$;
  }
 \For{each vector $\vec{s}$ of the latent space}{
 \For{$d = 1 : D$}{
  $h(d, \vec{s}, \Theta) = \mbox{argmax}_h\{M_{dh}(\Theta) s_h\}$\;
  $\bar{W}_d = W_{dh(d, \vec{s}, \Theta)}$ and $\bar{V}_d = V_{dh(d, \vec{s}, \Theta)}$\;
  $\vec{\tilde{\eta}}_d = \vec{\Phi}(\bar{W}_d, \bar{V}_d)$;
  }
  \For{$n = 1 : N$}{
  $q^{(n)}(\vec{s}\,) = p(\vec{s}\,|\,\vec{y}^{\,(n)},\Theta)$\;
\For{$h = 1 : H$ and $d = 1 : D$}{
  compute $q^{(n)}(\vec{s}\,)s_h$\; 
  compute $q^{(n)}(\vec{s}\,) \mathcal{A}_{dh}(\vec{s}, \Theta)$ where $\mathcal{A}_{dh}(\vec{s}, \Theta)$ is defined in \eqref{Adh}\; 
%
  }
  }
  }
  update parameters $\Theta$ using \eqref{update_W_real}-\eqref{Update_pi}\;
 }
 \caption{EM for training EF-MCA data models (for $L=2$ parameter distributions)}
 \label{algorithm1}
\end{algorithm}

\subsection{Optimization of Variational Distributions -- The E-step}
\label{sec:estep}
The main focus of the previous section was the derivation of parameter updates in the M-step (Eqns.~\eqref{update_W_real}-\eqref{Update_pi}). 
%
These M-step equations contain expectation values w.r.t.\ distributions $q^{(n)}(\vec{s}\,)$ that can be the exact posteriors $p(\vec{s}\,|\,\vec{y}^{\,(n)},\Theta)$
for sufficiently small scale models (small $H$). 
%
%
Larger $H$ require approximate posteriors $q^{(n)}(\vec{s}\,)$ as the computational cost of evaluating expectations w.r.t.\ exact posteriors scales exponentially, i.e., with $2^H$. 
To enable efficient approximations in variational E-steps,
different
approaches can be applied. For our purposes, we here use a variational
approximation based on truncated posteriors as a family of variational distributions \citep[][]{LuckeEggert2010,LuckeSahani2008}.
Concretely, we use variational distributions $q^{(n)}$ defined as:
\begin{equation}\label{eq:TVEMq}
q^{(n)}(\sVec \mid  \KKn, \Theta) \defeq \frac{p(\vec{s}\,|\,\vec{y}^{\, (n)},\Theta)}{\sum\limits_{\sVec \,' \in \KKn} p \left(\sVec \,' | \, \yVecN, \Theta \right)}\delta ( \sVec \in \KKn ), \quad n = 1, \ldots, N,
\end{equation}
where the set $\mathcal{K}^{(n)}$ contains, for each data point $\vec{y}^{\,(n)}$, a fixed number of hidden states $\sVec$, and where the indicator function $\delta ( \sVec \in \mathcal{K}^{(n)} )$ is equal to 1 if $\sVec$ is in the set $\mathcal{K}^{(n)}$ (and 0 otherwise). 
Distributions (\ref{eq:TVEMq}) approximate full posteriors by truncating sums over the whole latent space to sums over those subsets $\KKn$ which ideally accumulate most of the posterior mass. 

Variational approaches based on truncated posteriors
have been shown to result in very efficient optimization for mixture models \citep[][]{HirschbergerEtAl2022,ExarchakisEtAl2022} as well as LVMs \citep[e.g.,][]{SheikhLucke2016,SheltonEtAl2017,SheikhEtAl2019}.
%
When choosing $q^{(n)}$ as defined by \eqref{eq:TVEMq}
to approximate full posteriors, the expectation values $\mathbb{E}_{q^{(n)}}[\mathcal{A}_{dh}(\vec{s}, \Theta)]$ in (\ref{update_W_real}) and (\ref{update_V_real}) are, for the family of LVMs in \eqref{modelEq1}-\eqref{modelEq3}, obtained as follows:
\begin{align}\label{eq:TVEMdetail}
\mathbb{E}_{q^{(n)}}[\mathcal{A}_{dh}(\vec{s}, \Theta)] &= \frac{\sum\limits_{\sVec \in \KKn} p \left(\sVec, \yVecN \mid \Theta \right) \mathcal{A}_{dh}(\vec{s}, \Theta)}{\sum\limits_{\sVec \,' \in \KKn} p \left(\sVec \,', \yVecN \mid \Theta \right)} \\
&= \frac{\sum\limits_{\sVec \in \KKn} 
\exp \left( \sum_d \big( (\vec{\tilde{\eta}}_d(\vec{s},\Theta))^T \vec{T}(y_d^{(n)}) - A(\vec{\tilde{\eta}}_d(\vec{s},\Theta)) 
\big) \right) \mathrm{Bern}(\vec{s} \,; \vec{\pi}) \,  
\mathcal{A}_{dh}(\vec{s}, \Theta)}{\sum\limits_{\sVec \,' \in \KKn} 
\exp \left( \sum_d \big( (\vec{\tilde{\eta}}_d(\vec{s}\,',\Theta))^T \vec{T}(y_d^{(n)}) - A(\vec{\tilde{\eta}}_d(\vec{s}\,',\Theta)) 
\big) \right) \mathrm{Bern}(\vec{s}\,' ; \vec{\pi})
} \ ,  \nonumber
\end{align}
where $\mathrm{Bern}(\vec{s} \,; \vec{\pi})$ abbreviates the Bernoulli prior in \eqref{modelEq1}. Likewise, the expectation values $\mathbb{E}_{q^{(n)}}[s_h]$ in (\ref{Update_pi}) can be obtained (for each $h$) accordingly.

\subsubsection{Evolutionary Variational Optimization (EVO)}
\label{sec_EVO}
Crucial for variational optimization of truncated posteriors \eqref{eq:TVEMq} are strategies to efficiently and effectively populate the $\KKn$ sets with latent states capturing high posterior mass. 
To this end, here we apply a recently suggested strategy based on evolutionary algorithms (Evolutionary Variational Optimization, EVO; \citealt[][]{DrefsEtAl2022}).
%
The approach treats the sets $\KKn$ in Eqns.~\eqref{eq:TVEMq} and \eqref{eq:TVEMdetail} as variational parameters, which consequently yields the variational lower bound 
\eqref{EqnFreeEnergy} to be reformulated as follows \citep[see][for a concrete discussion]{Lucke2019}:
\begin{equation*}
  \mathcal{F}(\KK^{(1,\dots,N)}, \Theta) = \sum_{n} \log\Bigl(\sum_{\sVec\in\KKn} p(\yVecN,\sVec\mid\Theta)\Bigr).
\end{equation*}
Based on this form, the variational lower bound can be increased by replacing $\sVec\in\KKn$ with $\sVecNew \notin\KKn$ for each $n$ such that $p(\yVecN,\sVecNew\mid\Theta)>p(\yVecN,\sVec\mid\Theta)$. EVO, specifically, treats the bit vectors $\sVec$ as genomes and uses them to evolve $\sVecNew$ through application of elementary genetic operations including selection, mutation and crossover.
%
A requirement of EVO is that the model's joint $p(\yVecN,\sVec\mid\Theta)$ is analytically tractable in order to define a fitness function for the selection step. Apart from that, EVO is generically applicable, i.e., the approach does not involve model-specific derivations (in contrast to, e.g., mean-field approaches; \citealt[][]{JordanEtAl1999}).
For details, see \citep{DrefsEtAl2022,DrefsEtAl2022a}.
%
%
%
%

For the family of LVMs we consider here,
the required joints are directly defined by the chosen exponential family distributions.
Consequently, if the sufficient statistics $\vec{T}(y)$, the log-partition $A(\vec{\eta}\,)$ and the link $\vec{\tilde{\eta}}_d(\vec{s},\Theta)$
are efficiently computable, then so are the expectation values (\ref{eq:TVEMdetail}) and the entire variationally accelerated learning algorithm. 

\section{Parametrization of the Proposed Generative Models -- General Case}
\label{generalEF_MCA}
Let us now finally generalize the definition of the proposed family of generative models presented in \eqref{modelEq1}-\eqref{modelEq3} to any arbitrary $L$-parameter distribution. 
%
%
We remain using the independent Bernoulli distributions of \eqref{modelEq1}. For the noise distribution given by \eqref{modelEq2} and \eqref{modelEq3}, however, 
instead of two matrices $W$ and $V$, we can in general consider $L$ matrices (corresponding to the $L$ parameters of an exponential family distribution). Let us denote these matrices by $W^{(1)}$ to $W^{(L)}$, i.e. $\Theta = (\vec{\pi}, W^{(1)}, \ldots, W^{(L)})$. 
Therefore, we can generalize our definition of $\vec{\tilde{\eta}}_d(\vec{s},\Theta)$, which is now given by:
\begin{align}\label{etagen}
\vec{\tilde{\eta}}_d(\vec{s},\Theta) := \vec{\Phi}\big(\bar{W}^{(1)}_d(\vec{s}, \Theta),\ldots, \bar{W}^{(L)}_d(\vec{s}, \Theta) \big).
\end{align}
Moreover, given $L$ matrices $W^{(1)}$ to $W^{(L)}$ with $D \times H$ entries, we have to define the mappings $\bar{W}^{(1)}_d(\vec{s}, \Theta)$ to $ \bar{W}^{(L)}_d(\vec{s}, \Theta)$ for each observable $d$.
In analogy to the $L=2$ case, we do so by defining a matrix $M(\Theta)$ as follows:
\begin{align}\label{meanfgen}
\forall d,h:\phantom{iii}M_{dh}(\Theta) := F(W^{(1)}_{dh},\ldots,W^{(L)}_{dh}) \ \ \mbox{with}\ \ \ F(\vec{w}) =\, \mathbb{E}_{p(y;\,\vec{\Phi}(\vec{w}))}  [\,y\,] ,
\end{align}
where $\vec{w}$ is assumed a row vector. Using the matrix $M(\Theta)$, we now define our mappings $\bar{W}^{(l)}_d(\vec{s},\Theta)$ for $l = 1, \ldots, L$:
\begin{align}\label{maxhgen}
\forall d,h:\bar{W}^{(l)}_d(\vec{s},\Theta) := W^{(l)}_{dh(d,\vec{s}, \Theta)} \ \ \mbox{with}\ \ \ h(d,\vec{s}, \Theta) := \mathrm{argmax}_{h} \{M_{dh}(\Theta)\, s_h\}.
\end{align}
Also in the general case, definitions (\ref{meanfgen}) and (\ref{maxhgen}) guarantee that always the \textit{mean} of observable~$d$
is given by $\mu_d=\max_h\{M_{dh}(\Theta)\, s_h\}$. The above equations then finalize the definition of EF-MCA data models where the noise distribution can now be any arbitrary regular distribution of the exponential family.  
Then, for Eqns.~\eqref{modelEq1}-\eqref{modelEq3} with (\ref{etagen})-(\ref{maxhgen}), 
the following theorem applies which is a straightforward generalization of Theorem~\ref{theorem1}:

\begin{theorem}\label{theorem2}
Consider an EF-MCA data model (\ref{modelEq1})-(\ref{modelEq3}) with $p(y;\vec{\eta}\,)$ being a regular exponential family distribution with sufficient statistics $\vec{T}(y)$ of length
\mbox{$L\in\mathbb{N}$}. Moreover, let the parameters $\Theta$ contain $L$ matrices $W^{(1)},\ldots,W^{(L)}$ with $D \times H$ entries and let $\vec{\tilde{\eta}}_d(\vec{s},\Theta)$
be defined as in \eqref{etagen}-\eqref{maxhgen}. Then, the derivatives of the ELBO \eqref{EqnFreeEnergy} w.r.t.\ all $W^{(l)}_{dh}$ are zero if for all $d$, $h$, and $l$: 
\begin{align}\label{Update_W_general}
W^{(l)}_{dh} = \frac{\sum_{n = 1}^N \mathbb{E}_{q^{(n)}}  [ \mathcal{A}_{dh}(\vec{s},\Theta) ] \, T_l(y_d^{(n)})}{\sum_{n = 1}^N \mathbb{E}_{q^{(n)}}  [ \mathcal{A}_{dh}(\vec{s},\Theta) ] },
\end{align} 
where $\mathcal{A}_{dh}(\vec{s},\Theta)$ is given by \eqref{Adh}. 
\end{theorem}

\begin{proof}
The proof is a direct generalization of the proof of Theorem~\ref{theorem1} and will be further presented in Appendix~\ref{appendixA1}. 
\end{proof}




Now, given the relations \eqref{Update_W_general} for the matrices $W^{(l)}$, $l = 1, \ldots, L$, the following update equations can be used (likewise in a fixed-point sense):
\begin{align}\label{Update_W_gneral_real}
\forall d,h:\phantom{iii} (W^{(l)}_{dh})^{\mathrm{new}} = \frac{\sum_{n = 1}^N \mathbb{E}_{q^{(n)}}  [ \mathcal{A}_{dh}(\vec{s},\Theta^{\mathrm{old}}) ] \, T_l(y_d^{(n)})}{\sum_{n = 1}^N \mathbb{E}_{q^{(n)}}  [ \mathcal{A}_{dh}(\vec{s},\Theta^{\mathrm{old}}) ] }.
\end{align}
For the proposed EM algorithm, now in the general case, we initialize dictionaries $W^{(l)}_{dh}$ for $l = 1, \ldots, L$ and update their elements in each M-step using \eqref{Update_W_gneral_real}. After convergence, the converged dictionaries represent the optimal (possibly local optima) parameters of the model. Moreover, note that the prior parameter update remains unchanged and is given by \eqref{Update_pi}. Finally, also note that the application of EVO as variational acceleration is possible in the same way as for $L=2$.

Having obtained a generic optimization procedure for general EF-MCA, let us finally emphasize again in what sense the procedure is directly applicable.
Optimization for the general case iterates M-step equations~\eqref{Update_W_gneral_real} and \eqref{Update_pi} and the E-step (which can optionally use full posteriors or the truncated posteriors in \eqref{eq:TVEMq}). Given a distribution of the exponential family, we require for those equations the following: The sufficient statistics $\vec{T}(y)$, the log-partition $A(\vec{\tilde{\eta}}_d(\vec{s},\Theta))$, and the function $\mathcal{A}_{dh}(\vec{s},\Theta)$. For essentially all exponential family distributions conventionally used as noise models in LVMs, the sufficient statistics and log-partition functions $A(\vec{\eta}\,)$ are efficiently computable. Also the mapping $\vec{\tilde{\eta}}_d(\vec{s},\Theta)$ from the standard parametrization to
the natural parameters is usually a closed-form function. Only the function $\mathcal{A}_{dh}(\vec{s},\Theta)$ of \eqref{Adh} can contain a potential analytical intractability. The function is defined using $h(d,\vec{s}, \Theta)$ of Eqns.~\eqref{etagen} to \eqref{maxhgen} and involves the ($L$ dimensional) mapping $\vec{\Phi}$ from
the mean parameters to the natural parameters. For many examples the mapping is known to be closed-form (e.g., Bernoulli, Poisson, Exponential, Gaussian, Categorical etc.) or partly closed-form (e.g.,  Gamma). If the mapping is not closed-form (e.g., Beta, Gamma), then our optimization inherits precisely the
same analytical intractability problem as parameter optimization for the single distribution itself already exhibits. However, the intractability problem is in all common cases 
low-dimensional (usually $L=1$ or $2$), and the function $\vec{\Phi}$ is smooth and one-to-one, i.e., analytical or numerical solutions are usually available.

The here derived generic optimization thus contrasts with more formal frameworks of EM-like optimization that are discussed in the context of exponential families and/or graphical models. Usually such frameworks rest on assumptions for the variational distributions such as mean-field assumptions \citep[e.g.][]{SaulEtAl1996,Jaakkola2001}; or assumptions for the generative model itself \citep[][]{wainwright2008graphical}. For instance, if we assume the generative model to have an exponential family representation of its {\em joint} distribution, it is (under mild conditions) straightforward\footnote{One needs to use the exponential family form of the joint. Then the derivative of $\mathcal{F}(q, \Theta)$ in \eqref{appendix_ELBOeq} has to be computed using the equivalent of \eqref{A_eta_expected_value} for the joint distribution. Mild conditions have to apply for the parametrization
from standard parameters $\Theta$ to natural parameters $\vec{\zeta}$ which are usually fulfilled.} to derive a general M-step which takes the form:
$\sum_{n} \big( \mathbb{E}_{q^{(n)}} [\vec{T}(\vec{y}, \vec{s}\,)]
- \mathbb{E}_{p(\vec{y},  \vec{s} \,  |\,  \vec{\zeta}(\Theta) )} [\vec{T}(\vec{y}, \vec{s}\,)] \big) = \vec{0}$, 
where $\vec{T}(\vec{y}, \vec{s}\,)$ denotes the joint sufficient statistics given an observation $\vec{y}$ and a latent variable $\vec{s}$, 
 and $\vec{\zeta}(\Theta)$ represents the natural parameter of the joint.
%
Such a general M-step can serve theoretical considerations, or it can serve as a starting point to derive efficient optimization procedure \citep[e.g., if mean-field approaches are developed,][]{BleiEtAl2017,wainwright2008graphical}. It does by itself not  necessarily 
represent a directly applicable optimization procedure, however.  
In the case of the here considered EF-MCA models, a joint exponential family which is {\em genuine} \citep[i.e., truly non-curved, see][]{Efron1975,BickelDoksum2015,Efron2023} is not available, and differentiability of a (joint) partition function is not given because of the used maximum non-linearity. As a further example, energy models \citep[compare][]{HintonEtAl2006,LeCunEtAl2006,FischerIgel2014} usually do have a joint exponential family distribution by definition. However, their partition function is usually computationally intractable.

\section{Numerical Verification and Example Applications}
\label{section_experiments}
We now numerically verify the derived optimization procedure using different well-known distributions of the exponential family as examples.
Concretely, we use Bernoulli, Poisson and Exponential as one-parameter distributions, and Gaussian, Gamma and Beta as two-parameter distributions of the family, and optimize the corresponding EF-MCA data models
with EM presented in Alg.~\ref{algorithm1}. We use the results of Theorem~\ref{theorem1} (and Theorem~\ref{theorem2} as the general case) for the M-steps. For the E-steps, we use full posteriors when the used models
are sufficiently small or apply approximations \eqref{eq:TVEMq} otherwise\changedCR{\footnote{The codes are available in \url{https://github.com/tvlearn/evo}.}}.
For technical details, we refer to Appendix~\ref{appendixB},  and also refer to \citet[][]{MousaviEtAl2021} for further examples of the approach applied to Beta-distributed observables. 

\subsection{Artificial Data}
\label{sec:ArtificialData}
First, an example with artificial data and known data-generating parameters $\Theta^{\mathrm{gen}}$ (the ground-truth parameters) was considered to verify and evaluate the derived optimization procedure.   
Specifically, using the setup of a bars test \citep[][]{Foeldiak1990,LuckeSahani2008,DrefsEtAl2022}, 
we were interested in how well the derived update equations recover ground-truth parameters.  
We generated data using the generative model (\ref{modelEq1})-(\ref{modelEq3}) itself using link (\ref{eta}) with (\ref{meanf_eq}) and (\ref{maxh}), 
%
and considered Exponential and Bernoulli distributions for observables to assess the efficacy of the proposed updates \eqref{update_W_real} and \eqref{Update_pi} in increasing the log-likelihood of the data. We refer to the corresponding EF-MCA models as Exponential-MCA (E-MCA) and Bernoulli-MCA (B-MCA) models, respectively.

\begin{figure}[t!]
\centering
\includegraphics[scale=0.5]{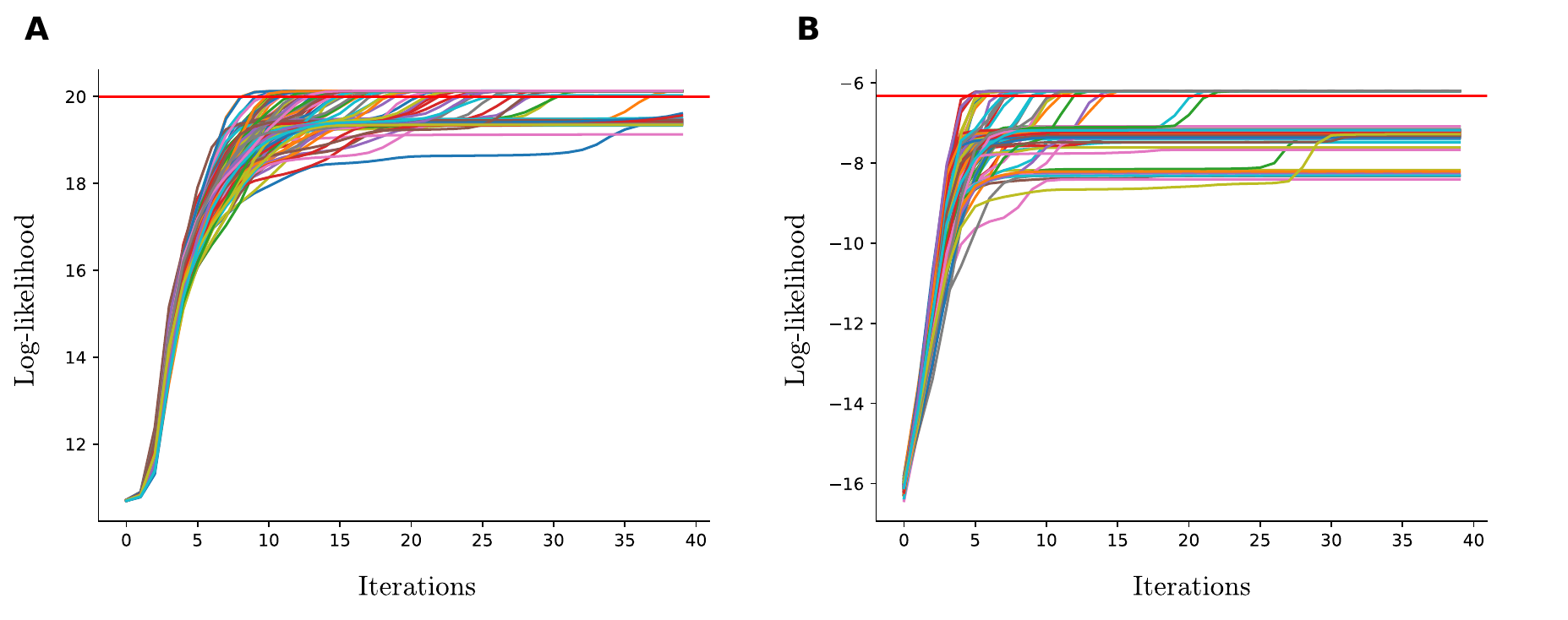}
\caption{Behaviour of the log-likelihood functions corresponding to (\textbf{A}) E-MCA and (\textbf{B}) B-MCA for 100 runs using artificial bars test. Different colors denote different runs of the algorithms.  
Both algorithms result in a monotonic increase of the log-likelihood function and demonstrate fast convergence (at least to a local optimum). 
Out of 100 runs, 71 for E-MCA and 29 for B-MCA, respectively, converged to the global optimum 
(i.e.,  to the ground-truth log-likelihood  denoted by the red horizontal line). The best runs achieved slightly higher log-likelihoods compared to the ground-truth, which may be attributed to the effect of overfitting due to the finite sample sizes.  
}
\label{fig_100runs}
\end{figure}

For data generation and for the models, we used $H = 10$ generative fields $\vec{W}_h$ in the form of horizontal and vertical bars each occupying $5$ pixels on a $D = 5 \times 5$ grid.
The latents were sampled independently (according to the prior) with probability $\pi_h$ for $h = 1, \ldots, H$, and the corresponding generative fields were superimposed non-linearly according to \eqref{maxh}.
%
We generated two sets of data points using the given EF-MCA models: One with Exponential distribution as the observation noise and another with Bernoulli distribution.
%
We then trained the corresponding EF-MCA models using each of the generated datasets: E-MCA was trained using the data with Exponential noise and likewise B-MCA using the data with Bernoulli noise. 
%
Details of the experimental settings and parameter initializations are presented in  Appendix~\ref{bars_test_appendix}.

We repeated this procedure 100 times and closely examined the behaviour of the log-likelihood functions for each of the runs. The results are depicted in Fig.~\ref{fig_100runs}. 
As it can be seen, the application of the optimization procedure (Theorems~\ref{theorem1} and \ref{theorem2}) gives rise to a robust and reliable algorithm that in practice monotonically increases the log-likelihood of the data to at least local log-likelihood optima (note that we used full posteriors here). Also, the 
figure illustrates that the algorithms converge to ground-truth log-likelihood values (log-likelihood values computed using the generative parameters; red horizontal lines in Fig. \ref{fig_100runs}) in some cases. 
In these runs, the learned parameters are very close to the generating ones with small differences that could be due to finite sample size effects (we have observed that increasing the number of data points instantly increases the estimation accuracy).  
Nevertheless, local optima can be observed for some cases where the learned log-likelihood values differ significantly from the ground-truth.  
This effect appears to be pronounced to varying degrees in different models. 
%
For instance, in the P-MCA model, we here observed 83 out of 100 runs that converged to global optimum. 
In comparison, the P-MCA approach of \citet[][]{LuckeSahani2008}, when  trained using exact EM, converged to global optimum in 92 out of 100 runs, i.e.  more often compared to the P-MCA model studied here. 
%
Such a difference, however, is expected since the latter model exploits an additional annealing procedure to avoid local optima.  While such annealing approaches have been commonly investigated in previous studies, our focus in this work remains on evaluating the monotonically increasing log-likelihood function for each of the EF-MCA data models (we will discuss further details in Appendix~\ref{bars_test_appendix}).  
Also Fig.~\ref{fig_ExponentialBars} displays the model parameters learned in one such run for the E-MCA algorithm.  

%
Besides the local optima effect, we also found the sparseness (the value of $\sum_h \pi_h$ or simply $\pi H$ if we use the same $\pi$ for all $h$) to be an important factor influencing the outcome of the algorithm: The \textit{reliability} of the algorithm (the probability of recovering all model parameters; compare, e.g., \cite{Spratling1999}) showed to decrease as the sparseness increases (see  Appendix~\ref{bars_test_appendix} for details).  

%
%
%
\begin{figure}[t!]
\centering
\includegraphics[scale=0.72]{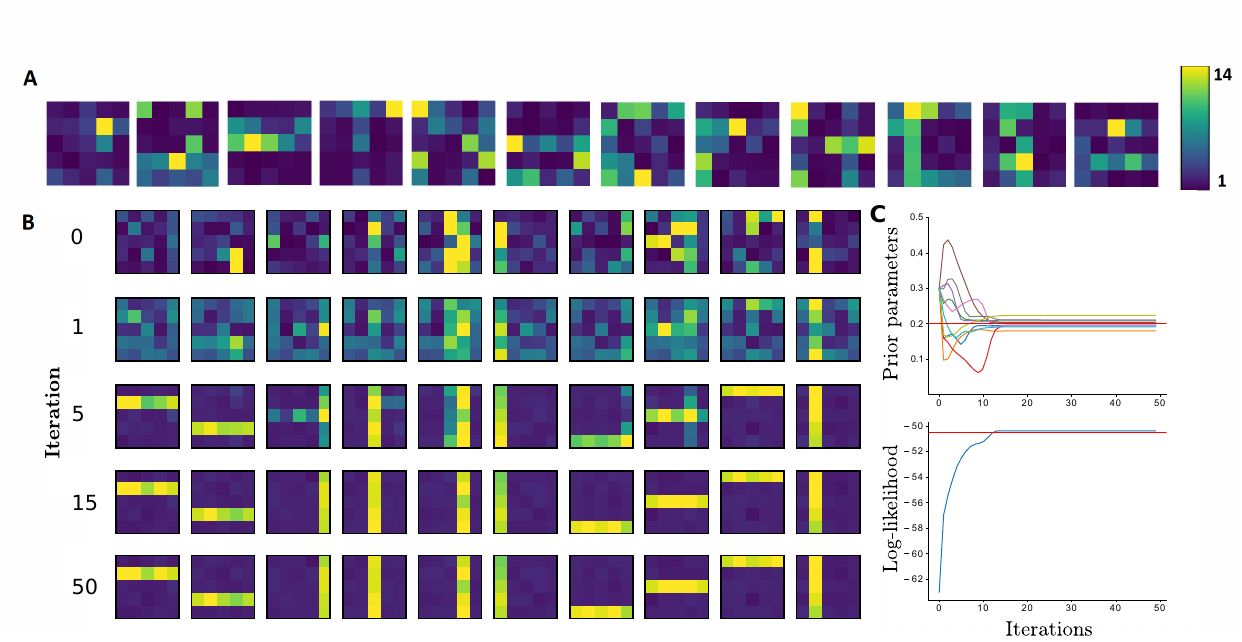}
\caption{\textbf{A} 12 examples of input data. \textbf{B} Learned generative fields using the E-MCA model with the number of iterations illustrated on the left hand side. \textbf{C}~Behaviour of the $\pi_h$ values for $h = 1, \ldots, H$ 
(top) and the log-likelihood function (bottom) given the 50 EM iterations. Here the generating parameter is $\pi_h^{\mathrm{gen}}=0.2$ for all $h$ (i.e.,  the sparseness is 2); 
see Appendix~\ref{bars_test_appendix} for more details.
}
\label{fig_ExponentialBars}
\end{figure}

\subsection{More Realistic Data}
%
After the initial verification above, let us now point to some example applications on real data. 
Applications often require large models such that computing full posteriors becomes infeasible.
To scale the algorithms to larger sizes, we applied variational E-steps
as described in Sec.\,\ref{sec:estep} using $|\KKn| = 60$ variational states for each $n$.
%

\subsubsection{Feature Extraction -- Natural Image Patches} %
\label{sec_natural_image_patches}
For the first application to real data, we considered the model of Example~\ref{example2}, i.e., Gaussian-MCA.
As discussed, the model has two matrices that we can optimize using
Theorem~\ref{theorem1}: $W$ for the mean and a combination of $V$ and $W$ for the variance of the Gaussian. To facilitate interpretation, we reparameterized the matrices back
to the normal Gaussian parameters by setting $\sigma^2_{dh}=V_{dh}-W^2_{dh}$ (see Appendix~\ref{appendixA2} for details).
%
Also, we followed \cite{BornscheinEtAl2013} and used a maximum magnitude non-linearity, i.e., we replaced $h(d, \vec{s}, \Theta)$ in \eqref{maxh} by (since $M(\Theta) = W$ for  Gaussian): 
\begin{equation}
h(d, \vec{s}, \Theta) = \argmax_h \{ | W_{dh} s_h | \},
\label{eq:mmca_combination_rule}
\end{equation}  
where $| x |$ denotes the absolute value of $x$. 
Such choice does not affect our main results from Theorems~\ref{theorem1} and \ref{theorem2} as the relations in \eqref{derivatives_WbarVbar1}-\eqref{derivatives_WbarVbar2} and also the proofs  (in Appendix~\ref{appendixA0}) remain valid.

%

Using a set of $N=100,000$ ZCA-whitened image patches of size $D=20 \times 20$ pixels taken from a database of natural scenes \citep[][]{HaterenSchaaf1998}, we trained a Gaussian-MCA model with $H=1,000$ components to learn individual dictionaries for component means and component variances.
After training, we observed a large variety of generative fields (GFs) for the component means (including the familiar Gabor-like and globular fields that have likewise been observed in previous studies using MCA models; e.g., \citealt{LuckeSahani2008,BornscheinEtAl2013}) together with a large variety of GFs for component variances. Figure \ref{fig:image_patches20} illustrates 20 examples of such GFs (see Appendix~\ref{Natural_images_appendix} for the full dictionaries).
\begin{figure}[t!]
\centering
\includegraphics[width=\textwidth]{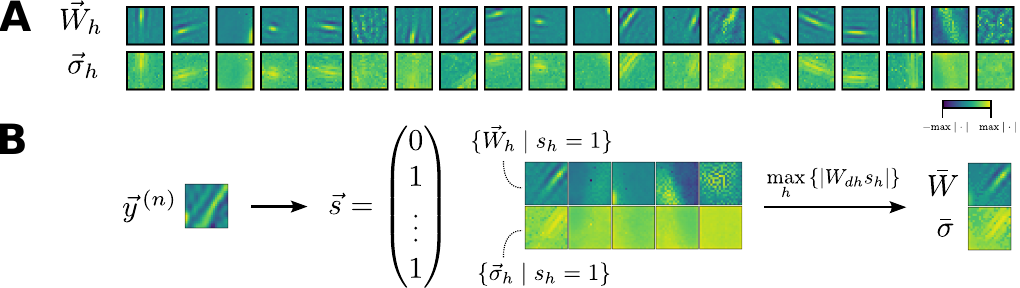}
\caption{{\bf A}~Exemplary dictionary elements $\vec{W}_{h}$ and $\vec{\sigma}_h$ learned by Gaussian-MCA on natural image patches of $D=20\times20$ pixels ($20$ generative fields each; we depict standard deviations $\sigma_{dh}$ instead of $V_{dh}$ to facilitate interpretation). See Appendix~\ref{Natural_images_appendix} for all $H=1000$ dictionaries.
{\bf B}~Illustration of data encoding. The left part depicts an exemplary training data point. The middle part depicts a latent vector $\sVec$ with high posterior probability as well as the dictionary elements of the components~$h$ associated with the non-zero entries in $\sVec$. These components effectively determine 
the Gaussian mean $\bar{W}$ and standard deviation $\bar{\sigma}$ through a point-wise maximum magnitude non-linearity
as shown on the right part
($\bar{W}=\{W_{d h(d, \vec{s}, \Theta)} \}_{d=1}^{D}$ and $\bar{\sigma}=\{\sigma_{d h(d, \vec{s}, \Theta)} \}_{d=1}^{D}$ with $h(d, \vec{s}, \Theta)$ given by \eqref{eq:mmca_combination_rule}).}
\label{fig:image_patches20}
\end{figure}
The observed variety of variance GFs stands in contrast to a uniform variance with equal value for all latents and observables (i.e., $\sigma_{dh}=\sigma, \forall d,h$) as assumed by standard SC or previous MCA versions. Compared to such previous approaches, a generalized encoding capturing component means and component variances
may provide a more elaborate model of the combination of first- and second-order statistics in real-world datasets. 
%
Such variance GFs may be used, e.g.,  to model intricate image structures, and/or to further analyze the effect of added noise on different parts of the image patches \citep[also see][]{mousavi2020ddDictionary,MousaviEtAl2021}.  

Besides,
the data and model dimensionalities (i.e., the values of $N$, $D$, $H$) used here demonstrate the scalability 
of the proposed EF-MCA data models when
accelerated by EVO (Sec.\,\ref{sec_EVO}).



\subsubsection{Noise type estimation}
\label{subsec:noise_type_estimation}
Next, we considered the problem of determining the unknown noise type of a given dataset. This problem is of importance for data analysis in a number of applications and has been actively researched during the last years (see Appendix~\ref{app:noise_type_estimation} for a discussion of related work). 
Here, the results of Theorems~\ref{theorem1} and \ref{theorem2} enabled us to develop approaches for determining the type of the data distribution. As a proof-of-concept in this direction, we investigated how, e.g., Gaussian-, Gamma- and Beta-MCA models can be used to determine which of the three noise distributions is better suited for a given dataset. 
To quantify suitability, we compared variational lower bounds (as approximations to log-likelihoods) obtained with the three models on a given dataset. 
Note that the considered models have the same number of parameters such that we can directly use their lower bounds for model selection (i.e., we do not have to consider penalty terms for the number of parameters; see, e.g., \citet[][]{MacKay2003} for a comparison regarding AIC or BIC). 

We investigated noise type estimation using both visual and acoustic data.  
We also refer the reader to Appendix~\ref{app:noise_type_estimation} for technical details of the experiments, 
 Examples~\ref{example1}, \ref{example2} and Appendices~\ref{appendixA2} and  \ref{appendixA3} for detailed information regarding the parametrization of the Gaussian-MCA and the Gamma-MCA models, respectively, as well as to \citep[][]{MousaviEtAl2021} for the details of the Beta-MCA.


\paragraph{Visual Data.}
We here considered standard benchmark images (depicted in Fig. \ref{fig:original_images}), to which we either added Gaussian, Gamma or Beta noise. We then trained the three models of Gaussian-, Gamma- and Beta-MCA on each of the noisy images.  
In order to have a reliable comparison,  we ran each algorithm three times (each time with a new set of initial parameters and  for 100 variational EM iterations),  and picked the one with the highest ELBO among these three runs.  
The results are presented in Tab.\,\ref{table:house}.
%
%
As can be observed and expected, the lower bound is, in most cases,  highest when the noise distribution assumed by the model matches the actual noise in the data.  
Factors potentially causing the few discrepancies between the best model (in terms of lower bound) and actual noise type, as well as the small differences between the lower bounds that observed in some cases may include, e.g.,  local optima effects and/or the noise levels which we have used for the experiments (we used noise level of 50 for the cases of Gaussian or Gamma noise and 0.1, 0.2 or 0.3 for the case of Beta noise).

\begin{table}[!t]
\caption{Comparison of the obtained ELBO values by the three different EF-MCA data models when trained on images corrupted with three different noises: Gaussian, Gamma and Beta were investigated here. Corrupted images were linearly rescaled to the interval $(0,1)$ for the sake of a reliable comparison. In addition,  we trained all the models with the same settings of hyper-parameters. For the case of Gamma-MCA (and also Beta-MCA for some cases), we used the approach presented in Appendix~\ref{bars_test_appendix} to avoid local optima. The results correspond to the highest value of the ELBO for each model among three different runs, and the bold numbers represent the best fit (for each noisy image) among the three models. We refer the reader to Sec.~\ref{subsec:noise_type_estimation} and Appendix  \ref{app:noise_type_estimation} for more details.}
\centering
\begin{tabular}{c c c c c c c c c}
\hline
Noise & Model & House & Camera & Peppers & Saturn & Bridge & Flag & Average \\ \hline
\multirow{3}{*}{Beta} & Beta-MCA & \textbf{39.00} & \textbf{-28.89} & \textbf{37.93} & 57.82 & \textbf{31.29} & 78.88 & \textbf{36.00} \\
& Gamma-MCA & 15.39 & -29.02 & 25.80 & \textbf{58.96} & 16.23 & \textbf{97.25} & 30.76 \\
& Gaussian-MCA & 16.89 & -42.36 & 14.73 & 32.31 & 7.99 & 66.00 & 15.92 \\
\hline
\multirow{3}{*}{Gamma} & Beta-MCA & 125.16 & 251.87 & 215.27 & \textbf{562.55} & 159.46 & 138.78 & 242.18 \\
& Gamma-MCA & \textbf{125.91} & \textbf{266.36} & \textbf{217.71} & 554.82 & \textbf{159.47} & \textbf{139.96} & \textbf{244.03} \\
& Gaussian-MCA & 116.38 & 232.57 & 200.05 & 534.19 & 143.00 & 118.13 & 224.05 \\
\hline
\multirow{3}{*}{Gaussian} & Beta-MCA & 123.43 & 143.77 & 119.66 & 150.23 & 116.20 & \textbf{146.02} & 133.21 \\
& Gamma-MCA & 121.70 & 143.15 & 118.50 & 148.65 & 114.11 & 143.97 & 131.68 \\
& Gaussian-MCA & \textbf{124.15} & \textbf{144.42} & \textbf{120.63} & \textbf{151.30} & \textbf{117.05} & 145.67 & \textbf{133.87} \\
\hline
\end{tabular}
\label{table:house}
\end{table}

\paragraph{Acoustic Data.}
As a further example for data with natural noise sources,  
we considered the CHiME dataset \citep[][]{FosterEtAl2015} and 
fitted Gaussian- and Gamma-MCA models to amplitude spectrograms of three selected audio files.  
For both algorithms, we performed three runs for each audio example, and for comparison then selected the one with the highest ELBO (which can be computed efficiently). The resulting ELBO values are shown in Tab.\,\ref{table:chimeresults}.
The higher ELBOs observed for Gamma-MCA suggest the Gamma distribution as a better suited noise model compared to Gaussian for the considered data, which is consistent with amplitude data being positive. While differences of the ELBO values across runs showed to be comparably small, larger differences across audio examples and algorithms could be observed. These may be attributed to, among other things, differences between the content and noise sources in the individual audio files (see Appendix~\ref{app:noise_type_estimation} for data-related details).

%


%
%
%

The experiments reveal the efficacy of the updates (\ref{update_W_real}) and (\ref{update_V_real}) 
and demonstrate applicability of the algorithms for noise type estimation at large scales. 
Model selection using other types of noise distributions can proceed along the same line but require a more elaborate treatment \citep[compare, e.g.,][]{valera2017automatic} which exceeds the purposes of this study.

\subsubsection{Denoising}
\label{Sec_denoising}

Finally, we used the presented algorithms to denoise images corrupted by non-Gaussian noises.
This task can be considered more difficult compared to the removal of additive white Gaussian noise which numerous established denoising algorithms are optimized for  \citep[compare, e.g.,][]{DabovEtAl2007,GuEtAl2014,BurgerEtAl2012,ChaudhuryAndRoy2017,ChenAndPock2017,ZhangEtAl2017,TaiEtAl2017,ZhangEtAl2018}. We investigated two different noise types, namely Poisson and Exponential, and used standard test images for benchmarking (see Appendix~\ref{denoising} and Fig.~\ref{fig:original_images} for details on the images used). 
%
Concretely,  
we applied Poisson-MCA (P-MCA) to denoise images corrupted by Poisson noise and analogously Exponential-MCA (E-MCA) to denoise images corrupted by Exponential noise. 
The family of EF-MCA models allows for treating a much larger variety of noise distributions, of course, but focusing on P- and E-MCA may illustrate
the potential also for other noise types. For denoising images with Beta noise (using the Beta-MCA model), the reader may be referred to \citet[][]{MousaviEtAl2021}.

\begin{table}[!t]
\caption{Results of the noise type estimation experiment with natural noise sources (see Sec.~\ref{subsec:noise_type_estimation} and Appendix~\ref{app:noise_type_estimation} for details). Listed are the highest ELBO values among three different runs given each of the considered  audio examples.}
%
%
\centering
\begin{tabular}{c c c}
\hline
 & Gaussian-MCA & Gamma-MCA \\ \hline
audio example 1 & 215.20 & \textbf{298.53}\\
audio example 2 & 95.61 & \textbf{250.77}\\
audio example 3 & 47.36 & \textbf{206.14} 
\end{tabular}
\label{table:chimeresults}
\end{table}

In general, one can see denoising as a (non-)linear inverse problem  \citep[compare][]{tarantola2005inverse} where the task is to estimate the underlying non-noisy image using a specific method. Formally, given a set $Y_{\mathrm{corrupted}} = f(Y_{\mathrm{original}})$ of corrupted data points ($f$ denotes a noise operator, in this case Poisson or Exponential), we aim at estimating the uncorrupted data $Y_{\mathrm{original}}$ by $Y_{\mathrm{estimate}}$. 
A variety of different approaches have been investigated in the literature in order to either obtain the corresponding $Y_{\mathrm{estimate}}$ using a so-called \textit{denoiser} (like the approaches mentioned above), or to somehow improve the performance of a given denoiser (either for denoising or for other relevant tasks) using established methods 
like the work by, e.g., \citet[][]{kadkhodaie2021stochastic} that uses the prior embedded within a denoiser. 
%
%
We here employed the proposed EF-MCA data models to restore the non-noisy images by \textit{only} observing a noisy version of the clean images. That is, P-MCA and E-MCA models were directly applied to the noisy image itself (using overlapping patches) without leveraging external, clean training data. 
This task is also referred to as \textit{zero-shot} denoising  \citep[compare, e.g.,][]{ShocherEtAl2018,ImamuraEtAl2019}. 
In detail, we followed \citet[][]{DrefsEtAl2022} and, given that $T(y) = y$ for Poisson and Exponential distributions (i.e., $M(\Theta) = W$), used the following formula\footnote{Estimators of other distributions would use the corresponding $M(\Theta)$ matrix  \citep[see][for denoising images corrupted by Beta noise]{MousaviEtAl2021}.} to estimate the non-noisy image pixels:
\begin{equation}
\big(\vec{y}^{\,(n)}_{\text{estimate}}\big)_d = \mathbb{E}_{q^{(n)}}  [ \bar{W}_d(\sVec, \Theta) ],
\end{equation}
where $d$ corresponds to pixel $d$ of the image patch $n$. 
%

\begin{figure}[t!]
\centering
\includegraphics[scale=0.34]{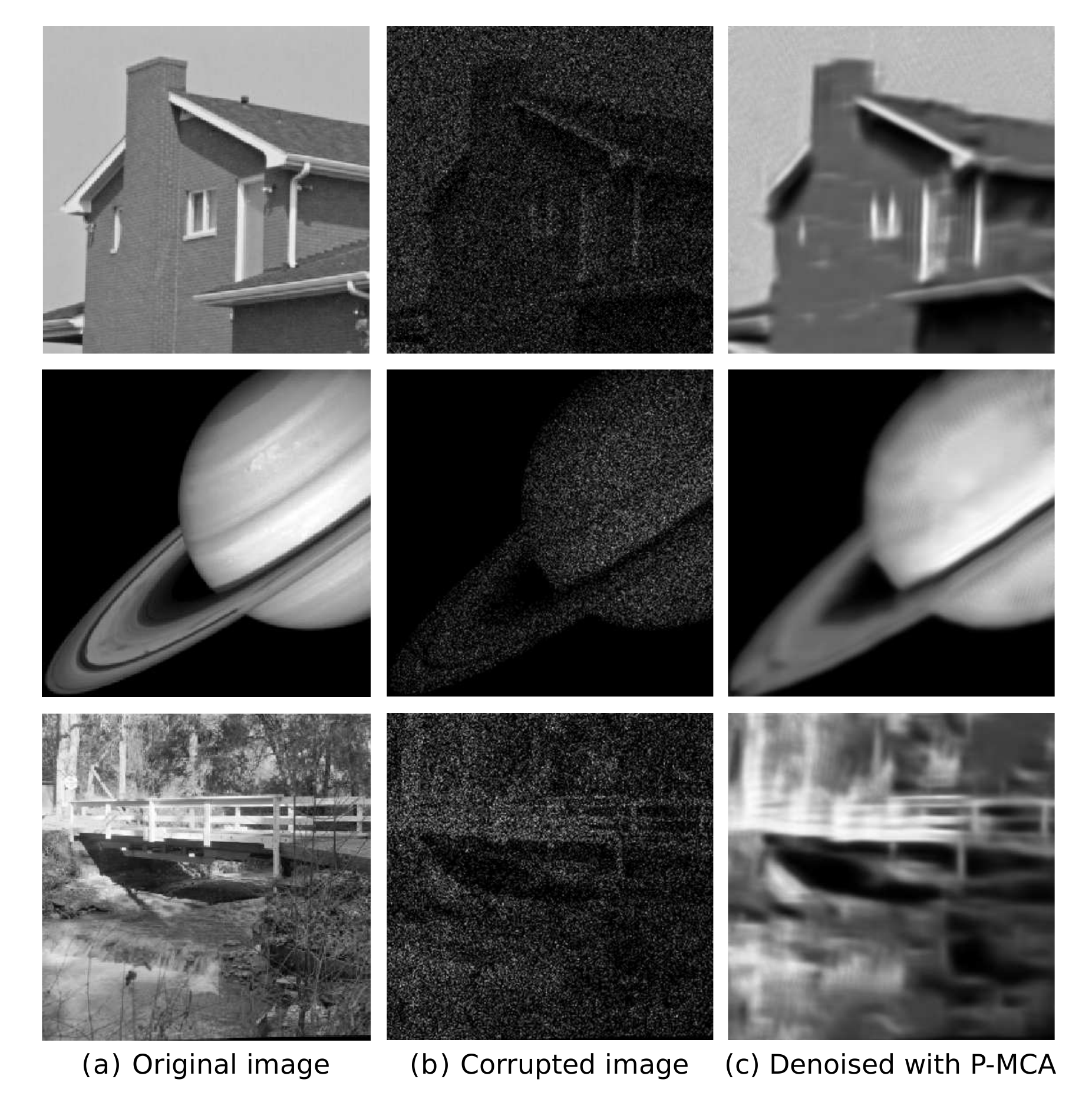}
\caption{Denoising of the House, Saturn and Bridge images when corrupted by Poisson noise (at peak value 2). The figure illustrates the original and noisy images together with the reconstructed images using P-MCA trained with $H = 100$ and $D = 400$ for 100 variational EM iterations. The corresponding PSNR values of the reconstructed images are 24.66 (House), 28.39 (Saturn) and 20.15 (Bridge); 
while the PSNR values of the noisy images are 4.94 (House), 7.11 (Saturn) and 6.21 (Bridge).
In addition, the denoised House image can be compared with Fig. 9 of \citet[][]{giryes2014sparsity} which depicts the reconstructed images using BM3D, SPDA and NLSPCA methods.}
\label{fig:Poisson_denoising}
\end{figure}

\paragraph{Poisson noise.} 
After Gaussian denoising, Poisson noise removal benchmarks do presumably provide the most extensive opportunities for quantitative comparisons of different denoising algorithms.
%
%
%
%
One possibility for denoising an image corrupted by Poisson noise is to simply ignore the specific properties of Poisson noise and employ a method that is tailored to Gaussian noise. For instance, 
BM3D
 \citep{DabovEtAl2007} can be applied directly. BM3D is a well-known denoising method which often outperforms SC or k-SVD \citep[][]{aharon2006k} algorithms on image denoising benchmarks. The original method is tailored to 2D-images corrupted by Gaussian noise, however.

A frequently followed alternative approach is converting the Poisson noise to Gaussian, and then applying an algorithm which assumes Gaussian observables. Such methods apply a non-linear variance-stabilizing transformation (VST) such as Anscombe \citep[][]{anscombe1948transformation} or Fisz \citep[][]{fisz1955limiting} that produces a signal in which the noise can approximately be treated as Gaussian with unit variance.
Approaches such as VST+BM3D use the Anscombe root transformation followed, in this case, by standard BM3D \citep[e.g.,][]{makitalo2010optimal}.
Such approaches, however, are specific to Poisson noise as similar transformations are in general not available for all exponential family distributions. But also in the case of Poisson noise, transformation to Gaussian noise can be problematic. In fact, it is known that transformation-based methods can yield poor results for low-intensity signals  \citep[][]{salmon2014poisson,makitalo2010optimal}.  
It should be noted that the Poisson noise is not additive and the peak value of the original image 
determines the strength of the noise (the lower the peak is, the stronger the noise will be).  
For low peak values (e.g. lower than 3), VSTs are less effective.
 As a result, many approaches have been suggested to improve such transformations or exploit alternative approaches 
in the case of low-intensity signals \citep[e.g.,][]{salmon2014poisson,makitalo2010optimal,rond2016poisson,azzari2016variance,giryes2014sparsity}.
One example is the P$^4$IP approach \citep[][]{rond2016poisson} that applies a general plug-and-play prior approach on Poisson inverse problems. The method still relies on Gaussian-based denoising algorithms, however. Other studies \citep[such as][]{makitalo2010optimal,azzari2016variance} attempt to improve the transformation techniques in order to increase the performance of the methods. 
A prominent work is an iterative VST framework, known as I+VST+BM3D, introduced  
by \cite{azzari2016variance} that, in a series of consecutive iterations, combines the noisy observation with a previously estimation of the image in order to improve the signal-to-noise ratio (SNR). 
The algorithm can then cope with extremely low SNR cases and produce state-of-the-art denoising results.

\begin{table}
\caption{Comparison of the PSNR values (in terms of dB) for the considered Poisson denoising benchmarks. 
The algorithms are divided into two groups for peaks 0.2 and 0.5, and three groups for peaks 1 and 2  (values of DenoiseNet and DCEA-UC do not exist for peaks 0.2 and 0.5). NLSPCA, SPDA and P-MCA are LVMs that assume a Poisson noise while the others  
are task-specific and highly optimized algorithms that either assume a Gaussian denoiser (the BM3D-based methods and P$^4$IP) or a DNN-based approach that is trained supervised (DenoiseNet and DCEA-UC). 
For each peak value and each image, the bold numbers denote the best PSNRs in comparison to the other models of the corresponding subgroup; and the underlined number denotes the best PSNR among all models. Listed values are averages over five noise realizations. See Secs.~\ref{Sec_denoising} and  \ref{denoising} for details.
}
\label{table:poisson_denoising}
\centering
\begin{tabular}{c c c c c c c c}
\noalign{\global\arrayrulewidth=0.3mm}
\arrayrulecolor{black}\hline
Method & Peak & House & Camera & Peppers & Saturn & Bridge & Flag \\ 
\noalign{\global\arrayrulewidth=0.3mm}
\arrayrulecolor{black}\hline
BM3D & 0.2 & 18.37 & 17.35 & 17.10 & 22.02 & 17.09 & 14.28 \\ 
VST+BM3D &  & 17.79 & 16.90 & 16.96 & 21.38 & 17.12 & 13.53 \\
I+VST+BM3D &   & \textbf{19.68} & \textbf{18.40} & \underline{\textbf{17.54}} & \textbf{23.13} & \underline{\textbf{18.13}} & \underline{\textbf{17.49}} \\
P$^4$IP &   & 19.48 & 17.82 & 17.31 & 23.05 & 17.54 & 14.82 \\
\noalign{\global\arrayrulewidth=0.01mm}
\arrayrulecolor{gray}\hline
SPDA &  & 17.83 & 16.93 & 16.75 & 21.52 & 16.80 & 16.58 \\
NLSPCA &  & 18.91 & 17.79 & 17.45 & 22.90 & 17.46 & 16.48 \\ 
P-MCA ($H = 20$) &   & \underline{\textbf{19.69}} & 18.19 & 17.35 & 23.00 & 17.45 & 16.75 \\
P-MCA ($H = 30$) &   & 19.36 & \underline{\textbf{18.48}} & \textbf{17.46} & \underline{\textbf{23.30}} & \textbf{17.51} & \textbf{17.09} \\
\noalign{\global\arrayrulewidth=0.3mm}
\arrayrulecolor{black}\hline
BM3D & 0.5 & 20.27 & 18.83 & 18.49 & 23.86 & 18.24 & 15.87 \\ 
VST+BM3D &  & 19.61 & 18.46 & 18.41 & 23.75 & 18.26 & 15.58 \\
I+VST+BM3D &  & \textbf{21.54} & \underline{\textbf{19.79}} & \underline{\textbf{19.05}} & \textbf{25.78} & \underline{\textbf{19.08}} & \textbf{18.60} \\
P$^4$IP &  & 20.93 & 19.27 & 18.86 & 25.19 & 18.47 & 16.50 \\
\noalign{\global\arrayrulewidth=0.01mm}
\arrayrulecolor{gray}\hline
SPDA &  & 20.51 & 18.90 & 18.66 & 25.50 & 18.46 & \underline{\textbf{19.67}} \\
NLSPCA &  & 20.85 & 19.23 & 18.78 & 24.91 & 18.50 & 18.80 \\ 
P-MCA ($H = 20$) &  & 21.58 & 19.49 & \textbf{18.96} & \underline{\textbf{25.79}} & \textbf{18.52} & 18.50 \\
P-MCA ($H = 30$) &  & \underline{\textbf{21.86}} & \textbf{19.67} & 18.94 & 25.48 & 18.47 & 19.06 \\
\noalign{\global\arrayrulewidth=0.3mm}
\arrayrulecolor{black}\hline
DenoiseNet & 1 & 22.87 & \underline{\textbf{21.59}} & \underline{\textbf{21.43}} & 26.26 & 19.83 & 19.45 \\
DCEA-UC &  & \textbf{23.00} & 21.47 & 20.91 & - & \underline{\textbf{19.87}} & - \\
\noalign{\global\arrayrulewidth=0.01mm}
\arrayrulecolor{gray}\hline
BM3D &  & 22.35 & 20.37 & 19.89 & 25.89 & 19.22 & 18.31 \\ 
VST+BM3D &  & 21.64 & 20.19 & 19.71 & 25.82 & 19.43 & 18.46 \\
I+VST+BM3D &  & \textbf{23.04} & \textbf{21.07} & \textbf{20.44} & \textbf{27.27} & \textbf{19.86} & \textbf{19.74} \\
P$^4$IP &  & 22.67 & 20.54 & 20.07 & 27.05 & 19.31 & 19.07 \\
\noalign{\global\arrayrulewidth=0.01mm}
\arrayrulecolor{gray}\hline
SPDA &  & 22.73 & 20.23 & 19.99 & 27.02 & 19.20 & \underline{\textbf{22.54}} \\
NLSPCA &  & 22.09 & 20.32 & 19.62 & 26.89 & 18.94 & 20.26 \\ 
P-MCA ($H = 100$) &  & \underline{\textbf{23.42}} & \textbf{20.54} & \textbf{20.35} & \underline{\textbf{27.46}} & \textbf{19.34} & 21.28 \\
\noalign{\global\arrayrulewidth=0.3mm}
\arrayrulecolor{black}\hline
DenoiseNet & 2 & \textbf{24.77} & \underline{\textbf{23.25}} & \underline{\textbf{23.19}} & 28.37 & 20.80 & \textbf{21.38} \\
DCEA-UC &  & 24.52 & 22.94 & 22.94 & - & \underline{\textbf{20.82}} & - \\
\noalign{\global\arrayrulewidth=0.01mm}
\arrayrulecolor{gray}\hline
BM3D &  & 24.18 & 22.13 & 21.97 & 27.42 & 20.31 & 20.81 \\ 
VST+BM3D &  & 23.79 & 21.97 & \textbf{22.02} & 27.95 & 20.49 & 20.79 \\
I+VST+BM3D &  & 24.62 & \textbf{22.25} & 21.93 & \textbf{28.85} & \textbf{20.69} & \textbf{21.18} \\
P$^4$IP &  & \textbf{24.65} & 21.87 & 21.33 & 28.93 & 20.16 & 21.04 \\
\noalign{\global\arrayrulewidth=0.01mm}
\arrayrulecolor{gray}\hline
SPDA &  & \underline{\textbf{25.09}} & \textbf{21.54} & 21.23 & \underline{\textbf{29.38}} & \textbf{20.15} & \underline{\textbf{24.92}} \\
NLSPCA &  & 23.86 & 20.76 & 20.52 & 28.22 & 19.47 & 20.86 \\ 
P-MCA ($H = 100$) &  & 24.70 & 21.16 & \textbf{21.44} & 28.64 & 20.11 & 21.41 \\
\noalign{\global\arrayrulewidth=0.3mm}
\arrayrulecolor{black}\hline
\end{tabular}
\end{table} 

More relevant for comparison to our approach are methods that explicitly assume a Poisson noise. As discussed in the previous experiment, a data model that exploits the properties of the Poisson distribution (for the case of Poisson denoising) may be more suitable rather than a Gaussian-based model.  
In this direction, one important study is the work by \citet[NLSPCA;][]{salmon2014poisson} that uses Poisson PCA in which the link between latents and observables (as discussed before) is considered to be a weighted linear sum that sets the Poisson's natural parameter. In another study, \citet[SPDA;][]{giryes2014sparsity} introduce a Poisson sparse coding model capable of learning dictionaries from a Poisson distributed dataset. Their method likewise considers a linear superposition of the latents. 
In contrast, the P-MCA model that is used here uses maximization to directly set the means of the Poisson observables. 
Figure~\ref{fig:Poisson_denoising} illustrates the reconstructed 'House', 'Saturn' and 'Bridge' images using the P-MCA model. The estimated `House' image can be further compared with Fig.~9 depicted by \citet[][]{giryes2014sparsity} which presents the results of BM3D, SPDA and NLSPCA.  
%

In addition to the approaches presented above, deep learning-based techniques have also been applied to Poisson denoising in recent years  \citep[see, e.g.,][]{remez2017deep,kumwilaisak2020image,deguchy2019deep,jin2018poisson,tolooshams2020convolutional}. Importantly, \citet[DenoiseNet;][]{remez2017deep} explored deep CNNs to denoise low-light images and reported state-of-the-art results for Poisson denoising at different peak values \citep[also see][for a further development]{remez2018class}. 
Another DNN-based algorithm for Poisson denoising is DCEA presented by \citet{tolooshams2020convolutional} which, similar to DenoiseNet, considers
a supervised setting for training. As mentioned earlier, the contribution  considers distributions of the natural exponential family which includes the Poisson distribution. \citeauthor{tolooshams2020convolutional} 
investigated their approach in two different settings:
Constrained and unconstrained DCEA, referred to as DCEA-C and DCEA-UN, respectively. 
Also compare
related work 
 by \citet[][]{kumwilaisak2020image} who use multi-directional long-short term memory (LSTM) networks along with the 
CNNs. 
Nevertheless, there are substantial differences between, for instance, DenoiseNet or DCEA-UN, on the one hand, and P-MCA, on the other. P-MCA is trained directly on the noisy image and the approach does not require a-priori information about the peak value of the corrupted data; this contrasts with DenoiseNet or DCEA-UN which are both trained supervised. 
Table~\ref{table:poisson_denoising} presents a quantitative performance comparison of different approaches using the standard measure of peak-signal-to-noise-ratio (PSNR). 
The methods we chose for comparison represent, to our best knowledge, state-of-the-art results for Poisson denoising at different peak values. We here chose four peak values and reported results for P-MCA considering different latent dimensionalities $H$ and using models with $D = 20 \times 20$ observables. In general, we observed that lower values of $H$ produce better results for low peak values while larger $H$ values result in better performance for high peak values.
An explanation could be that the information content of noisy data is lower (if the amount of data points remains the same) such that it becomes difficult to estimate the large number of parameters as required for a large $H$.  Consequently, such a hyper-parameter can be further optimized for our learning algorithm (observe the different $H$ values  for P-MCA in Tab.~\ref{table:poisson_denoising}).

As it can be observed in Tab.~\ref{table:poisson_denoising}, P-MCA produces competitive results and, in some cases, outperforms the other models; but it can also be outperformed, e.g., by I+VST+BM3D, DenoiseNet or DCEA-UN in different settings. At a closer inspection, however, these approaches are using 
considerable fine-tuning specific to images. Furthermore, the major difference compared to DenoiseNet or DCEA-UN is their supervised settings where they use clean (non-noisy) data for training.
P-MCA is, like SPDA and NLSPCA, not tailored to images. All these three approaches also
have in common that they are directly based on the assumption of a Poisson distribution for the observables. 
%
%
Considering Tab.~\ref{table:poisson_denoising}, P-MCA shows, in comparison to SPDA and NLSPCA, improved performance in terms of PSNR values in many of the investigated settings (especially at very low peak values). 
Such improvements for one example of the exponential family (P-MCA) may argue in favor of the general approach represented by the proposed LVMs
and their optimization based on Theorems~\ref{theorem1} and \ref{theorem2}. 
It should be mentioned that we did not rigorously optimize the performance of the P-MCA model (and the other EF-MCA algorithms applied here) using, e.g., extensive parameter tuning. In general, 
task performance could be improved, e.g., by exploiting task-specific patch-averaging procedures or additional methods to avoid local optima.

\begin{figure}[t!]
\centering
\includegraphics[scale=1.0]{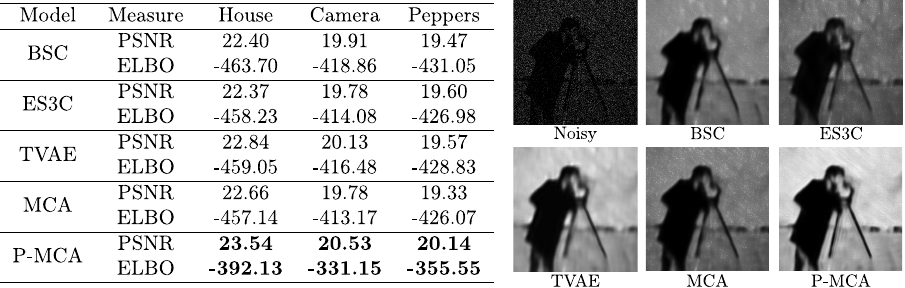}
\caption{Comparison of the PSNR and ELBO values obtained by the considered LVMs for Poisson denoising at peak 1. The results correspond to one run of each algorithm. Best values are highlighted.}
\label{table:PoissonResults}
\end{figure}

Besides the comparisons presented above, we also compared the performance of the P-MCA model with other similar Gaussian-based LVMs trained with EVO.  To this,  we considered a version of MCA studied in \citet[][]{SheikhEtAl2019}, binary sparse coding \citep[BSC;][]{HennigesEtAl2010},  truncated variational autoencoder \citep[TVAE;][]{DrefsEtAl2022a} and evolutionary spike-and-slab sparse coding (ES3C; \citet[][]{DrefsEtAl2022}; also compare \citet[][]{SheikhEtAl2014,goodfellow2012scaling}).  
%
These are all linear or non-linear LVMs that use binary latents (except ES3C that uses a spike-and-slab prior) and 
we have trained all of them using EVO \citep[][]{DrefsEtAl2022} with identical hyper-parameters (the details are discussed in Appendix~\ref{denoising}).  
Figure~\ref{table:PoissonResults} depicts the performance comparison of these approaches for Poisson denoising at peak value 1 (also see Fig.~\ref{table:PoissonResults_appendix}).   
As observed,  P-MCA outperforms all the other models (that are tailored to the Gaussian noise) in both achieving higher ELBO values (although the models are not directly comparable due to the different number of parameters),  and also higher PSNRs for the reconstruction. 
We also repeated the experiment at the peak value of 4.  
In that case,  TVAE obtained the highest PSNRs followed by P-MCA being the second best model. 

%
%

\begin{figure}[t!]
\centering
\includegraphics[scale=0.3]{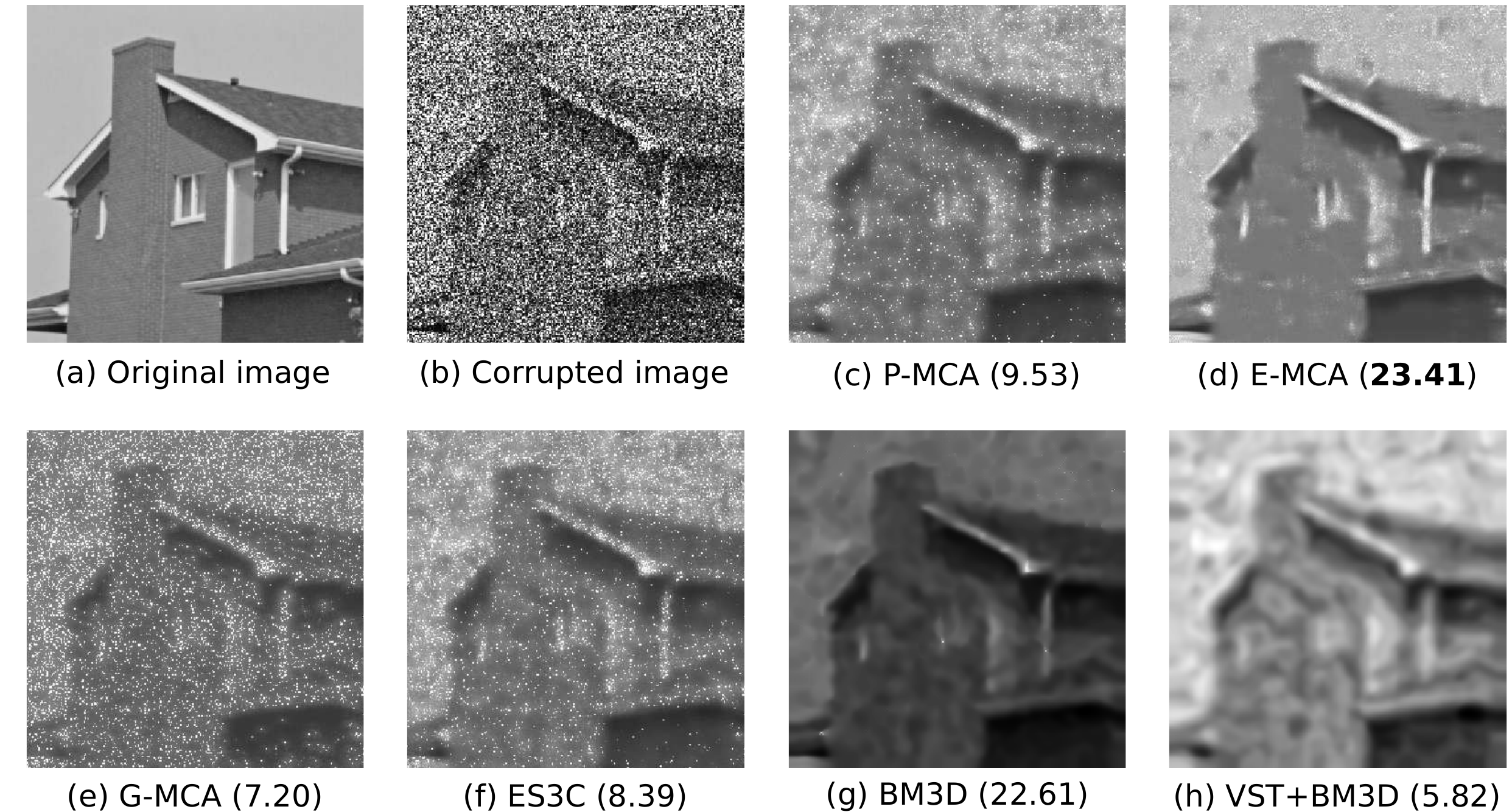}
\caption{Denoising of the House image when corrupted by Exponential noise (we did not rescale the image before adding noise). Depicted are (a) original, (b) corrupted and (c)-(h) reconstructed images and the corresponding PSNRs in terms of dB obtained by different models (see text for details). Here, E-MCA achieves the highest PSNR in comparison to the others which is consistent with its Exponential noise assumption. Note that P-MCA and E-MCA use one-parameter distributions as noise models and only use the information of one dictionary (related to the mean) while the others (G-MCA, ES3C and BM3D) use additional information corresponding to the variance parameter. 
}
\label{fig:denoising_house}
\end{figure}

%

%
%
%
\paragraph{Exponential noise.} 
Exponential noise is treated much less frequently such that a comparison similar to Tab.~\ref{table:poisson_denoising} is not possible. Instead, we here used the example of Exponential noise to highlight the importance of using the correct noise distribution. 
%
%
%
%
%
That is, we exploited E-MCA as an instance of our LVMs
with Exponential noise together with P-MCA and G-MCA (i.e., models assuming Poisson and Gaussian noise, respectively) for denoising the House image corrupted by Exponential noise.  
Additionally, we further 
applied the BM3D, VST+BM3D and ES3C algorithms. 
%
As illustrated in Fig.~\ref{fig:denoising_house}, 
E-MCA achieved the highest PSNR compared to all other approaches, which could be expected because E-MCA assumes the true image noise. Also, 
BM3D performed reasonably well; however, in combination with the Anscombe transformation, the algorithm (VST+BM3D) achieved the lowest PSNR, 
which may be attributed to the Anscombe transformation being specific to Poisson noise. 
Visual inspection may even argue in favor of using the right noise model more than what PSNRs would suggest. In fact, Fig.\ \ref{fig:denoising_house}
shows larger differences between the reconstructed images by E-MCA and BM3D than what may be expected by considering the PSNRs of the two approaches.

%
%
%
%
%

\section{Discussion}
%
Latent variable models (LVMs) have found very wide-spread use in Machine Learning, Statistics and Artificial Intelligence applications.
A central element of basically all LVM approaches is the procedure which optimizes an LVM's parameters given a set of data points.
%
%
For non-Gaussian observables, optimization is typically much more challenging than for Gaussian observables, and especially so when we seek generic optimization methods for a broad class of observable distributions. 

Here we studied an example of LVMs with a very generally defined range of observable distributions.
By using binary latents in combination with maximization in order to link latents to the means of observables, we could address parameter optimization in general.
While it would be challenging to use gradient based approaches for the maximum, the maximum combination of latents has convenient analytical properties. In our study, these properties allowed for the derivation of the main result: A general set of concise parameter update
equations applicable to any (regular and minimal) distribution of the exponential family. For small scale problems, optimization can
be performed based on full posteriors. For larger scale problems, we applied a novel variational acceleration \citep[EVO;][]{DrefsEtAl2022},
which is likewise generically and directly applicable (see Sec.~\ref{sec_EVO}).
%

%
%
The most closely related approach, EF-SC \citep[][]{lee2009exponential}, used MAP-based optimization and distributions of the exponential
family with a weighted linear sum to define the natural parameters of the used observation noise. Such a link from latents to observables has
advantages in terms of MAP-based training (mono-modal posteriors are maintained). A disadvantage is that different observable distributions
essentially define different non-linear superposition models, i.e., the observable means depend differently on the weighted
summation of latents for different observation distributions. Also, how to realize 
a link from latents to two (or more) parameter distributions seems to be an open question for previous approaches. At least (and as previously stated),
neither EF-SC nor PCA-like approaches which use the same link \citep[][]{CollinsEtAl2002,MohamedEtAl2009,salmon2014poisson} report numerical results for 
exponential family distributions with more than one parameter.
Additionally, EF-SC would not be able to learn prior parameters because of the used MAP approaches (additional cross-validation would be required). 
%
%
In contrast, the family of LVMs considered here is generally defined with parameter update equations (Theorems~\ref{theorem1} and \ref{theorem2})
that are directly applicable to any choice of a (regular and minimal) exponential family distribution. The used exact inference (for small $H$) or variational approximation (for large $H)$, furthermore, allows for learning of prior parameters.
%
%
%
For the approach investigated here, we consequently argue that it represents an exceptionally flexible and generally applicable LVM approach for unsupervised learning.
It encompasses, in principle, any choice of an exponential family distribution for its observables. As such, it allows for using many more distributions of the family than have here or elsewhere been considered including Categorical, Dirichlet, normal-gamma etc.
Preliminary example algorithms based on one- and two-parameter distributions of the exponential family suggest that feasible and competitive learning algorithms are obtained when applying the results derived here.

%
 The LVMs considered here share with the generalized linear model \citep[GLM;][]{NelderWedderburn1972} that both approaches are relatively generally defined.
%
For the here considered LVMs (EF-MCA), Theorems~\ref{theorem1} and \ref{theorem2} in conjunction with EVO represent a generic optimization procedure; for GLMs an example of a generic optimization is based on iterative reweighted least square optimization which results in generally applicable (iterative) parameter updates \citep[][]{McCullaghNelder1989}. The importance of the latter and similar approaches is reflected by a large variety of software packages for GLMs \citep[][]{RPackage}.
The class of EF-MCA models and the class of GLM models differ very significantly, however. Most salient is the focus of GLMs on regression compared to the focus on unsupervised learning for EF-MCA. However, also taking this difference aside, GLMs and EF-MCA strongly differ by how latents are linked to observables. GLMs are by definition (and by name) using weighted summation while EF-MCA uses weighted maximization. Furthermore, GLMs usually consider continuous latents while EF-MCA considers binary latents. Finally, also the treatment of multiple parameter distributions is different: GLMs can be defined with observable variances depending on the mean in a specific way, while the focus
remains regression through modeling of the mean dependence. For EF-MCA, multiple matrices are used to link latents and observables (see Fig.~\ref{fig_Network}) in the case when observable distributions with multiple parameters are used (compare Fig.~\ref{fig:image_patches20}). The use of multiple matrices is also a distinguishing feature compared to the other unsupervised learning approaches with more generally defined observable distributions \citep[][]{CollinsEtAl2002,MohamedEtAl2009,lee2009exponential}. 
%

Nevertheless, the central results here (i.e., the generic parameter update equations discussed in Theorems~\ref{theorem1} and \ref{theorem2}) are based on the assumption of binary latents and a maximum superposition, which 
may limit the application of our EF-MCA data models.  
%
%
%
%
While for many types of data, summation is a plausible assumption (e.g. acoustic waveforms for acoustic data), it has been argued that, e.g., for images or spectrogram representations of sounds, a maximum superposition from latents to observables is more closely aligned with the true generating process \citep[see, e.g.,][for a discussion]{Roweis2003,LuckeSahani2008,BornscheinEtAl2013,FrolovEtAl2016,SheikhEtAl2019}. Likewise, discrete latents may represent more realistic modeling choices in some cases \citep[e.g.\ for diseases/symptoms relations,][]{SingliarEtAl2006,JerniteEtAl2013,MousaviEtAl2021} while for other data continuous latents are clearly more suitable.

%
%
%

%
In order to numerically investigate the derived optimization procedure, we used example applications which included
structure finding in image patches, automatic estimation of noise distributions \citep[also compare][]{vergari2019automatic,valera2017automatic}, and denoising of images subject to non-Gaussian noises. In the case of Poisson noise, where a number of alternative approaches are available for comparison, we observed competitive performance. In particular, denoising results improved on models
which used Poisson noise with linear superposition, namely the P-MCA algorithm improved on Non-Local Sparse Principal Component Analysis \citep[NLSPCA;][]{salmon2014poisson} and the Sparse Poisson Denoising Algorithm \citep[SPDA;][]{giryes2014sparsity}. Comparison to NLSPCA and SPDA may be particularly
interesting in the context of this work. Both NLSPCA and SPDA use the standard linear link to observables via the natural parameter of the Poisson distribution, whereas P-MCA sets the mean of the Poisson  using maximization. 
Besides, NLSPCA and SDPA exploit continuous
latents while P-MCA uses binary latents. 
Tab.~\ref{table:poisson_denoising} shows, in general, a better denoising performance of P-MCA compared to both NLSPCA and SDPA at low-SNR values (although SPDA also yields very competitive results in some cases).
%
%
%
%
%
%
Poisson noise is just one example of an exponential family distribution, and denoising is just one application. Still, we note that P-MCA is competitive for the task without using additional, task-specific mechanisms (compare NLSPCA and SDPA) and without extensive fine-tuning of hyper-parameters.
Further applications of our approach included observables following Bernoulli and Exponential distributions (both one-parameter distributions) and, importantly, observables following Gaussian, Gamma and Beta distributions (as two-parameter distributions). The Beta distribution has also been considered in an application of our results to medical data of hearing impairments \citep[][]{MousaviEtAl2021}. 

%
%
In analogy to the literature on standard SC, further application examples would include inpainting, compression, feature learning, or compressed sensing using in principle any regular exponential family as observable distribution.  
Future work may also include the use of more richly structured prior distributions. As the derived update equations
apply in general for any binary latents, the independent Bernoulli prior could also be replaced by more general distributions for binary variables \citep[e.g.][]{DaiEtAl2013}
including those parameterized by deep models. Binary latents for deep generative graphical models are very common such that, 
e.g., deep SBNs \citep[][]{SaulEtAl1996,GanEtAl2015} or deep restricted Boltzmann Machines \citep[][]{Larochelle2008} could be used as well as deep models with
deterministic non-linearities \citep[e.g.][]{JangEtAl2017}. Future research may also investigate how the current results can be extended to other forms of non-linear superpositions. The non-linearity \eqref{eq:mmca_combination_rule} or an occlusion non-linearity as discussed by \citet[][]{DaiLucke2014} do maintain crucial properties required for Theorems~\ref{theorem1} and \ref{theorem2}, for instance. And non-linear superposition models could go further in the direction of deep models, e.g., by using repeated maximization with several layers of stochastic binary latents. Finally, also hybrid approaches combining results of LVMs with binary and continuous latents (cf.\,Sec.\,\ref{Relation_to_Previous_Work}) could represent a promising future research direction.

%
%

%
%


\acks{This work was funded by the Deutsche Forschungsgemeinschaft (DFG, German Research Foundation) within the excellence initiative (cluster of excellence H4a 2.0, EXC 2177/1 - project 390895286) and by the frame of 
the priority program SPP 2298 "Theoretical Foundations of Deep Learning" 
- project 464104047 (HM and JL). Furthermore, the contributions by FH 
were funded by the German Ministry of Research and Education (BMBF) in 
project 05M2020 (SPAplus, TP3), and the contributions by JD were funded 
within the DFG project 352015383 (HAPPAA, CRC 1330, B2). Large-scale 
numerical experiments were supported by the HPC Cluster of Oldenburg 
University (INST 184/225-1 FUGG) and by the HLRN network of HPC clusters 
(project nim00006).
}




\appendix

\section{Proof of the Main Theorems}
\label{app:theorem}
We here provide the full proofs of Theorems~\ref{theorem1} and \ref{theorem2} as well as further details.
%
First consider the function $\mathcal{A}_{dh}(\vec{s}, \Theta)$ defined in \eqref{Adh}. We show that the function has a useful property which is exploited in our proofs. 
For special cases of EF-MCA data models, this property has been observed before \citep[][]{LuckeSahani2008,LuckeEggert2010}
but the following Lemma represents the required generalization for exponential family observables.
%
%
\begin{lem}\label{lemma1}
Consider $\bar{W}_d(\vec{s}, \Theta)$ and $\bar{V}_d(\vec{s}, \Theta)$ that are defined in \eqref{maxh}. Then, for any well-behaved function $g$ and any arbitrary $\vec{s} \in \{0, 1\}^H$, we have
\begin{align}\label{g_W_approximation}
\mathcal{A}_{dh}(\vec{s}, \Theta)\, g \big( \bar{W}_d(\vec{s}, \Theta), \bar{V}_d(\vec{s}, \Theta) \big) = \mathcal{A}_{dh}(\vec{s}, \Theta)\, g(W_{dh}, V_{dh}).
\end{align}
\end{lem}
\begin{proof}
Given a vector $\vec{s}$, for each pair $(d_o, h_o)$ either of the following applies:
\begin{align}
h_o = h(d_o, \vec{s}, \Theta), \quad \mathrm{or} \quad h_o \neq h(d_o, \vec{s}, \Theta), \nonumber
\end{align}
where $h(d_o, \vec{s}, \Theta) = \mathrm{argmax}_{h} \{M_{d_oh}(\Theta)\, s_h\}$. 
First, let $h_o = h(d_o, \vec{s}, \Theta)$. It follows from (\ref{maxh}) that
\begin{align}
\mathcal{A}_{d_oh_o}(\vec{s}, \Theta) g \big( \bar{W}_{d_o}(\vec{s}, \Theta), \bar{V}_{d_o}(\vec{s}, \Theta) \big) &= \mathcal{A}_{d_oh_o}(\vec{s}, \Theta) g(W_{d_oh(d_o, \vec{s}, \Theta)}, V_{d_oh(d_o, \vec{s}, \Theta)}) \cr 
&= \mathcal{A}_{d_oh_o}(\vec{s}, \Theta) g(W_{d_oh_o}, V_{d_oh_o}). \nonumber
\end{align}
On the other hand, it follows from $h_o \neq h(d_o, \vec{s}, \Theta)$ and (\ref{Adh}) that $\mathcal{A}_{d_oh_o}(\vec{s}, \Theta) = 0$ which trivially satisfies the claim in \eqref{g_W_approximation}.
\end{proof}


\noindent{}The lemma illustrates how multiplication by $\mathcal{A}_{dh}(\vec{s}, \Theta)$ simplifies the left-hand-side of the equation \eqref{g_W_approximation} to the right-hand-side expression in which the second factor is independent of $\vec{s}$. The function 
%
%
$g$ can be any arbitrary function with arguments $\bar{W}_d(\vec{s}, \Theta)$ and $\bar{V}_d(\vec{s}, \Theta)$.  
Moreover, it should be noted that we here assumed the case of $L = 2$ and presented the lemma using $\bar{W}_d(\vec{s}, \Theta)$ and $\bar{V}_d(\vec{s}, \Theta)$. 
For the case of arbitrary $L$, the straightforward generalization of the lemma applies by considering the definition \eqref{maxhgen} and a function $g$ with argument(s) $\bar{W}_d^{(1)}(\vec{s}, \Theta), \ldots, \bar{W}_d^{(L)}(\vec{s}, \Theta)$. The general case will be used for the proof of Theorem~\ref{theorem2}.
%
%
%

\subsection{Parameter Update Equations -- Two-Parameters Case}
\label{appendixA0}

We first prove Theorem~\ref{theorem1} 
and show that the derivatives of $\mathcal{F}(q, \Theta)$ w.r.t.\ $W_{dh}$ and $V_{dh}$ yields the corresponding equations for $W$ and $V$ in \eqref{Update_W} and \eqref{Update_V}.

\begin{proof} 
\hspace*{-0.7ex}\textbf{of Theorem~\ref{theorem1}}
Consider a single dictionary element $W_{dh}$ and, for the sake of brevity, let us abbreviate $ \bar{W}_d(\vec{s}, \Theta)$ and $\bar{V}_d(\vec{s}, \Theta)$ by $\bar{W}_d$ and $\bar{V}_d $, respectively. Then using the chain rule and Eqns. \eqref{eta} and \eqref{derivatives_WbarVbar1}-\eqref{derivatives_WbarVbar2}, we obtain:

\begin{align}\label{theorem_Pderivative1}
\frac{\partial}{\partial W_{dh}} &\log(p(y_d; \vec{\tilde{\eta}}_d(\vec{s}, \Theta))) = \frac{\partial}{\partial W_{dh}} \log(p(y_d; \vec{\Phi}(\bar{W}_d, \bar{V}_d)))\\
=& \sum_{l=1}^2 \Big( \frac{\partial}{\partial W_{dh}} \Phi_l(\bar{W}_d, \bar{V}_d) \Big) \Big( \frac{\partial}{\partial \eta_l} \log(p(y_d; \vec{\eta}\,)) \big|_{\vec{\eta} = \vec{\Phi}(\bar{W}_d,\bar{V}_d)} \Big) \cr
=& \sum_{l=1}^2 \Big( \frac{\partial}{\partial W_{dh}} \bar{W}_d \Big) \Big( \frac{\partial}{\partial w} \Phi_l(w, \bar{V}_d) \big|_{w = \bar{W}_d} \Big) \Big( \frac{\partial}{\partial \eta_l} \log(p(y_d; \vec{\eta}\,)) \big|_{\vec{\eta} = \vec{\Phi}(\bar{W}_d,\bar{V}_d)} \Big) \cr
&+ \sum_{l=1}^2 \Big( \underbrace{\frac{\partial}{\partial W_{dh}} \bar{V}_d}_{=\ 0} \Big) \Big( \frac{\partial}{\partial v} \Phi_l(\bar{W}_d, v) \big|_{v = \bar{V}_d} \Big) \Big( \frac{\partial}{\partial \eta_l} \log(p(y_d; \vec{\eta}\,)) \big|_{\vec{\eta} = \vec{\Phi}(\bar{W}_d,\bar{V}_d)} \Big) \cr
=& \sum_{l=1}^2 \Big( \mathcal{A}_{dh}(\vec{s}, \Theta) \Big) \Big( \frac{\partial}{\partial w} \Phi_l(w, \bar{V}_d) \big|_{w = \bar{W}_d} \Big) \Big( T_l(y_d) - \frac{\partial A(\vec{\eta}\,)}{\partial \eta_l} \Big) \Big |_{\vec{\eta} = \vec{\Phi}(\bar{W}_d, \bar{V}_d)}.\nonumber
\end{align}
%
%
Moreover, from \eqref{A_eta_expected_value} we know that for any regular distribution (i.e. with finite $A(\vec{\eta}\,)$) of the exponential family, $A(\vec{\eta}\,)$ satisfies:
%
\begin{align}
\frac{\partial A(\vec{\eta}\,)}{\partial \eta_1} =\, \mathbb{E}_{p(y;\, \vec{\eta}\,)}  [ T_1(y) ] \quad \mbox{and} \quad \frac{\partial A(\vec{\eta}\,)}{\partial \eta_2} =\, \mathbb{E}_{p(y;\, \vec{\eta}\,)}  [ T_2(y) ].
\end{align}
Thus, we can further simplify:
\begin{align}\label{theorem_Pderivative2}
\frac{\partial}{\partial W_{dh}} &\log(p(y_d; \vec{\tilde{\eta}}_d(\vec{s}, \Theta))) = \sum_{l=1}^2 \Big( \mathcal{A}_{dh}(\vec{s}, \Theta) \Big) \Big( \frac{\partial}{\partial w} \Phi_l(w, \bar{V}_d) \big|_{w = \bar{W}_d} \Big)\\
&\times \Big( T_l(y_d)\, - \mathbb{E}_{p(y;\, \vec{\eta}\,)}  [ T_l(y) ] \Big) \Big |_{\vec{\eta} = \vec{\Phi}(\bar{W}_d, \bar{V}_d)}\cr
=& \sum_{l=1}^2 \Big( \mathcal{A}_{dh}(\vec{s}, \Theta) \Big) \Big( \frac{\partial}{\partial w} \Phi_l(w, \bar{V}_d) \big|_{w = \bar{W}_d} \Big) \Big( T_l(y_d) - \mathbb{E}_{p(y;\, \vec{\Phi}(\bar{W}_d, \bar{V}_d) )}  [ T_l(y) ] \Big).\nonumber
\end{align}
Now, using Lemma~\ref{lemma1} we obtain:
\begin{align}\label{theorem_Pderivative3}
\frac{\partial}{\partial W_{dh}} \log(p(y_d; \vec{\tilde{\eta}}_d(\vec{s}, \Theta))) =& \sum_{l=1}^2 \Big( \mathcal{A}_{dh}(\vec{s}, \Theta) \Big) \Big( \frac{\partial}{\partial w} \Phi_l(w, V_{dh}) \big|_{w = W_{dh}} \Big) \cr
&\times \Big( T_l(y_d) - \mathbb{E}_{p(y;\, \vec{\Phi}(W_{dh}, V_{dh}))}  [ T_l(y) ] \Big).
\end{align}
Note that the above equation depends on parameter $\vec{s}$ of the hidden states only through the function $\mathcal{A}_{dh}(\vec{s}, \Theta)$. This is an important property of Lemma~\ref{lemma1} that alleviates the complexity of the aforementioned equation and enables us to extract a set of concise update equations for dictionaries $W$ and $V$. To see this, we derive (using Eqn. \eqref{EqnFreeEnergy}):
\begin{align}\label{theorem_Pderivative4}
\frac{\partial}{\partial W_{dh}}\mathcal{F}(q, \Theta) 
=& \sum_{n} \sum_{\vec{s}} q^{(n)}(\vec{s}\,) \Big(\sum_{d'} \frac{\partial}{\partial W_{dh}} \log \big( p(y_{d'}^{(n)}; \vec{\tilde{\eta}}_{d'}(\vec{s}, \Theta)) \big) \Big)\cr
=& \sum_{n} \sum_{\vec{s}} q^{(n)}(\vec{s}\,) \Big( \sum_{d'} \sum_{l=1}^2 \Big( \mathcal{A}_{d'h}(\vec{s}, \Theta) \Big) \Big( \frac{\partial}{\partial w} \Phi_l(w, V_{d'h}) \big|_{w = W_{d'h}} \Big) \cr
&\times \Big( T_l(y_{d'}^{(n)}) - \mathbb{E}_{p(y;\, \vec{\Phi}(W_{d'h}, V_{d'h}))}   [ T_l(y) ] \Big) \delta_{dd'} \Big) \cr
=& \sum_{n} \sum_{\vec{s}} q^{(n)}(\vec{s}\,) \sum_{l=1}^2 \Big( \mathcal{A}_{dh}(\vec{s}, \Theta) \Big) \Big( \frac{\partial}{\partial w} \Phi_l(w, V_{dh}) \big|_{w = W_{dh}} \Big)\cr
&\times \Big( T_l(y_d^{(n)}) - \mathbb{E}_{p(y;\, \vec{\Phi}(W_{dh}, V_{dh}))}  [ T_l(y) ] \Big)\cr
=& \sum_{l=1}^2 \Big( \frac{\partial}{\partial w} \Phi_l(w, V_{dh}) \big|_{w = W_{dh}} \Big) \sum_{n} \mathbb{E}_{q^{(n)}}  [ \mathcal{A}_{dh}(\vec{s}, \Theta) ] \cr
&\times  \Big( T_l(y_d^{(n)}) - \mathbb{E}_{p(y;\, \vec{\Phi}(W_{dh}, V_{dh}))}  [ T_l(y) ] \Big),
\end{align}
where $\delta_{dd'}$ denotes the Kronecker delta, and using our mean value parametrization defined by $\vec{w} :=\, \mathbb{E}_{p(y;\, \vec{\Phi}(\vec{w}))}  [\vec{T}(y) ]$, we have:
\begin{align}
\frac{\partial}{\partial W_{dh}}\mathcal{F}(q, \Theta) =& \ \Big( \frac{\partial}{\partial w} \Phi_1(w, V_{dh}) \big|_{w = W_{dh}} \Big) \sum_{n} \mathbb{E}_{q^{(n)}}  [ \mathcal{A}_{dh}(\vec{s}, \Theta) ] \Big( T_1(y_d^{(n)}) - W_{dh} \Big)\\
&+ \Big( \frac{\partial}{\partial w} \Phi_2(w, V_{dh}) \big|_{w = W_{dh}} \Big) \sum_{n} \mathbb{E}_{q^{(n)}}  [ \mathcal{A}_{dh}(\vec{s}, \Theta) ] \Big( T_2(y_d^{(n)}) - V_{dh} \Big).\nonumber
\end{align}

\noindent{}Now, independently of the functions $\frac{\partial}{\partial w} \Phi_l(w, V_{dh}) \big|_{w = W_{dh}}$ for $l = 1, 2$, the derivative of the ELBO w.r.t. $W_{dh}$ is zero, i.e. $\frac{\partial \mathcal{F}}{\partial W_{dh}} = 0$, if it is the case that:  
\begin{align}
\sum_n \mathbb{E}_{q^{(n)}}  [ \mathcal{A}_{dh}(\vec{s}, \Theta) ]\,  \big(T_1(y_d^{(n)}) - W_{dh}\big) &= 0\ \ \ \mbox{and}\\
\sum_n \mathbb{E}_{q^{(n)}}  [ \mathcal{A}_{dh}(\vec{s}, \Theta) ]\ \big(T_2(y_d^{(n)}) - V_{dh}\big)\ &= 0\,.
\end{align}
%
Rearranging terms yields \eqref{Update_W} and \eqref{Update_V} and completes the proof. The proof proceeds along the same lines
for $\frac{\partial \mathcal{F}}{\partial V_{dh}}$ which results in the same set of equations. 
\end{proof}

\subsection{Second Derivatives}
\label{Sec_second_derivative}
%
Theorem~\ref{theorem1} provides equations \eqref{Update_W} and \eqref{Update_V} for the matrices $W$ and $V$. 
%
Fulfilling the equations implies that the derivatives of $\mathcal{F}(q, \Theta)$ w.r.t.\ $W$ and $V$ vanish. If we assume convergence of the corresponding fixed-point equations \eqref{update_W_real} and \eqref{update_V_real} that we use as M-steps, then $W$ and $V$ approach the values that solve \eqref{Update_W} and \eqref{Update_V} arbitrary closely. And in the numeric evaluations such convergence is observed. The M-step includes, for completeness, also the update for $\vec{\pi}$ in \eqref{Update_pi}.
If we further suppose for the M-step that $W$ and $V$ have fully converged (and that $\vec{\pi}$ has been updated) then all derivatives of $\mathcal{F}(q, \Theta)$ vanish, i.e., $\mathcal{F}(q, \Theta)$ has a fixed-point w.r.t.\ the model parameters $\Theta=(\vec{\pi},W,V)$. 
But vanishing derivatives can also correspond to (local) minima or saddle points and not necessarily maxima.
We can gather more information about the fixed-points by considering, e.g.,  the updates \eqref{update_W_real} for one element $W_{dh}$ of the matrix $W$. If we now divert from the fixed-point by changing $W_{dh}$, we can ask if the objective $\mathcal{F}(q, \Theta)$ then increases or decreases. We do so by evaluating the second derivative at an (arbitrary) fixed-point of $\mathcal{F}(q, \Theta)$.  We start with the first derivative:
\begin{align}
\frac{\partial}{\partial W_{dh}}&\mathcal{F}(q, \Theta) = \sum_{n} \sum_{\vec{s}} q^{(n)}(\vec{s}\,) \frac{\partial}{\partial W_{dh}} \Big( \sum_{d'} \log \big( p(y_{d'}^{(n)}; \vec{\tilde{\eta}}_{d'}(\vec{s}, \Theta)) \big) \Big)\\
=& \sum_{n} \sum_{\vec{s}} q^{(n)}(\vec{s}\,) \sum_{l=1}^2 \Big( \mathcal{A}_{dh}(\vec{s}, \Theta) \Big) \Big( \frac{\partial}{\partial w} \Phi_l(w, V_{dh}) \big|_{w = W_{dh}} \Big) 
 \Big( T_l(y_d^{(n)}) - \frac{\partial A(\vec{\eta}\,)}{\partial \eta_l} \Big) \Big |_{\vec{\eta} = \vec{\Phi}(W_{dh}, V_{dh})}\cr
 =& \  \underbrace{ \Big( \frac{\partial}{\partial w} \Phi_1(w, V_{dh}) \big|_{w = W_{dh}} \Big) }_{:=\ G_1} 
  \underbrace{ \sum_{n} \sum_{\vec{s}} q^{(n)}(\vec{s}\,)  \Big( \mathcal{A}_{dh}(\vec{s}, \Theta) \Big)  
 \Big( T_1(y_d^{(n)}) - \frac{\partial A(\vec{\eta}\,)}{\partial \eta_1} \Big) \Big |_{\vec{\eta} = \vec{\Phi}(W_{dh},  V_{dh})} }_{:=\ X_1} \cr
 &+  \underbrace{ \Big( \frac{\partial}{\partial w} \Phi_2(w, V_{dh}) \big|_{w = W_{dh}} \Big) }_{:=\ G_2} 
 \underbrace{  \sum_{n}  \sum_{\vec{s}} q^{(n)}(\vec{s}\,)  \Big( \mathcal{A}_{dh}(\vec{s}, \Theta) \Big) 
 \Big( T_2(y_d^{(n)}) - \frac{\partial A(\vec{\eta}\,)}{\partial \eta_2} \Big) \Big |_{\vec{\eta} = \vec{\Phi}(W_{dh}, V_{dh})} }_{:=\ X_2}.\nonumber
\end{align}
%
%
%
%
Then,  considering the definitions of the terms $X_1$, $X_2$, $G_1$ and $G_2$ above, 
the second derivative yields (at the fixed-point):
\begin{align}\label{SecondDerivative_eq1}
\frac{\partial^2}{(\partial W_{dh})^2}\mathcal{F}(q, \Theta) &= 
\frac{\partial G_1}{\partial W_{dh}} X_1 + G_1 \frac{\partial X_1}{\partial W_{dh}} + \frac{\partial G_2}{\partial W_{dh}} X_2 + G_2 \frac{\partial X_2}{\partial W_{dh}} \cr
&= G_1 \frac{\partial X_1}{\partial W_{dh}} + G_2 \frac{\partial X_2}{\partial W_{dh}}\ ,
\end{align}
since $X_1 = X_2 = 0$ as $W$ and $V$ satisfy \eqref{Update_W} and \eqref{Update_V}.  
Now, using the chain rule we get:
\begin{align}\label{SecondDerivative_eq2}
\frac{\partial X_1}{\partial W_{dh}} &=  \Big( \frac{\partial X_1}{\partial \eta_1} \frac{\partial \eta_1}{\partial W_{dh}} \Big) 
\Big |_{\vec{\eta} = \vec{\Phi} (W_{dh}, V_{dh})}
+ \Big( \frac{\partial X_1}{\partial \eta_2} \frac{\partial \eta_2}{\partial W_{dh}} \Big) 
\Big |_{\vec{\eta} = \vec{\Phi} (W_{dh}, V_{dh})} \cr
&= - \Big( \sum_{n,\vec{s}} q^{(n)}(\vec{s}\,) \mathcal{A}_{dh}(\vec{s}, \Theta)  \frac{\partial^2 A(\vec{\eta}\,)}{\partial \eta_1 \partial \eta_1}  \big(  \frac{\partial}{\partial W_{dh}} \Phi_1(W_{dh}, V_{dh})  \big)  \Big)  
\Big |_{\vec{\eta} = \vec{\Phi} (W_{dh}, V_{dh})} \cr
& \ - \Big( \sum_{n,\vec{s}} q^{(n)}(\vec{s}\,) \mathcal{A}_{dh}(\vec{s}, \Theta) \frac{\partial^2 A(\vec{\eta}\,)}{\partial \eta_2 \partial \eta_1} \big(  \frac{\partial}{\partial W_{dh}} \Phi_2(W_{dh}, V_{dh}) \big)  \Big)
 \Big |_{\vec{\eta} = \vec{\Phi} (W_{dh}, V_{dh})} \cr
&= - \sum_{n,\vec{s}} q^{(n)}(\vec{s}\,) \mathcal{A}_{dh}(\vec{s}, \Theta)  \big( \frac{\partial^2 A(\vec{\eta}\,)}{\partial \eta_1 \partial \eta_1} G_1 + \frac{\partial^2 A(\vec{\eta}\,)}{\partial \eta_2 \partial \eta_1} G_2 \big) 
 \Big |_{\vec{\eta} = \vec{\Phi} (W_{dh}, V_{dh})}. 
\end{align}
Note that we here treat $\mathcal{A}_{dh}(\vec{s}, \Theta)$ as a constant since the second derivative of $\bar{W}_d$ w.r.t.  $W_{dh}$ is equal to zero (recall Eqns.~\eqref{derivatives_WbarVbar1}-\eqref{derivatives_WbarVbar2}). 
Similarly we can show:

\begin{align}\label{SecondDerivative_eq3}
\frac{\partial X_2}{\partial W_{dh}} &= 
- \sum_{n,\vec{s}} q^{(n)}(\vec{s}\,) \mathcal{A}_{dh}(\vec{s}, \Theta)  \big( \frac{\partial^2 A(\vec{\eta}\,)}{\partial \eta_1 \partial \eta_2} G_1 + \frac{\partial^2 A(\vec{\eta}\,)}{\partial \eta_2 \partial \eta_2} G_2 \big) 
 \Big |_{\vec{\eta} = \vec{\Phi} (W_{dh}, V_{dh})}.
\end{align}
Now, inserting \eqref{SecondDerivative_eq2} and \eqref{SecondDerivative_eq3} into \eqref{SecondDerivative_eq1}, we get:
\begin{align*}
\frac{\partial^2}{(\partial W_{dh})^2}&\mathcal{F}(q, \Theta) = 
G_1 \frac{\partial X_1}{\partial W_{dh}} + G_2 \frac{\partial X_2}{\partial W_{dh}} \cr
&=
G_1 \Big(   - \sum_{n,\vec{s}} q^{(n)}(\vec{s}\,) \mathcal{A}_{dh}(\vec{s}, \Theta)  \big( \frac{\partial^2 A(\vec{\eta}\,)}{\partial \eta_1 \partial \eta_1} G_1 + \frac{\partial^2 A(\vec{\eta}\,)}{\partial \eta_2 \partial \eta_1} G_2 \big) 
 \Big |_{\vec{\eta} = \vec{\Phi} (W_{dh}, V_{dh})}
\Big) \\
& \ + G_2 \Big( - \sum_{n,\vec{s}} q^{(n)}(\vec{s}\,) \mathcal{A}_{dh}(\vec{s}, \Theta)  \big( \frac{\partial^2 A(\vec{\eta}\,)}{\partial \eta_1 \partial \eta_2} G_1 + \frac{\partial^2 A(\vec{\eta}\,)}{\partial \eta_2 \partial \eta_2} G_2 \big) 
 \Big |_{\vec{\eta} = \vec{\Phi} (W_{dh}, V_{dh})}
 \Big)  \\
&= 
- \sum_{n,\vec{s}} q^{(n)}(\vec{s}\,) \mathcal{A}_{dh}(\vec{s}, \Theta) 
\mathcal{G}^T (W_{dh}, V_{dh}) \big( \mathbf{H}_{A} \big)   
\mathcal{G} (W_{dh}, V_{dh}),
\end{align*}
where $\mathbf{H}_{A}$ denotes the Hessian matrix of $A$ (based on \eqref{A_eta_expected_value}, it also corresponds to the covariance matrix between $T_1(y)$ and $T_2(y)$),
and $\mathcal{G} := \left(\begin{array}{c} \hspace{-0.5ex}G_1\hspace{-1.5ex}\phantom{i}\\ \hspace{-0.5ex}G_2\hspace{-1.5ex}\phantom{i}\end{array}\right)$.
%
The term $\mathcal{G}^T (W_{dh}, V_{dh}) \big( \mathbf{H}_{A} \big)   
\mathcal{G} (W_{dh}, V_{dh})$ constructs a quadratic form, and the Hessian of the log-partition $A(\vec{\eta}\,)$ is known to be positive semi-definite; and also positive definite for minimal representations \citep[e.g.,][]{wainwright2008graphical,BickelDoksum2015}. As $q^{(n)}(\vec{s}\,)$ and $\mathcal{A}_{dh}(\vec{s}, \Theta)$ are non-negative, 
we therefore obtain:
%
\begin{align}\label{SecondDeriMax}
\frac{\partial^2}{(\partial W_{dh})^2}\mathcal{F}(q, \Theta) \leq 0 \ .
\end{align}
%
The second derivative can, furthermore, be considered strictly smaller than zero under mild conditions. For instance, we
can demand the condition that there exist indices $n'$ and $\vec{s} \,'$ such that $q^{(n')}(\vec{s} \,') \mathcal{A}_{dh}(\vec{s} \,', \Theta) \neq 0$,
then it follows that $\mathbb{E}_{q^{(n)}}  [ \mathcal{A}_{dh}(\vec{s}, \Theta) ] \neq 0$. The condition can only be violated artificially with
potentially few data points with specific properties, and the condition is satisfied for the types of data that are usually of interest (including
all data considered in this work). 
If the condition is combined with properties of $\vec{\Phi}$ ensuring
non-zero $\mathcal{G} (W_{dh}, V_{dh})$ and with minimality (ensuring positive definite $\mathbf{H}_{A}$), then
strictly smaller second derivatives are obtained. Consequently, the diversion of $W_{dh}$ from the fixed-point result in a
decrease of the objective, and (assuming convergence) updates \eqref{update_W_real} will increase $\mathcal{F}(q, \Theta)$ until a (local)
maximum is reached.

%
Nonetheless, it is challenging to infer a more general behavior.
First of all, the fixed-point updates \eqref{update_W_real} could in theory show intricate behavior such as convergence to an attractor
for difference equations. Moreover, in practice all model parameter $\Theta$ are updated which would require the analysis
of all second derivatives including all mixed second derivatives for all parameters (amounting to an exhaustive classification of all
fixed-points).  Further investigations would also require the combined optimization using E- as well as M-steps. Still,
Eqn. (\ref{SecondDeriMax}) provides some information about the individual updates \eqref{update_W_real}, and it applies for an arbitrary
$W_{dh}$, of course, and for an arbitrary fixed-point. We also do note that the very same result is obtained for individual updates
of $V_{dh}$ using \eqref{update_V_real}, i.e., the derivations will be analogous, with the same result of second derivatives being smaller than zero.

\subsection{Parameter Update Equations -- General Case}
\label{appendixA1}
We derived Theorem~\ref{theorem1} for the case of $L=2$, i.e., for distributions of the exponential family with sufficient statistics of length two.
This choice was for notational convenience only, and the proof of Theorem~\ref{theorem1} suggests a straightforward generalization for arbitrary $L$.  
We already discussed the generalization of our formulations in Sec.~\ref{generalEF_MCA} and stated Theorem~\ref{theorem2} for the general case. 
Here we provide the proof of Theorem~\ref{theorem2}.
%
%
First recall the details presented in Sec.~\ref{generalEF_MCA}. Then, similar to the $L=2$ case, observe that 
the derivatives of $\mathcal{F}(q, \Theta)$ contain derivatives of $\bar{W}_d^{(l)}(\vec{s}, \Theta)$ for $l = 1, \ldots, L$ w.r.t.\ the dictionary elements $W_{dh}^{(l)}$ for each $d$ and $h$.
For these derivatives the generalization of 
 Eqns.~\eqref{derivatives_WbarVbar1}-\eqref{derivatives_WbarVbar2} becomes:   
\begin{align}\label{pre_theorem4_1}
\frac{\partial}{\partial W_{dh}^{(l)}} \bar{W}_d^{(l')}(\vec{s}, \Theta) = 
\left\{
\begin{aligned}
&\mathcal{A}_{dh}(\vec{s}, \Theta) \quad \mbox{if} \ \ l = l'  \\
&0 \qquad \qquad \ \ \mbox{otherwise}
\end{aligned}
\right\}
&= \mathcal{A}_{dh}(\vec{s}, \Theta) \delta_{ll'}\,,
\end{align}
where $\mathcal{A}_{dh}(\vec{s}, \Theta)$ is given by \eqref{Adh}. 
Lemma~\ref{lemma1} can be also generalized by again considering the cases $h = h(d, \vec{s}, \Theta)$ and $h \neq h(d, \vec{s}, \Theta)$ separately, which then yields: 
\begin{align}\label{Lemma1_GeneralCase}
\mathcal{A}_{dh}(\vec{s}, \Theta)\, g \big( \bar{W}^{(1)}_d(\vec{s}, \Theta),  \ldots, \bar{W}^{(L)}_d(\vec{s}, \Theta) \big) = \mathcal{A}_{dh}(\vec{s}, \Theta)\, g(W^{(1)}_{dh}, \ldots, W^{(L)}_{dh}).
%
\end{align}
\begin{proof}
\hspace*{-0.7ex}\textbf{of Theorem~\ref{theorem2}}
Consider a single parameter $W^{(l)}_{dh}$ for an arbitrary $l = 1, \ldots, L$ and abbreviate $\bar{W}^{(l)}_d(\vec{s}, \Theta)$ by $\bar{W}^{(l)}_d$. Then, using the chain rule and Eqn. \eqref{etagen}, we get: 
\begin{align}\label{theorem_Pderivative1}
\frac{\partial}{\partial W^{(l)}_{dh}}& \log\big( p(y_d; \vec{\tilde{\eta}}_d(\vec{s}, \Theta)) \big) = \frac{\partial}{\partial W^{(l)}_{dh}} \log\big( p(y_d; \vec{\Phi}(\bar{W}^{(1)}_d, \ldots, \bar{W}^{(L)}_d)) \big)\\
=& \sum_{l'=1}^L \Big( \frac{\partial}{\partial W^{(l)}_{dh}} \Phi_{l'}(\bar{W}^{(1)}_d, \ldots, \bar{W}^{(L)}_d) \Big) \Big( \frac{\partial}{\partial \eta_{l'}} \log(p(y_d; \vec{\eta}\,)) \big|_{\vec{\eta} = \vec{\Phi}(\bar{W}^{(1)}_d, \ldots, \bar{W}^{(L)}_d)} \Big) \cr
=& \sum_{l'=1}^L \Big\{ \sum_{l''=1}^L \Big( \frac{\partial}{\partial W^{(l)}_{dh}} \bar{W}^{(l'')}_d \Big) \Big( \frac{\partial}{\partial w} \Phi_{l'}(\bar{W}^{(1)}_d, \ldots, w, \ldots, \bar{W}^{(L)}_d) \big|_{w = \bar{W}^{(l'')}_d} \Big) \Big\} \cr
& \times \Big( \frac{\partial}{\partial \eta_{l'}} \log(p(y_d; \vec{\eta}\,)) \big|_{\vec{\eta} = \vec{\Phi}(\bar{W}^{(1)}_d, \ldots, \bar{W}^{(L)}_d)} \Big). \nonumber
\end{align}  
Now, using \eqref{pre_theorem4_1} and also \eqref{A_eta_expected_value} that states for any regular distribution of the exponential family $\frac{\partial A(\vec{\eta}\,)}{\partial \eta_{l'}} =\, \mathbb{E}_{p(y;\, \vec{\eta}\,)}  [ T_{l'}(y) ]$, we obtain:
\begin{align*}
\frac{\partial}{\partial W^{(l)}_{dh}} \log(p(y_d; \vec{\tilde{\eta}}_d(\vec{s}, \Theta))) &= \sum_{l'=1}^L \Big( \mathcal{A}_{dh}(\vec{s}, \Theta) \Big) \Big( \frac{\partial}{\partial w} \Phi_{l'}(\bar{W}^{(1)}_d, \ldots, w, \ldots, \bar{W}^{(L)}_d) \big|_{w = \bar{W}^{(l)}_d} \Big) \\
& \quad \times \Big( T_{l'}(y_d) - \frac{\partial A(\vec{\eta}\,)}{\partial \eta_{l'}} \Big) \Big |_{\vec{\eta} = \vec{\Phi}(\bar{W}^{(1)}_d, \ldots, \bar{W}^{(L)}_d)} \\
&= \sum_{l'=1}^L \Big( \mathcal{A}_{dh}(\vec{s}, \Theta) \Big) \Big( \frac{\partial}{\partial w} \Phi_{l'}(\bar{W}^{(1)}_d, \ldots, w, \ldots, \bar{W}^{(L)}_d) \big|_{w = \bar{W}^{(l)}_d} \Big)\cr
& \quad \times \Big( T_{l'}(y_d)\, - \mathbb{E}_{p(y;\, \vec{\eta}\,)}  [ T_{l'}(y) ] \Big) \Big |_{\vec{\eta} = \vec{\Phi}(\bar{W}^{(1)}_d, \ldots, \bar{W}^{(L)}_d)}\cr
&= \sum_{l'=1}^L \Big( \mathcal{A}_{dh}(\vec{s}, \Theta) \Big) \Big( \frac{\partial}{\partial w} \Phi_{l'}(\bar{W}^{(1)}_d, \ldots, w, \ldots, \bar{W}^{(L)}_d) \big|_{w = \bar{W}^{(l)}_d} \Big)\cr
&\quad \times \Big( T_{l'}(y_d)\, - \mathbb{E}_{p(y;\, \vec{\Phi}(\bar{W}^{(1)}_d, \ldots, \bar{W}^{(L)}_d))}  [ T_{l'}(y) ] \Big).
\end{align*}
Further, using \eqref{Lemma1_GeneralCase} simplifies the above equation as follows:
\begin{align*}
\frac{\partial}{\partial W^{(l)}_{dh}} \log(p(y_d; \vec{\tilde{\eta}}_d(\vec{s}, \Theta))) &= \sum_{l'=1}^L \Big( \mathcal{A}_{dh}(\vec{s}, \Theta) \Big) \Big( \frac{\partial}{\partial w} \Phi_{l'}(W^{(1)}_{dh}, \ldots, w, \ldots, W^{(L)}_{dh}) \big|_{w = W^{(l)}_{dh}} \Big)\cr
&\quad \times \Big( T_{l'}(y_d)\, - \mathbb{E}_{p(y;\, \vec{\Phi}(W^{(1)}_{dh}, \ldots, W^{(L)}_{dh}))}  [ T_{l'}(y) ] \Big).
\end{align*}
Note that the equation now depends on the hidden state $\vec{s}$
only through the function $\mathcal{A}_{dh}(\vec{s}, \Theta)$. Then, using Eqn. \eqref{EqnFreeEnergy}, it can be stated that:
\begin{align}\label{theorem_Pderivative4}
\frac{\partial}{\partial W^{(l)}_{dh}}&\mathcal{F}(q, \Theta) 
 = \sum_{n} \sum_{\vec{s}} q^{(n)}(\vec{s}\,) \Big( \sum_{d'} \frac{\partial}{\partial W^{(l)}_{dh}} \log\big( p( y_{d'}^{(n)}; \vec{\tilde{\eta}}_{d'}(\vec{s}, \Theta)) \big) \Big)\cr
&= \sum_{n} \sum_{\vec{s}} q^{(n)}(\vec{s}\,) \sum_{l'=1}^L \Big( \mathcal{A}_{dh}(\vec{s}, \Theta) \Big) \Big( \frac{\partial}{\partial w} \Phi_{l'}(W^{(1)}_{dh}, \ldots, w, \ldots, W^{(L)}_{dh}) \big|_{w = W^{(l)}_{dh}} \Big)\cr
&\quad \times \Big( T_{l'}(y_d^{(n)})\, - \mathbb{E}_{p(y;\, \vec{\Phi}(W^{(1)}_{dh}, \ldots, W^{(L)}_{dh}))}   [ T_{l'}(y) ] \Big)\cr
&= \sum_{l'=1}^L \Big( \frac{\partial}{\partial w} \Phi_{l'}(W^{(1)}_{dh}, \ldots, w, \ldots, W^{(L)}_{dh}) \big|_{w = W^{(l)}_{dh}} \Big)  \cr
&\quad \phantom{xxxix} \times \sum_{n} \mathbb{E}_{q^{(n)}}   [ \mathcal{A}_{dh}(\vec{s}, \Theta) ] \Big( T_{l'}(y_d^{(n)}) - W^{(l')}_{dh} \Big), \nonumber
\end{align}
%
where for the last step we exploited the mean value parametrization defined by $\vec{w} :=\, \mathbb{E}_{p(y;\, \vec{\Phi}(\vec{w}) )}  [\vec{T}(y) ]$. 
Therefore, independently of the functions $\frac{\partial}{\partial w} \Phi_{l'}(W^{(1)}_{dh}, \ldots, w, \ldots, W^{(L)}_{dh}) \big|_{w = W^{(l)}_{dh}}$ for each $l'$, we have $\frac{\partial \mathcal{F}}{\partial W^{(l)}_{dh}} = 0$ if for all $l = 1, \ldots, L$ we have:
\begin{align}
\sum_n \mathbb{E}_{q^{(n)}}  [ \mathcal{A}_{dh}(\vec{s}, \Theta) ] \Big(T_l(y_d^{(n)}) - W^{(l)}_{dh}\Big) = 0,
\end{align}
which yields Eqn.\,\eqref{Update_W_general} and completes the proof.
\end{proof}

\noindent{}The general case does, of course, also include the $L=1$ case and consequently Bernoulli, Exponential or the Poisson distribution. Moreover, similar to the $L = 2$ case, we can gather further information about points where $W_{dh}^{(l)}$ satisfies Eqn.\,\eqref{Update_W_general}. 
Considering the second derivatives again as in Sec.~\ref{Sec_second_derivative}, we can   likewise show that fixed-point solutions for a single $W_{dh}^{(l)}$ and for fixed $d$ and $h$ correspond to (local) maxima if the Hessian matrix $\mathbf{H}_{\mathcal{F}}$ is positive definite. 

\section{Additional Details on Parameter Update Equations}
\label{app:models}

\subsection{Parametrization of the Gaussian-MCA}
\label{appendixA2}
For Gaussian-MCA (G-MCA) in Example~\ref{example2}, the generative model was shown to be given by:
\begin{align}
p(\vec{s}\,| \Theta) &= \prod_{h=1}^H \big( \pi_h^{s_h} (1 - \pi_h)^{1 - s_h} \big), \\
p(\vec{y}\,|\,\vec{s}, \Theta) &= \prod_{d=1}^D \mathcal{N}(y_d;\bar{W}_d(\vec{s}, \Theta),\bar{V}_d(\vec{s}, \Theta)-\bar{W}^2_d(\vec{s}, \Theta)),
%
\end{align}
where we can use the update equations of Theorem~\ref{theorem1} for $W$ and $V$. To obtain some intuition, we can also define the function:
%
\begin{align}
\bar{\sigma}^2_d(\vec{s}, \Theta) &= \bar{V}_d(\vec{s}, \Theta)-\bar{W}^2_d(\vec{s},\Theta), \quad \forall d. \nonumber
\end{align}
%
Because of the definition of $\bar{W}_d(\vec{s}, \Theta)$ and $\bar{V}_d(\vec{s}, \Theta)$ in (\ref{maxh}) we get:
\begin{align}
\bar{\sigma}^2_d(\vec{s}, \Theta) &= \sigma^2_{d h(d,\vec{s},\Theta)}\ \mbox{\ where\ }\ h(d,\vec{s},\Theta) = \mathrm{argmax}_{h} \{W_{dh} s_h\}\ \mbox{\ and\ }\ \sigma^2_{dh}=V_{dh}-W^2_{dh} ,  \nonumber
\end{align}
%
such that $\sigma^2$ is a matrix with $D \times H$ entries. The considered generative model then becomes:
\begin{align}
p(\vec{s}\,| \Theta) &= \prod_{h=1}^H \big( \pi_h^{s_h} (1 - \pi_h)^{1 - s_h} \big), \\
p(\vec{y}\,|\,\vec{s}, \Theta) &= \prod_{d=1}^D \mathcal{N}(y_d;\bar{W}_d(\vec{s},\Theta), \bar{\sigma}^2_d(\vec{s}, \Theta)).
%
\end{align}

\noindent{}The latents thus change the mean via the matrix $W$ and the variance via the matrices $V$ and $W$.
We can now use the update rules for $V_{dh}$ and $W_{dh}$ and compute the elements $\sigma^2_{dh}$ as a result. Alternatively, we can also combine the update rules for $V_{dh}$ and $W_{dh}$
to directly obtain an update rule for $\sigma^2_{dh}$:

\begin{align}
(\sigma^2&_{dh})^{\mathrm{new}} = V^{\mathrm{new}}_{dh}-(W^{\mathrm{new}}_{dh})^2 = V^{\mathrm{new}}_{dh}-2\,W^{\mathrm{new}}_{dh}\,W^{\mathrm{new}}_{dh}  +(W^{\mathrm{new}}_{dh})^2\cr
			       &= \frac{\sum_{n = 1}^N \mathbb{E}_{q^{(n)}}  [ \mathcal{A}_{dh}(\vec{s},\Theta) ] \, (y_d^{(n)})^2}{\sum_{n = 1}^N \mathbb{E}_{q^{(n)}}  [ \mathcal{A}_{dh}(\vec{s}, \Theta) ] }
                                 - 2 \frac{\sum_{n = 1}^N \mathbb{E}_{q^{(n)}}  [ \mathcal{A}_{dh}(\vec{s},\Theta) ] \, y_d^{(n)}}{\sum_{n = 1}^N \mathbb{E}_{q^{(n)}}  [ \mathcal{A}_{dh}(\vec{s}, \Theta) ] } W^{\mathrm{new}}_{dh}
                                 +(W^{\mathrm{new}}_{dh})^2\cr
			       &= \frac{\sum_{n = 1}^N \mathbb{E}_{q^{(n)}}  [ \mathcal{A}_{dh}(\vec{s},\Theta) ] \, \big( (y_d^{(n)})^2 - 2y_d^{(n)}W^{\mathrm{new}}_{dh} + (W^{\mathrm{new}}_{dh})^2 \big)}{\sum_{n = 1}^N \mathbb{E}_{q^{(n)}}  [ \mathcal{A}_{dh}(\vec{s}, \Theta) ] }\cr
			       &= \frac{\sum_{n = 1}^N \mathbb{E}_{q^{(n)}}  [ \mathcal{A}_{dh}(\vec{s},\Theta) ] \, \big( y_d^{(n)} - W^{\mathrm{new}}_{dh} \big)^2}{\sum_{n = 1}^N \mathbb{E}_{q^{(n)}}  [ \mathcal{A}_{dh}(\vec{s}, \Theta) ] }.
%
%
\end{align}

\noindent{}This form of update is more familiar as it evaluates the square deviation from the mean given by $W^{\mathrm{new}}_{dh}$. Also note that we first have to compute $W^{\mathrm{new}}_{dh}$ before
updating $\sigma^2_{dh}$, which is in analogy, e.g., to standard Gaussian mixtures for variance (or covariance) updates. For the G-MCA application of Sec.~\ref{sec_natural_image_patches}, we showed  $W_{dh}$ and $\sigma_{dh}$ (standard deviations) for whitened natural image patches.

\subsection{Parametrization of the Gamma-MCA}
\label{appendixA3}
Let Gamma be the noise distribution in \eqref{modelEq1}-\eqref{modelEq3}; i.e., consider the following  generative model: 
\begin{align} 
p(\vec{s}\,| \Theta) &= \prod_{h=1}^H \big( \pi_h^{s_h} (1 - \pi_h)^{1 - s_h} \big), \label{eq:pygs_gamma_prior} \\
p(\vec{y}\,|\,\vec{s}, \Theta) &= \prod_{d=1}^D \mbox{Gamma}(y_d; \vec{\tilde{\eta}}_d (\vec{s}, \Theta)), \quad \mbox{for} \quad y_d\in (0, \infty). \label{eq:pygs_gamma1}
%
\end{align}
%
%
Given the shape and rate parameters $\alpha, \beta >0$ of the Gamma distribution, the natural parameters and sufficient statistics are given by: 
\begin{align*}
\vec{\eta} = (- \beta, \alpha - 1)^T, \quad \vec{T}(y) = (y, \log(y))^T.
\end{align*}
Now, based on the mean value parametrization in  \eqref{mean_value_parameters} and also \eqref{phi_function}, we get:
\begin{align*}
w_1 = \mathbb{E}_{p(y; \, \vec{\Phi}(\vec{w}))}  [ y ] \,  \mbox{\ and\ }\   w_2= \mathbb{E}_{p(y;\, \vec{\Phi}(\vec{w}))}  [ \log(y) ], 
\end{align*}
where $\mathbb{E}_{p(y;\, \vec{\Phi}(\vec{w}))}  [ y ]$ represents the mean of observables and is equal to $- \frac{\eta_2 + 1}{\eta_1}$. Also we know that
\begin{align*}
\mathbb{E}_{p(y;\, \vec{\Phi}(\vec{w}))}  [ \log(y) ] = \psi(\eta_2 + 1) - \log(- \eta_1), 
\end{align*}
with the Digamma function $\psi(\cdot)$ defined by:
\begin{align}\label{digamma_func}
\psi(x) &= \frac{d}{dx} \log(\Gamma(x)) = \log(x) - \frac{1}{2x} - \frac{1}{12x^2} + \ldots.
\end{align}
Thus, the relation between the mean value parameters and the natural parameters can be expressed by: 
\begin{align*}
\vec{w} = \begin{pmatrix}
w_1\\
w_2
\end{pmatrix} = \begin{pmatrix}
- \frac{\eta_2 + 1}{\eta_1}\\
\psi(\eta_2 + 1) - \log(- \eta_1)
\end{pmatrix}.
\end{align*}
In order to compute the function $\vec{\Phi}$, we further substitute $\eta_1 = - \frac{\eta_2 + 1}{w_1}$ (the first equation above) into $w_2 = \psi(\eta_2 + 1) - \log(- \eta_1)$ (the second equation above) that yields:
\begin{align}\label{etaApprox}
w_2 &\approx \log(\eta_2 + 1) - \frac{1}{2(\eta_2 + 1)} - \log(\frac{\eta_2 + 1}{ w_1}) = - \frac{1}{2(\eta_2 + 1)} + \log(w_1),
\end{align}
where we approximated the Digamma function with its first two terms. Then, after some simplifications, it can be easily seen that the above system of equations results in:
\begin{align}\label{eta1_eta2Gamma}
\eta_1 \approx \frac{-1}{2 w_1 (\log(w_1) - w_2)} \quad \mathrm{and} \quad \eta_2 \approx \frac{1}{2(\log(w_1) - w_2)} - 1.
\end{align}
Therefore, using the natural parameters \eqref{eta1_eta2Gamma} and the definition of the link function in \eqref{eta}, we get:
\begin{align}\label{etaGamma}
\vec{\tilde{\eta}}_d&(\vec{s},\Theta) = \vec{\Phi}\big(\bar{W}_d(\vec{s}, \Theta), \bar{V}_d(\vec{s}, \Theta)\big) \approx \left( 
\begin{array}{c} \frac{-1}{2 \bar{W}_d(\vec{s}, \Theta) \big( \log(\bar{W}_d(\vec{s}, \Theta)) - \bar{V}_d(\vec{s}, \Theta)\big)} \\
\frac{1}{2 \big( \log(\bar{W}_d(\vec{s}, \Theta)) - \bar{V}_d(\vec{s}, \Theta)\big)} - 1
\end{array} \right). \nonumber
\end{align}
The function $\vec{\Phi}$ defines the link between natural and mean value parameters which will be used in the EM algorithm. Observe that in each EM iteration,  we require to compute $\vec{\tilde{\eta}}_d(\vec{s},\Theta)$ (e.g., to compute the posteriors) given the mean value parameters $\bar{W}_d(\vec{s}, \Theta)$ and $\bar{V}_d(\vec{s}, \Theta)$. These mean values can be further used to estimate the variance of the Gamma distribution such as $\bar{\sigma}_d^2(\vec{s}, \Theta)$. Alternatively, we can use the variance formula of the Gamma distribution (given by $\sigma^2 = \frac{\eta_2 + 1}{\eta_1^2}$) and let:
\begin{align}
\bar{\sigma}_d^2(\vec{s}, \Theta) \approx 2 \bar{W}^2_d(\vec{s}, \Theta) \big( \log(\bar{W}_d(\vec{s}, \Theta)) - \bar{V}_d(\vec{s}, \Theta) \big), \quad \forall d,
\end{align}
where (similar to the Gaussian case)
\begin{align}
\bar{\sigma}^2_d(\vec{s}, \Theta) &= \sigma^2_{d h(d,\vec{s},\Theta)}\ \ \mbox{\ and\ }\ \ h(d,\vec{s},\Theta) = \mathrm{argmax}_{h} \{W_{dh} s_h\}.
\end{align}
Likewise, the variance elements can be directly updated in each M-step by:
\begin{align}\label{sigmaM_gamma}
(\sigma^2_{dh})^{\mathrm{new}} &\approx 2 (W^{\mathrm{new}}_{dh})^2 \Big( \log(W^{\mathrm{new}}_{dh}) - V^{\mathrm{new}}_{dh} \Big) \cr
&= 2 (W^{\mathrm{new}}_{dh})^2 \Big( \frac{\sum_{n = 1}^N \mathbb{E}_{q^{(n)}}  [ \mathcal{A}_{dh}(\vec{s},\Theta) ] \, \log(W^{\mathrm{new}}_{dh})}{\sum_{n = 1}^N \mathbb{E}_{q^{(n)}} [ \mathcal{A}_{dh}(\vec{s}, \Theta) ] } - \frac{\sum_{n = 1}^N \mathbb{E}_{q^{(n)}}  [ \mathcal{A}_{dh}(\vec{s},\Theta) ] \, \log(y_d^{(n)})}{\sum_{n = 1}^N \mathbb{E}_{q^{(n)}}  [ \mathcal{A}_{dh}(\vec{s}, \Theta) ] } \Big) \cr
&= 2 \frac{\sum_{n = 1}^N \mathbb{E}_{q^{(n)}}  [ \mathcal{A}_{dh}(\vec{s},\Theta) ] \, (W^{\mathrm{new}}_{dh})^2 \big( \log(W^{\mathrm{new}}_{dh}) - \log(y_d^{(n)}) \big) }{\sum_{n = 1}^N \mathbb{E}_{q^{(n)}}  [ \mathcal{A}_{dh}(\vec{s}, \Theta) ] }.
\end{align}

%

\noindent{}The above equation has been used for updating the variance of the Gamma-MCA model in our experiments in Sec. \ref{subsec:noise_type_estimation}. Despite the approximation of the Digamma function, we observed that the aforementioned equation provides, in practice, a good estimation of the variance parameter. Other approximations of the Digamma function can be further used to improve the performance of the method. For instance, we also used the first three and also four terms of the Digamma function in \eqref{digamma_func} for our experiments and consistently observed a better approximation of the variance.

\subsection{The Evidence Lower Bound (ELBO)}
\label{appendix_ELBO}
For completeness, we here show the ELBO / free energy \eqref{EqnFreeEnergy} can be derived as a lower bound of the log-likelihood function.  For this, we use the view as presented, e.g., by \citet[][]{NealHinton1998,Bishop2006}
and introduce variational distributions $q^{(n)}(\vec{s}\,)$ defined over the latent variables
(as discussed before, these distributions can be full posteriors $p(\sVec\mid\yVecN,\Theta)$, or denote approximations in the case of intractable posteriors).
Then, for an EF-MCA data model  \eqref{modelEq1}-\eqref{modelEq3} we have:
\begin{align}\label{appendix_ELBOeq}
\mathcal{L}(\Theta) &= \sum_n \log \big(\sum_{\vec{s}} q^{(n)}(\vec{s}\,) \frac{p(\vec{y}^{\,(n)}, \vec{s}\,| \Theta)}{q^{(n)}(\vec{s}\,)}\big) 
\geq \sum_n \sum_{\vec{s}} q^{(n)}(\vec{s}\,) \log \big( \frac{p(\vec{y}^{\,(n)}, \vec{s}\,| \Theta)}{q^{(n)}(\vec{s}\,)}\big)\cr
&= \sum_n \sum_{\vec{s}} q^{(n)}(\vec{s}\,) \log \big( p(\vec{y}^{\,(n)}, \vec{s}\, | \Theta)\big) - \sum_n \sum_{\vec{s}} q^{(n)}(\vec{s}\,) \log\big( q^{(n)}(\vec{s}\,) \big) = \mathcal{F}(q, \Theta),
\end{align}
where we used Jensen's inequality in the first line \citep[see, e.g.,][for the details]{Bishop2006}.
%
%
%
%
The last equation above then represents the free energy function $\mathcal{F}(q, \Theta)$ (or the ELBO) given by \eqref{EqnFreeEnergy}.

\section{Additional Details on Numerical Experiments}
\label{appendixB}

\subsection{The Bars Test}
\label{bars_test_appendix}
For the bars tests presented in Sec.~\ref{sec:ArtificialData}, we set $\pi_h^{\mathrm{gen}} = 0.2$ for $h = 1, \ldots, H$ (i.e., two active bars on average per data point) and generated $N = 1000$ i.i.d. data points according to each of the considered EF-MCA models. The ground-truth parameter $W$ was set to a value of $10$ for the bars and $1$ for the non-bar pixels in the case of E-MCA, and $0.99$ for the bars and $0.01$ for the non-bars in the case of B-MCA. Next, we trained E- and B-MCA models on their corresponding datasets and updated their parameters using  \eqref{update_W_real} and \eqref{Update_pi} for $50$ full EM iterations.  
We initialized $W$ by 
computing the mean of the data and adding a small amount of Gaussian noise; 
also the priors $\pi_h$ were initialized at $0.3$ for each $h$. 
In each EM iteration,  the fixed-point M-step equation was iterated only once, as we observed comparable results compared to running it multiple times for both E- and B-MCA models.  Also for P-MCA, we used $H = 6$, $D = 9$ and set other parameters to be similar to the E-MCA model.
\ \\

\noindent{}\textbf{Relation to Previous MCA Approaches.} 
In previous studies \citep[e.g.,][]{LuckeSahani2008,LuckeEtAl2009,LuckeEggert2010,BornscheinEtAl2013}, the maximum function has been commonly approximated using a smooth function for the sake of numerical stabilities.  That is, for the case of Poisson or Gaussian distributions (where $M(\Theta) = W$), a smooth function $\bar{W}_d^{\rho}$ is used to approximate the function $\bar{W}_d$ as $\rho$ approaches infinity: 
\begin{align*}
\bar{W}_d^{\rho}(\vec{s}, \Theta) := \Big( \sum_{h=1}^H (W_{dh} s_h )^{\rho}  \Big)^{\frac{1}{\rho}} \quad \mbox{where} \quad 
\lim_{\rho\rightarrow\infty} \bar{W}_d^{\rho}(\vec{s}, \Theta) = \bar{W}_d (\vec{s}, \Theta).
\end{align*}
Parameter $\rho$ does in fact control the level of non-linearity and it can be varied during the learning  \citep[as it has been done by][]{LuckeEtAl2009}, or it can be set to a large fixed value \citep[similar to the procedure considered by][who used, e.g.,  $\rho = 21$]{BornscheinEtAl2013}.  
Consequently, $\mathcal{A}_{dh} (\vec{s}, \Theta)$ can be rewritten by \citep[see, e.g.,][for the details]{LuckeSahani2008}:
\begin{align*}
\mathcal{A}_{dh} (\vec{s}, \Theta) := \lim_{\rho\rightarrow\infty} \Big( \frac{\partial}{\partial W_{dh}} \bar{W}_d^{\rho}(\vec{s}, \Theta) \Big) = \lim_{\rho\rightarrow\infty} 
\frac{s_h (W_{dh})^{\rho}}{\sum_{h} s_h (W_{dh} )^{\rho} }\ .
\end{align*}
We here observed that such an approximation is beneficial to avoid numerical instabilities which can be used, e.g., along other approaches for avoiding local optima.
\ \\

\noindent{}\textbf{Avoiding Local Optima.} 
In our bars test experiments, we executed the algorithms multiple times while holding the hyper-parameters fixed and using different realizations of the initial model parameters in each run. We observed that the best solutions recovered all bars with high accuracy. 
Nonetheless, 
as mentioned before, slight overfitting effects occurred in some cases that could be diminished by increasing the number of data points. We also frequently observed that the algorithms could not recover the ground-truth parameters. This can be seen as the effect of local optima which is different for each of the EF-MCA models. Local optima effects showed, for instance, to be more severe for Gamma-MCA compared to others.
While annealing methods have frequently been applied as a measure against local optima effects \citep[e.g.,][]{UedaNakano1998,LuckeSahani2008}, we here investigated an alternative approach in order to prevent the algorithm from converging to sub-optimal solutions. In detail, we first employed the Gaussian-MCA model as it is shown to be more robust against local optima and performed 50 EM iterations. We then used the learned $\Theta$ parameters (of the Gaussian-MCA) to initialize the parameters of the desired EF-MCA model. We observed that, e.g. for Gamma-MCA, such a method helps alleviate the effect of local optima.

\begin{figure}[t!]
  \centering
  \includegraphics[scale=0.45]{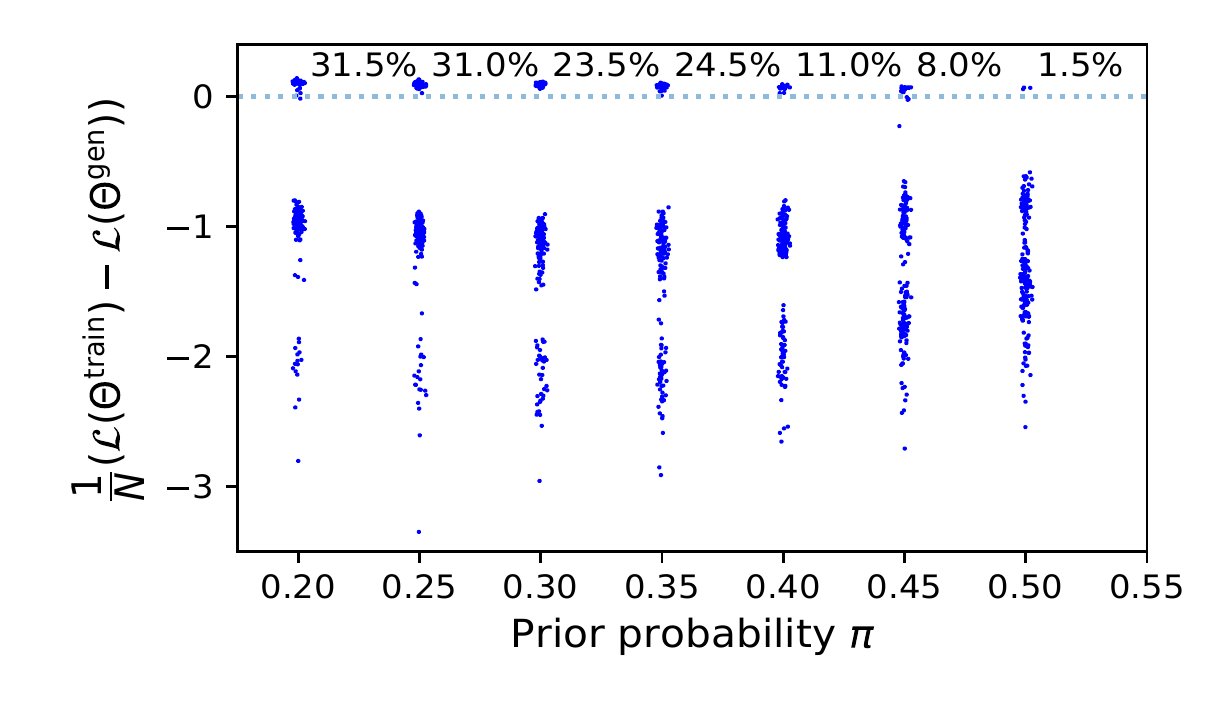}
  \caption{Differences of log-likelihoods between trained B-MCA models and ground-truth values for bars test datasets.
  Average active number of bars ($\pi H$) is set by different prior probabilities $0.2 \leq \pi \leq 0.5$ in step size $0.05$.
  Individual runs are shown with small offsets along the $\pi$-axis for a better visualization.   Percentages refer to the fraction of runs that converged to ground-truth values (in such cases, the learned log-likelihoods are usually higher than the ground-truth due to overfitting).}
  \label{fig:Reliability}
\end{figure}
%
\ \\
\noindent{}\textbf{Reliability For Binary Data.} 
In addition to the local optima effect, the inherent complexity of the model can also affect its performance. That is, for instance, increasing the sparseness (i.e., $\sum_h \pi_h$ which in this case is equal to $\pi H$ as we assumed $\pi_h = \pi$ for $h = 1, \ldots, H$) will decrease the reliability of the model.
In fact, one measure that has been of interest is how robust an algorithm is w.r.t. the average number of active bars, and how often it reaches the global vs.\ any local optimum. The probability of recovering all bars has been commonly termed \textit{reliability} of the algorithm \citep[][]{Spratling2006}. In the following, we investigate the reliability of the Bernoulli-MCA (B-MCA) model for different
levels of sparseness. For our purposes, we notably did not optimize learning to improve reliability \citep[e.g., by introducing annealing procedures;][]{LuckeEggert2010} but used the canonical form of
B-MCA with exact E-steps (i.e., full posteriors) and updated $W$ and $\vec{\pi}$ using Eqns. \eqref{update_W_real} and \eqref{Update_pi}.
We generated artificial datasets according to the B-MCA model similar to the procedure presented above with $W_{dh} \in \{0.01, 0.99\}$ and varied the value of $\pi$ that determines the average number of active bars per data point.
For each value of $\pi$, we generated $200$ different bars datasets each with $N = 1000$ data points and then fitted a B-MCA model.
We computed $50$ full EM iterations with the same initialization as the experiments above and measured the reliability in terms of the percentage of all trained B-MCA models that achieved a higher log-likelihood value than the ground-truth.  
It should be noted that we commonly observed the overfitting effect as we used a finite sample size here and thus, the best runs were observed to have log-likelihood values that are slightly higher than the ground-truth log-likelihood.  Results are presented in Fig. \ref{fig:Reliability}.

As it can be observed, the best runs reach higher log-likelihood values than the ground-truth. These runs do recover all bar patterns and the prior parameters.
Even for values $\pi=0.5$, i.e., for five out of ten bars per data point on average, one to two out of 100 runs do extract all bars. Thus, the best runs in terms of higher log-likelihood values can be automatically determined without
knowledge of the ground-truth such that the proposed algorithm (B-MCA here) yields a reliable approach for extracting all bars for binary data  \citep[compare][]{Spratling2006,FrolovEtAl2016}. 

\begin{figure}
\centering
\includegraphics[scale=0.6]{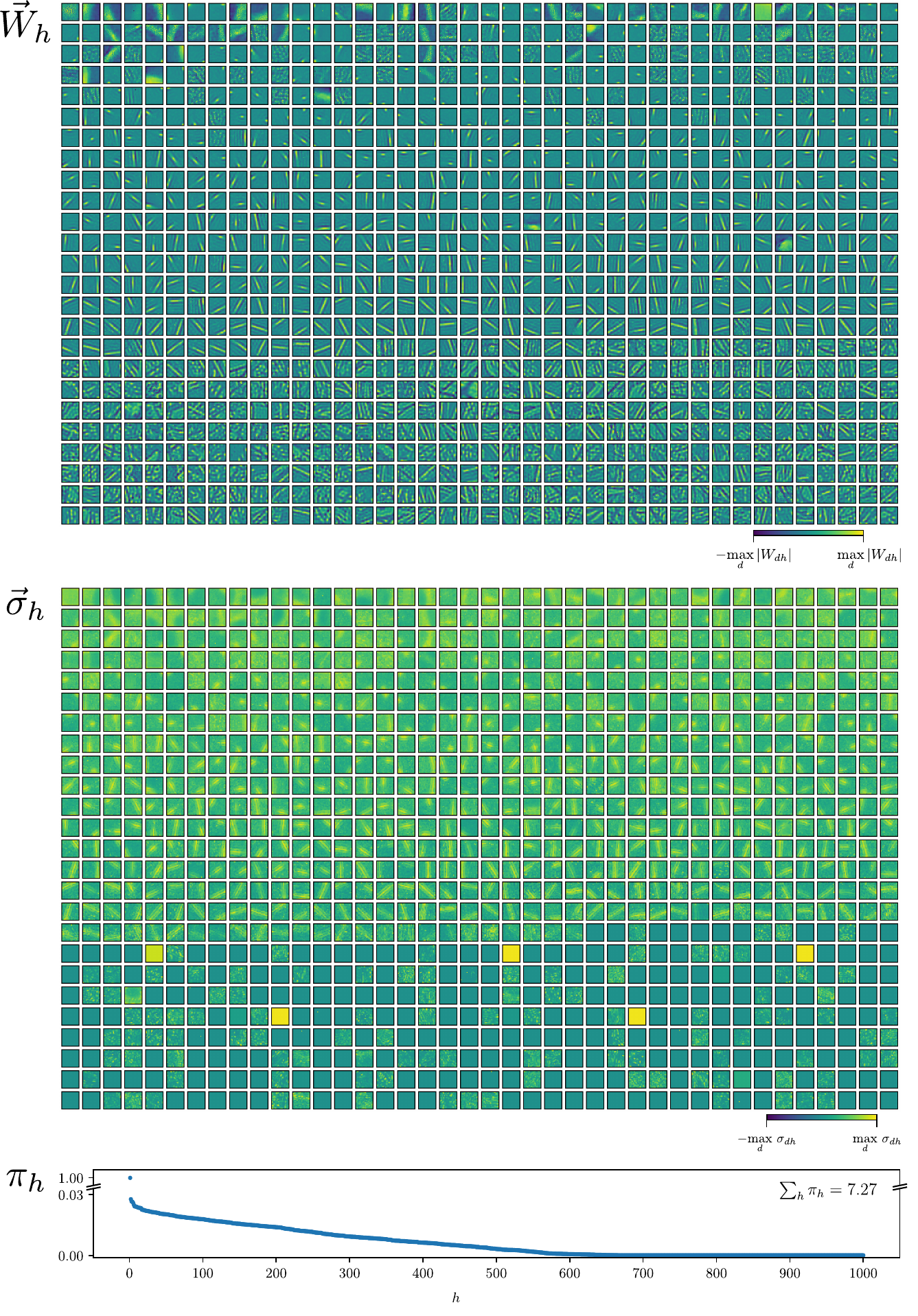}
\caption{Complete dictionaries with $H=1,000$ component means and component variances learned from natural image patches using Gaussian-MCA (for improved visibility, we illustrate standard deviations instead of variances;
compare Sec.~\ref{sec_natural_image_patches}). The generative fields are ordered according to their activations, starting with the fields corresponding to the most active hidden units.}
\label{fig:image_patches1000}
\end{figure}

\subsection{Feature Extraction -- Natural Image Patches} %
\label{Natural_images_appendix}
To initialize the model parameters $\Theta=\{\vec{\pi}, W, \sigma^2 \}$ 
of the Gaussian-MCA model described in Sec.~\ref{sec_natural_image_patches}, we followed the procedure used in \citet[][]{mousavi2020ddDictionary}: The $\pi_h^{\init}$ were randomly uniformly drawn from the interval $[0.1,0.5]$ for $h=1,\dots,H$. 
The generative fields (GFs) for component variances were set equal to the average variance per pixel, i.e., $(\sigma_{dh}^2)^{\init}=\sigma_{\mathcal{Y}}^2\ \forall d,h$ with $\sigma_{\mathcal{Y}}^2 = \frac{1}{DN} \sum_{d=1}^{D} \sum_{n=1}^{N} (y_d^{(n)} - \bar{y}_d)^2$ and $\bar{y}_d=\frac{1}{N}\sum_{n=1}^{N} y_d^{(n)}$.
The initial values for the GFs for component means were obtained by clustering the dataset using 
a variational GMM algorithm (we used the var-GMM-S algorithm implemented in \citep{vcGMMCodeBitbucket} together with the default hyper-parameters; also see \citet{HirschbergerEtAl2022,ForsterLucke2018Sublinear}),  
and then setting $\vec{W}_h^{\init}$ equal to the learned $H$ cluster centers. 
In order to ensure positive values for $\bar{\sigma}$ (see, e.g.,  Fig.\,\ref{fig:image_patches20}), we used the background model discussed in \citet[][]{MousaviEtAl2021} and enforced one hidden unit $s_1={\rm const}=1$ during learning. 
%
%
We performed 500 EM iterations, initially updating only $W$ and $\vec{\pi}$, while $\sigma^2$ remained fixed at its initial value. After approximately half of the iterations, i.e., when we could no longer observe significant changes of the lower bound, we fixed $W$ and $\vec{\pi}$ and only updated $\sigma^2$ for the remaining EM steps (we found this procedure to help avoiding oscillations of the lower bound which we observed when updating all parameters jointly from the beginning).
Figure~\ref{fig:image_patches1000} depicts the $H=1000$ learned GFs as well as the learned priors
(the GFs for the background unit correspond to the first depicted dictionary element, the respective prior equals one).

Using compute resources of the HPC cluster CARL of the University of Oldenburg (which is equipped with Intel Xeon E5-2650 v4 12C, E5-2667 v4 8C and E7-8891 v4 10c CPUs), we observed an average runtime of approximately 10 minutes per EM step for our implementation when being executed distributedly on 556 CPU cores.

\begin{figure}[t!]
\centering
\includegraphics[width=\textwidth]{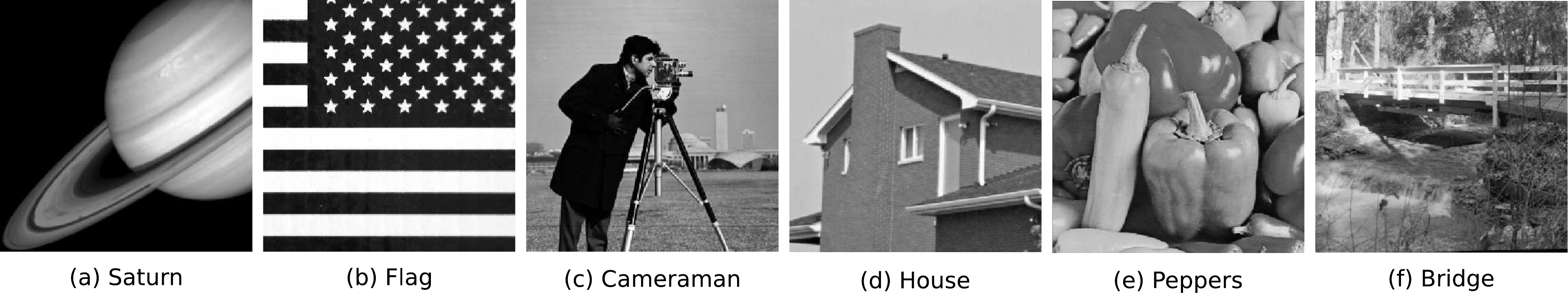}
\caption{Original images used for denoising benchmarks.}
\label{fig:original_images}
\end{figure}

\subsection{Noise Type Estimation}
\label{app:noise_type_estimation}
The problem of noise type estimation has been addressed in a number of contributions. \cite{TeymurazyanEtAl2013}, for instance, studied the problem of distinguishing the type of noise in Positron Emission Tomography (PET) data \citep[][]{TeymurazyanEtAl2013}. In PET, radioactive fluids are injected into humans or animals in order to obtain images for diagnostics or scientific studies. As a result, knowing the type of noise is crucial for different image processing routines \citep[][]{TeymurazyanEtAl2013,MouEtAl2018}. 
Other related work focused on machine learning automation  \citep[e.g.,][]{valera2017automatic,vergari2019automatic}. In \citet[][]{valera2017automatic}, a Bayesian approach is presented to infer the type of data in categories such as categorical, ordinal, count, real-valued, positive real-valued or interval \citep[also see][]{MolinaEtAl2018}. 
%
As discussed in Sec.~\ref{Relation_to_Previous_Work}, such approaches are different from our work here as they exploit (among other differences) a mixture of likelihood functions.  
We, on the other hand,  trained each of the EF-MCA data models here and directly used the obtained ELBOs to infer the noise type of a given dataset. 

For the experiments in Sec.~\ref{subsec:noise_type_estimation}, we used $D = 144$ and $H = 100$ (for images) or $H = 512$ (for acoustic data), and trained the EF-MCA data models for 100 variational EM iterations. 
The three examples in Tab.~\ref{table:chimeresults} that we selected from the CHiME dataset \citep[][]{FosterEtAl2015} were: CR-lounge-200110-1711.s0-chunk46.48kHz.wav, CR-lounge-200110-1601.s0-chunk25.48kHz.wav and CR-lounge-200110-1601.s0-chunk7.48kHz.wav. 
Amplitude spectrograms were computed as follows: Time domain signals were resampled to $22.050$\, Hz and cut into $2$ seconds long mono segments. We then computed the STFT using a $2048$-point FFT, $512$ samples frame shift and Hann windowing. This resulted in spectrograms with $1025$ frequency channels and $87$ time steps. The amplitude spectrograms were then cut into patches of size $D=12\times12$ resulting in $N=77,064$ data points. After learning, we observed very small values of components $\sigma_{dh}$ (between $0.1$ and $0.2$) 
for both Gaussian- and Gamma-MCA models which resulted in very high, positive free energies in Tab.~\ref{table:chimeresults}.

\subsection{Denoising}
\label{denoising}
We used images with gray scales in the interval $[0, 255]$ (see the original images in Fig.~\ref{fig:original_images}) and, 
rescaled the pixel amplitudes by dividing by the maximum amplitude and then multiplying by the desired peak value.  
Similar to \citet[][]{DrefsEtAl2022},  we further cut the images into smaller patches by moving a sliding window over the noisy versions.  
We then employed the EF-MCA algorithms with $D = 20 \times 20$ \citep[similar to][]{giryes2014sparsity} and different $H$ values (for Poisson noise), and $D = 12 \times 12$ and $H = 512$ (for Exponential noise). Each algorithm was trained for $100$ variational EM iterations without any annealing or further pre- or post-processing. 
For both E- and P-MCA, we manually set the minimum value of $W$ to $10^{-2}$ after each M-step in order to avoid zero values. 

Tab.~\ref{table:poisson_denoising} lists the PSNR values for VST+BM3D, I+VST+BM3D and P$^4$IP from \citet[][]{azzari2016variance}, BM3D, NLSPCA and SPDA from \citet[][]{giryes2014sparsity}, DenoiseNet from \citet[][]{remez2017deep} and DCEA-UC from \citet[][]{tolooshams2020convolutional}. 
Also, 
Fig.~\ref{table:PoissonResults_appendix} presents the reconstructed images using different generative models exploited for Poisson denoising.

\begin{figure}[t!]
\centering
\includegraphics[scale=1.15]{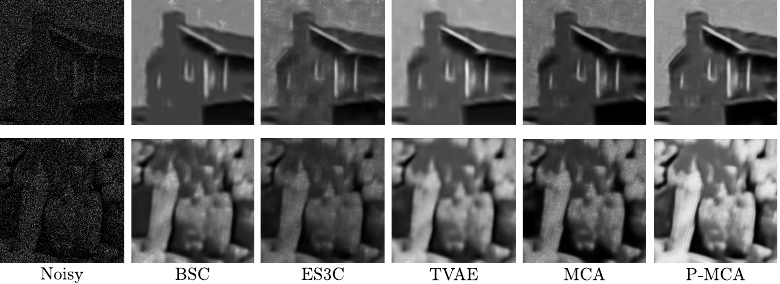}
\caption{Reconstructed images obtained by different LVMs for Poisson denoising at peak value 1. The corresponding PSNR and ELBO values are presented in Fig.~\ref{table:PoissonResults}.}
\label{table:PoissonResults_appendix}
\end{figure}

\bibliography{22-0359}

\end{document}